\documentclass{article}

\usepackage[final]{nips_2018}

\usepackage[utf8]{inputenc} % allow utf-8 input
\usepackage[T1]{fontenc}    % use 8-bit T1 fonts
\usepackage{hyperref}       % hyperlinks
\usepackage{url}            % simple URL typesetting
\usepackage{booktabs}       % professional-quality tables
\usepackage{amsfonts}       % blackboard math symbols
\usepackage{nicefrac}       % compact symbols for 1/2, etc.
\usepackage{microtype}      % microtypography

\pdfoutput=1

\usepackage{natbib}
\usepackage{url}
\usepackage{hyperref}
\hypersetup{colorlinks,
	linkcolor=blue,
	citecolor=blue,
	urlcolor=magenta,
	linktocpage,
	plainpages=false}
\usepackage{amsmath,amsthm,amsfonts}

\usepackage{algorithm}
\usepackage[noend]{algorithmic}
\usepackage{amsmath}
\usepackage{thm-restate}
\usepackage{amsfonts}
\usepackage{amssymb}
\usepackage{color}
\usepackage{mathrsfs}
\usepackage{enumitem}
\usepackage{bm}
\usepackage{booktabs}
\usepackage{pgfplots}
\usepackage{graphicx}
\graphicspath{{./fig/}}
\usepackage{xr}
\newcommand{\defeq}{\mathrel{\mathop:}=}

\newcommand{\vect}[1]{\ensuremath{\mathbf{#1}}}
\newcommand{\mat}[1]{\ensuremath{\mathbf{#1}}}

\newcommand{\grad}{\nabla}
\newcommand{\hess}{\nabla^2}

\newcommand{\norm}[1]{\|{#1}\|}

\newcommand{\trans}{^{\top}}
\newcommand{\poly}{\text{poly}}

\newcommand{\proj}{\mathcal{P}}

\newcommand\inner[2]{\langle #1, #2 \rangle}

\newcommand{\Z}{\mathbb{Z}}

\newcommand{\R}{\mathbb{R}}
\newcommand{\E}{\mathbb{E}}
\renewcommand{\P}{\mathbb{P}}

\newcommand{\Ind}{\mathbb{I}}

\newcommand{\fracpar}[2]{\frac{\partial #1}{\partial  #2}}

\newcommand{\A}{\mat{A}}
\newcommand{\B}{\mat{B}}

\newcommand{\I}{\mat{I}}

\newcommand{\D}{\mat{D}}

\newcommand{\e}{\vect{e}}
\newcommand{\g}{\vect{g}}
\renewcommand{\u}{\vect{u}}
\renewcommand{\v}{\vect{v}}
\newcommand{\w}{\vect{w}}
\newcommand{\x}{\vect{x}}
\newcommand{\y}{\vect{y}}
\newcommand{\z}{\vect{z}}

\renewcommand{\H}{\mathcal{H}}
\newcommand{\alg}{\mathcal{A}}
\newcommand{\F}{\mathcal{F}}

\newcommand{\nn}{\nonumber}
\newcommand{\eps}{\varepsilon}
\newcommand{\mineval}{\lambda_{\text{min}}}

\renewcommand{\sf}{f}
\newcommand{\sF}{F}
\newcommand{\sS}{S}
\newcommand{\sH}{H}

\newcommand{\tf}{\tilde{f}}
\newcommand{\tF}{\tilde{F}}
\newcommand{\tS}{\tilde{S}}
\newcommand{\tH}{\tilde{H}}
\newcommand{\gnu}{\tilde{\nu}}

\newtheorem{theorem}{Theorem}
\newtheorem{lemma}[theorem]{Lemma}

\theoremstyle{definition}
\newtheorem{definition}[theorem]{Definition}
\newtheorem{assumption}{Assumption}
\newtheorem{prob}{Problem}

\newcommand{\relu}{\text{ReLU}}
\newcommand{\emploss}{\hat{R}_n}
\newcommand{\poploss}{R}
\newcommand{\neibor}{\mathfrak{B}}
\renewcommand{\D}{\mathcal{D}}
\newcommand{\btheta}{\bm{\theta}}

\newcounter{lnotecount}[section]

\title{On the Local Minima of the Empirical Risk}

\author{
	Chi Jin\thanks{The first two authors contributed equally.} \\
	University of California, Berkeley \\
	\texttt{chijin@cs.berkeley.edu} 
	\And Lydia T.~Liu\footnotemark[1] \\
	University of California, Berkeley \\
	\texttt{lydiatliu@cs.berkeley.edu} 
	\And Rong Ge  \\
	Duke University \\
	\texttt{rongge@cs.duke.edu} 
	\And Michael I.~Jordan \\
	University of California, Berkeley \\
	\texttt{jordan@cs.berkeley.edu} 
}

\begin{document}

\maketitle

\begin{abstract}
% !TeX root = main.tex 

Population risk is always of primary interest in machine learning; 
however, learning algorithms only have access to the empirical risk. 
Even for applications with nonconvex nonsmooth losses (such as modern 
deep networks), the population risk is generally significantly more 
well-behaved from an optimization point of view than the empirical
risk.  In particular, sampling can create many spurious local minima.
We consider a general framework which aims to optimize a smooth nonconvex 
function $F$ (population risk) given only access to an approximation 
$f$ (empirical risk) that is pointwise close to $F$ (i.e., 
$\norm{F-f}_{\infty} \le \nu$). Our objective is to find the 
$\epsilon$-approximate local minima of the underlying function $F$ while 
avoiding the shallow local minima---arising because of the tolerance $\nu$---which 
exist only in $f$. We propose a simple algorithm based on stochastic 
gradient descent (SGD) on a smoothed version of $f$ that is guaranteed 
to achieve our goal as long as $\nu \le O(\epsilon^{1.5}/d)$.
We also provide an almost matching lower bound showing that our algorithm 
achieves optimal error tolerance $\nu$ among all algorithms making a 
polynomial number of queries of $f$.
% that among all algorithms making polynomial number of queries of $f$, our SGD-based approach achieves the optimal trade-off between the function error $\nu$, solution precision $\epsilon$ and problem dimension $d$. 
As a concrete example, we show that our results can be directly used to 
give sample complexities for learning a ReLU unit.

% This paper considers a general optimization framework which aims to find approximate local minima of a smooth nonconvex function $F$ (population risk) given only access to the function value of another function $f$ (empirical risk) that is pointwise close to $F$ (i.e., $\norm{F-f}_{\infty} \le \nu$). We propose a simple algorithm based on stochastic gradient descent (SGD) on a smoothed version of $f$ which is guaranteed to find an $\epsilon$-second-order stationary point of the original function $F$ if $\nu \le O(\epsilon^{1.5}/d)$, thus escaping all saddle points of $F$ and all the additional local minima introduced by $f$. We also provide an almost matching lower bound showing that among all algorithms making polynomial number of queries, our SGD-based approach achieves the optimal trade-off between the function error $\nu$, solution precision $\epsilon$ and problem dimension $d$. %, among all algorithms making a polynomial number of queries.
% As a concrete example, we show that our results can be directly used to give sample complexities for learning a ReLU unit, whose empirical risk is nonsmooth.

\end{abstract}

% !TeX root = main.tex 

\section{Introduction}

The optimization of nonconvex loss functions has been key to the success of 
modern machine learning.  While classical research in optimization focused
on convex functions having a unique critical point that is both locally 
and globally minimal, a nonconvex function can have many local maxima, 
local minima and saddle points, all of which pose significant challenges for 
optimization. A recent line of research has yielded significant progress
on one aspect of this problem---it has been established that favorable 
rates of convergence can be obtained even in the presence of saddle points, using 
simple variants of stochastic gradient descent~\citep[e.g.,][]{ge2015escaping, 
carmon16accelerated,agarwal17finding,jin17escape}.  These research results
have introduced new analysis tools for nonconvex optimization, and it is
of significant interest to begin to use these tools to attack the problems 
associated with undesirable local minima.

It is NP-hard to avoid all of the local minima of a general nonconvex function. 
But there are some classes of local minima where we might expect that simple
procedures---such as stochastic gradient descent---may continue to prove effective.  
In particular, in this paper we consider local minima that are created by small 
perturbations to an underlying smooth objective function.  Such a setting is natural 
in statistical machine learning problems, where data arise from an underlying
population, and the population risk, $F$, is obtained as an expectation over
a continuous loss function and is hence smooth; i.e., we have $F(\btheta) = 
\E_{\z \sim \D} [L(\btheta; \z)]$, for a loss function $L$ and population distribution 
$\D$.   The sampling process turns this smooth risk into an empirical risk, 
$f(\btheta) = \sum_{i=1}^n L(\btheta; \z_i)/n$, which may be nonsmooth and 
which generally may have many shallow local minima.  From an optimization 
point of view $f$ can be quite poorly behaved; indeed, it has been observed in 
deep learning that the empirical risk may have exponentially many shallow 
local minima, even when the underlying population risk is well-behaved and 
smooth almost everywhere \citep{brutzkus17globally,auer96exponentially}.
From a statistical point of view, however, we can make use of classical 
results in empirical process theory~\citep[see, e.g.,][]{boucheron,
bartlett03rademacher} to show that, under certain assumptions on the
sampling process, $f$ and $F$ are uniformly close:
\begin{equation}\label{eq:problem_intro}
    \norm{F - f}_{\infty} \le \nu, 
\end{equation}
where the error $\nu$ typically decreases with the number of samples $n$.  
See Figure \ref{fig:intro}(a) for a depiction of this result, and
Figure \ref{fig:intro}(b) for an illustration of the effect of sampling on
the optimization landscape. We wish to exploit this nearness of $F$
and $f$ to design and analyze optimization procedures that find 
approximate local minima (see Definition \ref{def:exactSOSP}) of the 
smooth function $F$, while avoiding the local minima that exist only in 
the sampled function $f$.

Although the relationship between population risk and empirical risk is
our major motivation, we note that other applications of our framework 
include two-stage robust optimization and private learning (see Section 
5.2).  In these settings, the error $\nu$ can be viewed as the amount of adversarial 
perturbation or noise due to sources other than data sampling.  As in the
sampling setting, we hope to show that simple algorithms such as stochastic 
gradient descent are able to escape the local minima that arise as a function
of $\nu$.

\begin{figure}[t]
	\centering
	\includegraphics[width=0.28\textwidth, trim=80 70 250 10, clip]{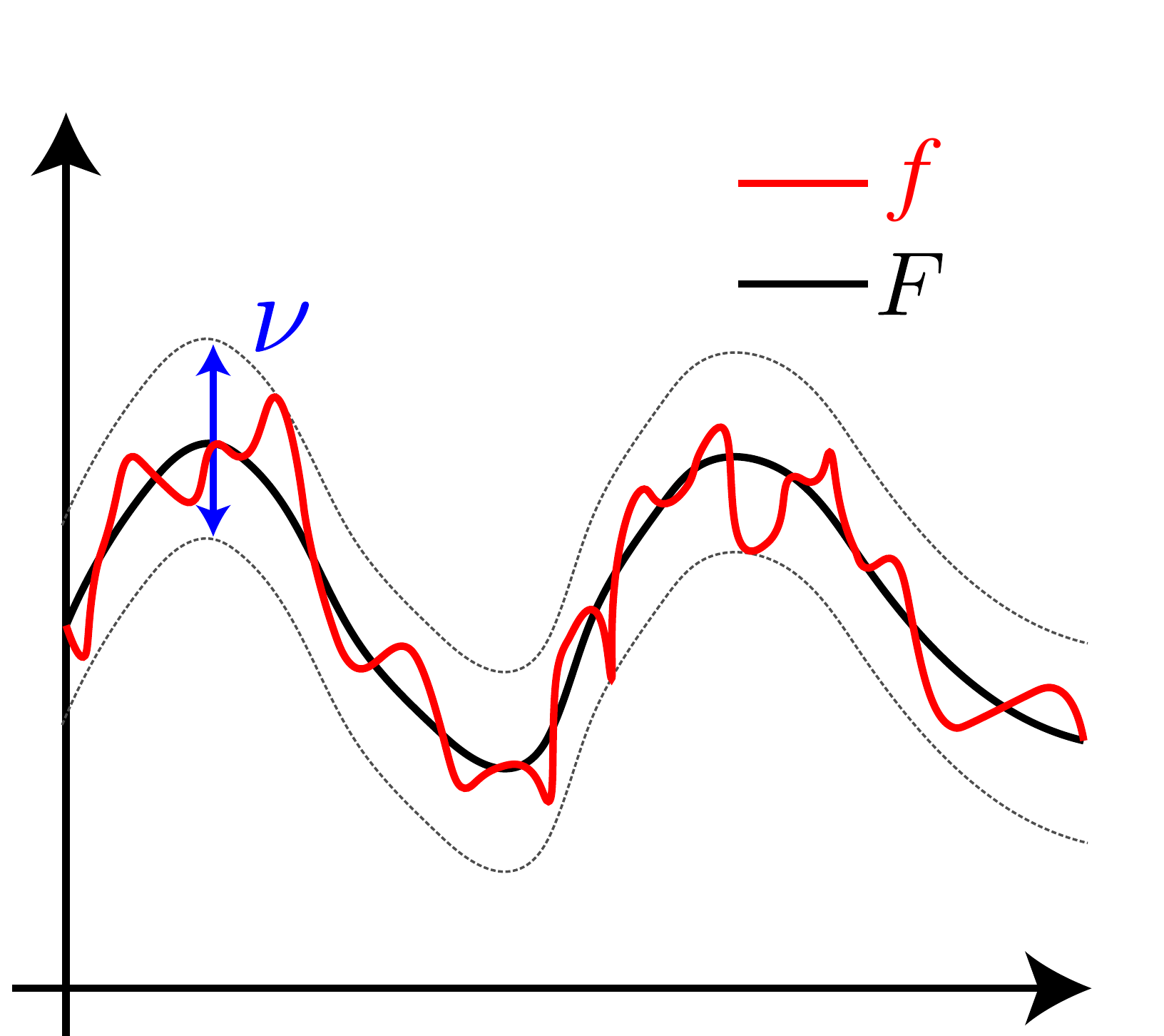}%
	\includegraphics[width=0.36\textwidth, trim=0 30 30 10, clip]{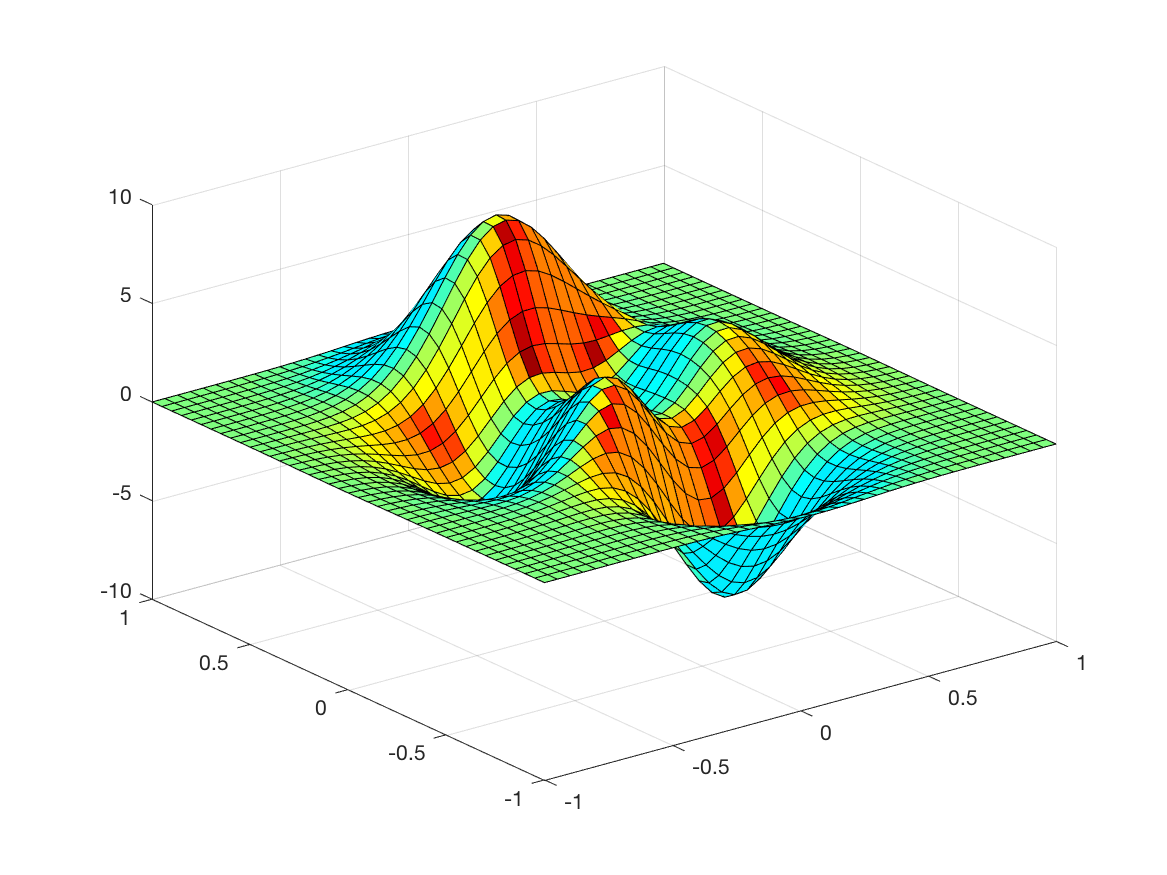}%
	\includegraphics[width=0.36\textwidth, trim=0 30 30 10, clip]{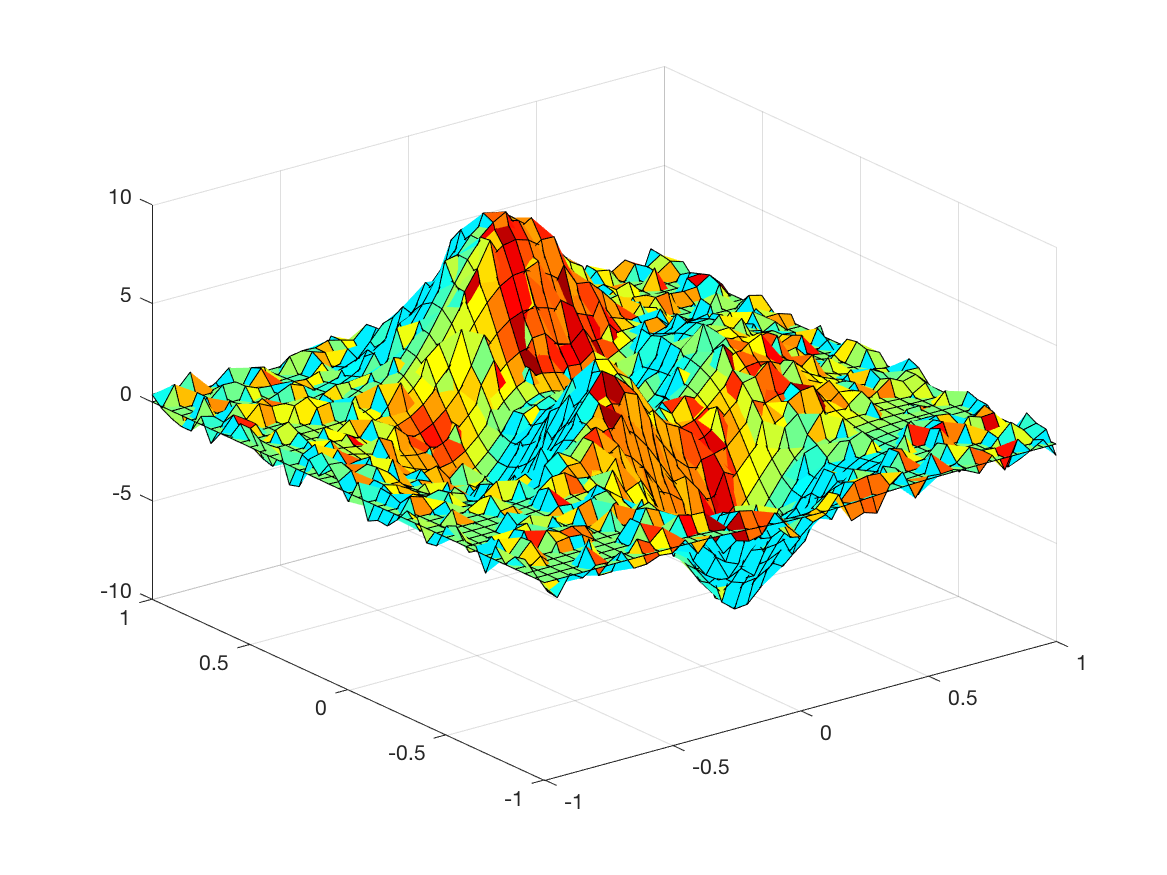}
	\caption{a) Function error $\nu$; b) Population risk vs empirical risk} \label{fig:intro}
\end{figure}
%\cnote{make f and F color in Figure 1 more distinguishable?}

Much of the previous work on this problem studies relatively small values of $\nu$,
leading to ``shallow'' local minima, and applies relatively large amounts of noise,
through algorithms such as simulated annealing \citep{belloni2015escaping} 
and stochastic gradient Langevin dynamics (SGLD) \citep{zhang2017hitting}. 
While such ``large-noise algorithms'' may be justified if the goal is to
approach a stationary distribution, it is not clear that such large 
levels of noise is necessary in the optimization setting in order to escape 
shallow local minima.  %Focusing on methods that are specifically designed for optimization, 
The best existing result for the setting of nonconvex 
$F$ requires the error $\nu$ to be smaller than $O(\epsilon^2/d^8)$, 
where $\epsilon$ is the precision of the optimization guarantee (see Definition 
\ref{def:exactSOSP}) and $d$ is the problem dimension~\citep{zhang2017hitting} 
(see Figure~\ref{fig:contri}).  A fundamental question is whether algorithms
exist that can tolerate a larger value of $\nu$, which would imply that they can 
escape ``deeper'' local minima.  In the context of empirical risk minimization, 
such a result would allow fewer samples to be taken while still providing a
strong guarantee on avoiding local minima.

We thus focus on the two central questions: \textbf{(1) Can a simple, 
optimization-based algorithm avoid shallow local minima despite the lack 
of ``large noise''?  (2) Can we tolerate larger error $\nu$ in the 
optimization setting, thus escaping ``deeper'' local minima?  
What is the largest error that the best algorithm can tolerate?} 

In this paper, we answer both questions in the affirmative, establishing
optimal dependencies between the error $\nu$ and the precision of a 
solution~$\epsilon$. We propose a simple algorithm based on 
SGD (Algorithm~\ref{algo:PSGD}) that is guaranteed to find an approximate 
local minimum of $F$ efficiently if $\nu \le O(\epsilon^{1.5}/d)$, 
thus escaping all saddle points of $F$ and all additional local minima 
introduced by $f$. Moreover, we provide a matching lower bound (up to logarithmic factors) for all algorithms making a polynomial number of queries of $f$.
% showing whenever $\nu \ge \tilde{\Omega}(\epsilon^{1.5}/d)$, no 
%algorithm can successfully achieve above objective within a polynomial 
%number of queries to $f$ in the worst case, 
The lower bound shows that 
our algorithm achieves the optimal tradeoff between $\nu$ and $\epsilon$, 
as well as the optimal dependence on dimension $d$.  We also consider 
the information-theoretic limit for identifying an approximate local 
minimum of $F$ regardless of the number of queries. We give a sharp 
information-theoretic threshold: $\nu = \Theta(\epsilon^{1.5})$ (see 
Figure~\ref{fig:contri}). 

As a concrete example of the application to minimizing population risk, we show 
that our results can be directly used to give sample complexities for learning 
a ReLU unit, whose empirical risk is nonsmooth while the population risk is smooth 
almost everywhere.

\vspace{-5mm}

\begin{figure}[t]
\centering
\includegraphics[width=.9\textwidth]{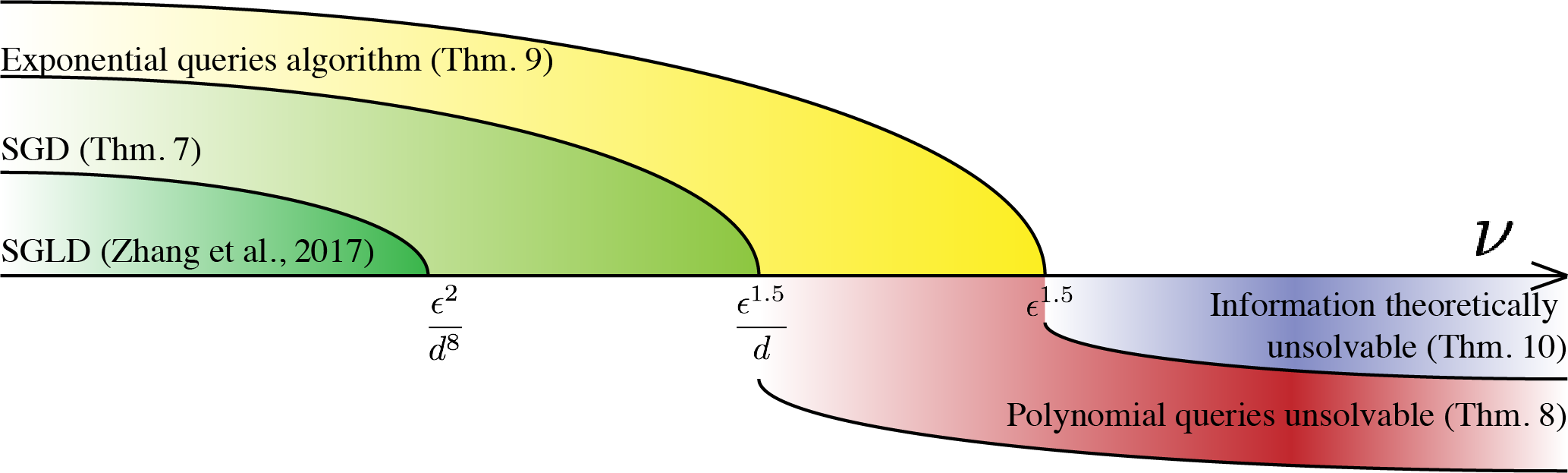}
\caption{Complete characterization of error $\nu$ vs accuracy $\epsilon$ and dimension $d$.} \label{fig:contri}
\end{figure}

\subsection{Related Work}
\label{sec:related}

% \begin{table}[t]
%   \begin{center}
%     {\renewcommand{\arraystretch}{1.3}
%     \begin{tabular}  {c | c | c}
%        \toprule
% \textbf{Algorithm}  &  $\nu$ \textbf{Upper Bound}  & $\nu$ \textbf{Lower Bound} \\ 
% \midrule 
% SGLD \citep{zhang2017hitting} & $O(\epsilon^2/d^8)$ & - \\
% \textbf{SGD (This Work)} & $O(\epsilon^{1.5}/d)$ & $\tilde{\Omega}(\epsilon^{1.5}/d)$ \\
% \bottomrule
% \end{tabular}
%       \caption{Upper Bound and Lower Bound}
%       \label{tab:main}
%     }
%   \end{center}
%   % \praneeth{I think it will be cleaner to make the dependence on smoothness paramters explicit here.}
%   \vspace{-4ex}
% \end{table}

A number of other papers have examined the problem of optimizing a target function 
$F$ given only function evaluations of a function $f$ that is pointwise close to $F$. 
\citet{belloni2015escaping} proposed an algorithm based on simulated annealing. 
The work of \citet{risteski16alg} and  \citet{singer15information} discussed
lower bounds, though only for the setting in which the target function $F$ is convex.
For nonconvex target functions $F$, \citet{zhang2017hitting} studied the problem of 
finding approximate local minima of $F$, and proposed an algorithm based on 
Stochastic Gradient Langevin Dynamics (SGLD) \citep{welling11bayesian},
with maximum tolerance for function error $\nu$ scaling as $O(\epsilon^2/d^8)$\footnote{The difference between the scaling for $\nu$ asserted here and the $\nu = O(\epsilon^2)$ 
claimed in \citep{zhang2017hitting} is due to difference in assumptions. 
In our paper we assume that the Hessian is Lipschitz with respect to the 
standard spectral norm; \citet{zhang2017hitting} make such an assumption with 
respect to nuclear norm.}. Other than difference in algorithm style and $\nu$ tolerance as shown in Figure \ref{fig:contri}, we also note that we do not require regularity assumptions on top of smoothness, which are inherently required by the MCMC algorithm proposed in \citet{zhang2017hitting}.
Finally, we note that in parallel, \citet{kleinberg18an} solved a similar problem using SGD under the assumption that $F$ is one-point convex.

% SGLD is a sampling-based algorithm which adds 
% Gaussian noise to the iterates, and converges asymptotically to a stationary distribution. It requires many regularity assumptions in addition to the smoothness assumptions 
% made in this paper.  Moreover, the noise, which arises from sampling Brownian 
% motion, is large. In contrast, our algorithm based on stochastic gradient descent only adds a small amount of perturbation
% and asymptotically converges to a neighborhood around the local minimum. 
% Although the two algorithms both use gradients with noise, the aforementioned differences are critical for getting the optimal dependencies between the function error $\nu$ and the accuracy $\epsilon$.

% In a parallel work, \citet{kleinberg18an} studied the convergence of SGD to a global minimum under the assumption of one-point convexity. 
%\cite{Hazan2016graduated} used homotopy to \cnote{Shorten this comparison, and add Yang's work (parallel work, 1 point convex) and Hazan's graduating optimization, homotopy?}

Previous work has also studied the relation between the landscape of empirical risks and the landscape of population risks for nonconvex functions. \citet{mei2016landscape} examined a special case where the individual loss functions $L$ are also smooth, which under some assumptions implies uniform convergence of the gradient and Hessian of the empirical risk to their population versions. % this allow two landscape share very similar properties so that standard optimization algorithms can be used.
 \citet{loh2013regularized} showed for a restricted class of nonconvex losses that even though many local minima of the empirical risk exist, they are all close to the global minimum of population risk. %Empirical work by \citet{choromanska15} conjectured that local minima found by SGD have small population risk for certain deep neural networks .

Our work builds on recent work in nonconvex optimization, in particular, 
results on escaping saddle points and finding approximate local minima. 
% there are various guarantees for finding stationary points of smooth functions. 
Beyond the classical result by \citet{nesterov1998introductory} for finding 
first-order stationary points by gradient descent, recent work has given
guarantees for escaping saddle points by gradient descent \citep{jin17escape} 
and stochastic gradient descent \citep{ge2015escaping}. 
\citet{agarwal17finding} and \citet{carmon16accelerated} established faster 
rates using algorithms that make use of Nesterov's accelerated gradient descent 
in a nested-loop procedure \citep{nesterov83accel}, and \citet{jin17accel} have
established such rates even without the nested loop.
% for escaping saddle points; in follow-up work, \citet{jin17accel} showed that a simple variant of Nesterov's accelerated GD \citep{nesterov83accel} attains the same rates. 
There have also been empirical studies on various types of local 
minima \citep[e.g.][]{keskar2016large, dinh2017sharp}.

Finally, our work is also related to the literature on zero-th order optimization 
or more generally, bandit convex optimization. Our algorithm uses 
function evaluations to construct a gradient estimate and perform SGD, 
which is similar to standard methods in this community \citep[e.g.,][]{flaxman05online, 
agarwal10optimal,duchi15optimal}. Compared to first-order optimization, however,
the convergence of zero-th order methods is typically much slower, depending 
polynomially on the underlying dimension even in the convex setting \citep{shamir13bandit}. 
Other derivative-free optimization methods include simulated annealing \citep{Kirkpatrick671} and evolutionary algorithms \citep{rechenberg73}, whose convergence guarantees are less clear.

\section{Preliminaries}\label{sec:prelim}

\paragraph{Notation}
We use bold lower-case letters to denote vectors, as in $\x, \y, \z$. We use $\norm{\cdot}$ to denote the $\ell_2$ norm of vectors and spectral norm of matrices. For a matrix, $\lambda_{\min}$ denotes its smallest eigenvalue. For a function $f:\R^d \to \R$, $\grad f $ and $\hess f$ denote its gradient vector and Hessian matrix respectively. We also use $\norm{\cdot}_\infty$ on a function $f$ to denote the supremum of its absolute function value over entire domain, $\sup_{\x \in \R^d} |f|$. We use $\mathbb{B}_0(r)$ to denote the $\ell_2$ ball of radius $r$ centered at $0$ in~$\R^d$. 
We use %notation $O(\cdot), \Theta(\cdot), \Omega(\cdot) $ to hide only absolute constants which
%do not depend on any problem parameter, and 
notation $\tilde{O}(\cdot),\tilde{\Theta}(\cdot), \tilde{\Omega}(\cdot) $ to hide only absolute constants and
poly-logarithmic factors. A multivariate Gaussian distribution with mean $\boldsymbol{0}$ and covariance $\sigma^2$ in every direction is denoted as $\mathcal{N}(\boldsymbol{0}, \sigma^2\boldsymbol{I})$. Throughout the paper, we say ``polynomial number of queries'' to mean that the number of queries depends polynomially on all problem-dependent parameters. % to denote the multivariate Gaussian distribution with mean $\boldsymbol{\mu}$ and covariance matrix $\boldsymbol{\Sigma}$.

\vspace{-2mm}

\paragraph{Objectives in nonconvex optimization} Our goal is to find a point that has zero gradient and positive semi-definite Hessian, thus escaping saddle points.  We formalize this idea as follows.
\begin{definition}\label{def:exactSOSP}
	$\x$ is called a \textbf{second-order stationary point} (SOSP) or \textbf{approximate local minimum} of a function $F$ if
	\begin{equation*}
		\norm{\grad F(\x)} = 0 \text{~and~} \lambda_{\min}(\hess F(\x)) 
		\geq 0.
	\end{equation*}
\end{definition}

We note that there is a slight difference between SOSP and local 
minima---an SOSP as defined here
% We call the above an \emph{approximate} local minimum instead of local minimum because it 
does not preclude higher-order saddle points, which themselves can be NP-hard to escape from \citep{AnandkumarG16}.

Since an SOSP is characterized by its gradient and Hessian, and since 
convergence of algorithms to an SOSP will depend on these derivatives
in a neighborhood of an SOSP, it is necessary to impose smoothness
conditions on the gradient and Hessian.  A minimal set of conditions
that have become standard in the literature are the following.
\begin{definition}\label{def:gradLip}
A  function $F$ is \textbf{$\ell$-gradient Lipschitz} if 
$~ \forall \x, \y ~~\norm{\grad F(\x) - \grad F(\y)} \le \ell \norm{\x - \y}.$
\end{definition}

\begin{definition}\label{def:HessLip}
A function $F$ is \textbf{$\rho$-Hessian Lipschitz} if 
$~\forall \x, \y ~~\norm{\hess F(\x) - \hess F(\y)} \le \rho \norm{\x - \y}.$
\end{definition}

\noindent
Another common assumption is that the function is bounded.
% For technical algorithms (see theorem \ref{thm:basic}) require various smoothness conditions, as well as boundedness:

%\cnote{motivate smoothness from the reason we study the property gradient and Hessian}

%In this paper, we study the optimization of bounded functions that have smooth first and second order derivatives.
\begin{definition}\label{def:bounded}
A function $F$ is \textbf{$B$-bounded} if for any $\x$ that $|F(\x)| \le B$.
\end{definition}

%Intuitively, these gradient and Hessian Lipschitz properties ensure that the loss function $F$ is well behaved so that algorithms based on gradient or Hessian can make steady progress.
%Such smoothness guarantees imply that the gradient and Hessian cannot change too rapidly.

For any finite-time algorithm, we cannot hope to find an \emph{exact} SOSP. 
Instead, we can define $\epsilon$-approximate SOSP that satisfy relaxations 
of the first- and second-order optimality conditions.  Letting $\epsilon$
vary allows us to obtain rates of convergence.
% instead, existing algorithms find an $\epsilon$-approximate SOSP, where $\epsilon$ corresponds to how close the point is to being a SOSP.
\begin{definition}\label{def:SOSP}
$\x$ is an $\epsilon$-\textbf{second-order stationary point} ($\epsilon$-SOSP) of a $\rho$-Hessian Lipschitz function $F$ if
\begin{equation*}
\norm{\grad F(\x)} \le \epsilon \text{~and~} \lambda_{\min}(\hess F(\x)) \ge -\sqrt{\rho\epsilon}.
\end{equation*}
\end{definition}
% The requirements on gradient and Hessian are immediate relaxations of the first and second order optimality conditions. 

Given these definitions, we can ask whether it is possible to find 
an $\epsilon$-SOSP in polynomial time under the Lipchitz properties.
Various authors have answered this question in the affirmative.
\begin{theorem}
% [General nonconvex optimization result]
\cite[e.g.][]{carmon16accelerated,agarwal17finding,jin17escape} \label{thm:basic}
If the function $F:\R^d\to \R$ is $B$-bounded, $l$-gradient Lipschitz and $\rho$ Hessian Lipschitz, given access to the gradient (and sometimes Hessian) of $F$, it 
is possible to find an $\epsilon$-SOSP in $\mbox{poly}(d, B,l,\rho,1/\epsilon)$ time.
\end{theorem}

% These algorithms may use Hessian information, such as Hessian-vector products \citep{agarwal17finding}, or carefully apply small perturbations to the gradient \citep[e.g. noisy stochastic gradient, perturbed gradient descent in ][]{ge2015escaping, jin17escape} to escape saddle points.
%\rnote{Ideally I was hoping to define SOSP first but I then realized that it contains a $\rho$...}

%The goal of non-convex optimization is to find higher order stationary points, in particular, second-order stationary points. An $\epsilon$-approximate second-order stationary point is defined as follows.

% !TeX root = main.tex 
\vspace{-2mm}

\section{Main Results}
\vspace{-2mm}

In the setting we consider, there is an unknown function $F$ (the population risk) that has regularity properties (bounded, gradient and Hessian Lipschitz). However, we only have access to a function $f$ (the empirical risk) that may not even be everywhere differentiable. The only information we use is that $f$ is pointwise close to $F$. More precisely, we assume

\begin{assumption} \label{assump} We assume that the function pair 
($F:\R^d\to\R, f:\R^d\to\R$) satisfies the following properties:
\begin{enumerate}
    \item $F$ is $B$-bounded, $\ell$-gradient Lipschitz, $\rho$-Hessian Lipschitz.
    \vspace{-2mm}
    \item $f, F$ are $\nu$-pointwise close; i.e., $\|F - f\|_\infty \le \nu$.
\end{enumerate}
\end{assumption}

As we explained in Section~\ref{sec:prelim}, our goal is to find second-order stationary points of $F$ given only function value access to $f$. More precisely:

\begin{prob}\label{problem}
	Given a function pair ($F, f$) that satisfies Assumption \ref{assump}, find an $\epsilon$-second-order stationary point of $F$ with only access to values of $f$.
\end{prob}

The only way our algorithms are allowed to interact with $f$ is to query a point $\x$, and obtain a function value $f(\x)$. This is usually called a {\em zero-th order} oracle in the optimization literature. In this paper we give tight upper and lower bounds for the dependencies between  $\nu$, $\epsilon$ and $d$, both  for algorithms with polynomially many queries and in the information-theoretic limit.

\subsection{Optimal algorithm with polynomial number of queries}

% Our algorithm (Algorithm~\ref{algo:PSGD}) is a modification of stochastic gradient descent algorithm. 
There are three main difficulties in applying stochastic gradient descent to Problem~\ref{problem}: (1) in order to converge to a second-order stationary point of $F$, the algorithm must avoid being stuck in saddle points; (2) the algorithm does not have access to the gradient of $f$; (3) there is a gap between the observed $f$ and the target $F$, which might introduce non-smoothness or additional local minima. %introduced by $f$.
The first difficulty was addressed in \cite{jin17escape} by perturbing the iterates in a small ball; % using the uniform distribution over the ball $\mathbb{B}_0(r)$ centered at $\mathbf{0}$ with small radius $r$; 
this pushes the iterates away from any potential saddle points. For the latter two difficulties, we apply Gaussian smoothing to $f$ and use $\z[f(\x+\z)- f(\x)]/\sigma^2$ ($\z \sim \mathcal{N}(0,\sigma^2 \I)$) as a stochastic gradient estimate. 
This estimate, which only requires function values of $f$, is well known 
in the zero-th order optimization literature \citep[e.g.][]{duchi15optimal}. 
For more details, see Section~\ref{sec:polyupper}.

\begin{algorithm}[t]
	\caption{Zero-th order Perturbed Stochastic Gradient Descent (ZPSGD)}\label{algo:PSGD}
	\begin{algorithmic}
		\renewcommand{\algorithmicrequire}{\textbf{Input: }}
		\renewcommand{\algorithmicensure}{\textbf{Output: }}
		\REQUIRE $\x_0$, learning rate $\eta$, noise radius $r$, mini-batch size $m$.
		\FOR{$t = 0, 1, \ldots, $}
		\STATE sample $(\z^{(1)}_t, \cdots, \z^{(m)}_t) \sim \mathcal{N}(0,\sigma^2 \I)$
		\STATE $\g_t(\x_t)  \leftarrow  \sum_{i=1}^m \z^{(i)}_t[f(\x_t+\z^{(i)}_t)- f(\x_t)]/(m\sigma^2)$
		\STATE $\x_{t+1} \leftarrow \x_t - \eta (\g_t(\x_t) + \xi_t), \qquad \xi_t \text{~uniformly~} \sim \mathbb{B}_0(r)$
		\ENDFOR
		\STATE \textbf{return} $\x_T$
	\end{algorithmic}
\end{algorithm}

In short, our algorithm (Algorithm~\ref{algo:PSGD}) is a variant of SGD, which uses $\z[f(\x+\z)- f(\x)]/\sigma^2$ as the gradient estimate (computed over mini-batches), and adds isotropic perturbations. Using this algorithm, we can achieve the following trade-off between $\nu$ and $\epsilon$. 
%Formally, we have the following convergence guarantee.

\begin{theorem}[Upper Bound (ZPSGD)] \label{thm:upperbound_informal}
	Given that the function pair ($F, f$) satisfies Assumption \ref{assump} with $\nu \le O(\sqrt{\epsilon^3/\rho}\cdot(1/d))$, then for any $\delta >0$, with smoothing parameter $\sigma=\Theta(\sqrt{\epsilon/(\rho d)})$, learning rate $\eta=1/\ell$, perturbation $r = \tilde{\Theta}(\epsilon)$, and mini-batch size $m=\poly(d, B, \ell,\rho,1/\epsilon, \log (1/\delta))$, ZPSGD will find an $\epsilon$-second-order stationary point of $F$ with probability $1-\delta$, in $\poly(d, B, \ell,\rho,1/\epsilon, \log (1/\delta))$  number of queries.
\end{theorem}
% \cnote{Is proper choice of $\eta, r, m$ OK? or should we write it out?}

Theorem \ref{thm:upperbound_informal} shows that assuming a small enough function error $\nu$, ZPSGD will solve Problem~\ref{problem} within a number of queries that is polynomial in all the problem-dependent parameters. The tolerance on function error $\nu$ varies inversely with the number of dimensions, $d$. This rate is in fact optimal for all polynomial queries algorithms.
In the following result, we show that the $\epsilon, \rho,$ and $d$ dependencies in function difference $\nu$ are tight up to a logarithmic factors in $d$. %Informally, our lower bound establishes the existence of a function pair satisfying Assumption \ref{assump} with $\nu$ the same order as in Theorem \ref{thm:upperbound_informal} (up to log factor in $d$) such that no polytime algorithm can find an SOSP of $F$. 

\begin{restatable}[Polynomial Queries Lower Bound]{theorem}{thmlowerbnd} \label{thm:lowerbound_informal}
For any $B >0, \ell >0, \rho > 0$ there exists $\epsilon_0 = \Theta(\min\{\ell^2/\rho, (B^2 \rho/d^2)^{1/3}\})$ such that for any $\epsilon \in (0, \epsilon_0]$,
there exists a function pair ($F, f$) satisfying Assumption \ref{assump} with $\nu = \tilde{\Theta}(\sqrt{\epsilon^3/\rho} \cdot (1/d))$, so that any algorithm that only queries a polynomial number of function values of $f$ will fail, with high probability, to find an $\epsilon$-SOSP of $F$. % given only polynomial number of zero-th order queries of $f$.
\end{restatable}

This theorem establishes that for any $\rho, \ell, B$ and any $\epsilon$ small enough, we can construct a randomized `hard' instance ($F, f$) such that any (possibly randomized) algorithm with a polynomial number of queries will fail to find an $\epsilon$-SOSP of $F$ with high probability. Note that the error $\nu$ here is only a poly-logarithmic factor larger than the requirement for our algorithm. %This implies that, if we allow the function difference $\nu$ to be larger by just a polylog factor in $d$, no algorithm can solve Problem~\ref{problem} on every instance of ($F, f$) in polynomial queries with probability greater than a constant. 
In other words, the guarantee of our Algorithm~\ref{algo:PSGD} in Theorem \ref{thm:upperbound_informal} is optimal up to a logarithmic factor.

% \cnote{Change statement of $\ell$ and $B$, current one looks a bit weird...}

\subsection{Information-theoretic guarantees}

If we allow an unlimited number of queries, % for the algorithms, %consider a larger class of algorithms than those running in polynomial queries, 
we can show that the upper and lower bounds on the function error tolerance $\nu$ no longer depends on the problem dimension $d$. That is, Problem~\ref{problem} exhibits a statistical-computational gap---polynomial-queries algorithms are unable to achieve the information-theoretic limit. We first state that an algorithm (with exponential queries) is able to find an $\epsilon$-SOSP of $F$ despite a much larger value of error $\nu$. The basic algorithmic idea is that an $\epsilon$-SOSP must exist within some compact space, such that once we have a subroutine that approximately computes the gradient and Hessian of $F$ at an arbitrary point, we can perform a grid search over this compact space (see Section \ref{app:exp} for more details):

\begin{theorem}\label{thm:exp_upperbound}
		There exists an algorithm so that if the function pair ($F, f$) satisfies Assumption \ref{assump} with $\nu \le O(\sqrt{\epsilon^3/\rho})$ and $\ell > \sqrt{\rho \epsilon}$, then the algorithm will find an $\epsilon$-second-order stationary point of $F$ with an exponential number of queries.
\end{theorem}

%
%\begin{restatable}[]{theorem}{expupperbound} \label{thm:exp_upperbound}
%	There exists an algorithm so that if the function pair ($F, f$) satisfies Assumption \ref{assump} with $\nu \le O(\sqrt{\epsilon^3/\rho})$ and $\ell > \sqrt{\rho \epsilon}$, then the algorithm will find an $\epsilon$-second-order stationary point of $F$ with an exponential number of queries.
%\end{restatable}

We also show a corresponding information-theoretic lower bound that prevents any algorithm from even identifying a second-order stationary point of $F$. This completes the characterization of function error tolerance $\nu$ in terms of required accuracy $\epsilon$. % and problem dimension $d$.

\begin{theorem} \label{thm:exp_lowerbound}
For any $B >0, \ell >0, \rho > 0$, there exists $\epsilon_0 = \Theta(\min\{\ell^2/\rho, (B^2 \rho/d)^{1/3}\})$ such that for any $\epsilon \in (0, \epsilon_0]$
there exists a function pair ($F, f$) satisfying Assumption \ref{assump} with $\nu = O(\sqrt{\epsilon^3/\rho})$, so that any algorithm will fail, with high probability, to find an $\epsilon$-SOSP of $F$.
\end{theorem}

\subsection{Extension: Gradients pointwise close}
We may extend our algorithmic ideas to solve the problem of optimizing an unknown smooth function $F$ when given only a gradient vector field $\g: \R^d \to \R^d$ that is pointwise close to the gradient $\grad F$. Specifically, we answer the question: what is the error in the gradient oracle that we can tolerate to obtain optimization guarantees for the true function $F$? We observe that our algorithm's tolerance on gradient error is much better compared to Theorem \ref{thm:upperbound_informal}. Details can be found in Appendix \ref{app:extension_grad} and \ref{app:gradients_pw_close}.

% \textbf{Assumptions:}
% \begin{enumerate}
% 	\item $F$ is $\ell$-gradient Lipschitz, i.e. $\norm{\grad F(\x) - \grad F(\y)} \le \ell \norm{\x - \y}$. 
% 	\item $F$ is $\rho$-Hessian Lipschitz, i.e. $\norm{\hess F(\x) - \hess F(\y)} \le \rho \norm{\x - \y}$. 
% 	\item $\grad f$ exists, and $F$ and $f$'s gradients are $\tilde{\nu}$-pointwise close: $\sup_{x\in \R^d} \norm{\grad f(x) - \grad F(x)} \leq \tilde{\nu}$
% 	\item $f$ is $L$-Lipschitz
% 	\item (optional) $f$ if $\tilde{\ell}$-gradient Lipschitz
% 	\end{enumerate}
	
% 	\subsubsection{Result}
	
% 	\textbf{Algorithm}: run Perturbed Stochastic Gradient Descent (PSGD) with following stochastic gradient oracle:
% 	\begin{equation*}
% 	\g(\x;\z) = \grad f(\x+ \z), \text{~where~} \z \sim \mathcal{N}(0,\sigma^2 \I)
% 	\end{equation*}
% 	\begin{theorem}
% 		%Assume $\tilde{\nu} \le \sqrt{\frac{\epsilon^3}{\rho d^3}}$, then by choosing $\sigma = \sqrt{\frac{\epsilon}{\rho d}}$, PSGD will find $\epsilon$-second-order stationary point of $F$ in following iterations w.h.p:
% 		% \begin{equation*}
% 		%\tilde{O}\left( \frac{\rho d B}{\epsilon^3} \left( \frac{\ell B^2}{\epsilon^2} + \frac{L^2}{\sqrt{\rho\epsilon}}\right)  \right)
% 		% \end{equation*}
% 		Our analysis can give polytime algorithm work out the case $\tilde{\nu} \le \frac{\epsilon}{\sqrt{d}}$.
% 		\end{theorem}
% 		%
% 		%\begin{remark}
% 		%	Removing Assumption 5, we can still get above polynomial queries guarantee, but incur a lot poly factors of $d$.
% 		%\end{remark}

% !TeX root = main.tex 
\vspace{-3mm}
\section{Overview of Analysis}\label{sec:overview}
\vspace{-2mm}

In this section we present the key ideas underlying our theoretical results. We will focus on the results for algorithms that make a polynomial number of queries (Theorems~\ref{thm:upperbound_informal} and \ref{thm:lowerbound_informal}). %The information-theoretic lower bound (Theorem~\ref{thm:exp_lowerbound}) is a simple adaptation of Theorem~\ref{thm:lowerbound_informal} using the same construction, and the results for exponential numbers of queries is based on a grid search---the details are deferred to Appendix~\ref{app:exp}.
\vspace{-2mm}

\subsection{Efficient algorithm for Problem~\ref{problem}}\label{sec:polyupper}

\vspace{-2mm}

We first argue the correctness of Theorem \ref{thm:upperbound_informal}. As discussed earlier, there are two key ideas in the algorithm: Gaussian smoothing and perturbed stochastic gradient descent. Gaussian smoothing allows us to transform the (possibly non-smooth) function $f$ into a smooth function $\tilde{f}_\sigma$ that has similar second-order stationary points as $F$; at the same time, it can also convert function evaluations of $f$ into a stochastic gradient of $\tilde{f}_\sigma$. We can use this stochastic gradient information to find a second-order stationary point of $\tilde{f}_\sigma$, which by the choice of the smoothing radius is guaranteed to be an approximate second-order stationary point of $F$.

First, we introduce Gaussian smoothing, which perturbs the current point $\x$ using a multivariate Gaussian and then takes an expectation over the function value.
% Since we only have access to the function values of $f$, to perform an algorithm like gradient descent, we need a way to get gradient information. To do that, we define a smoothed function $\tilde{f}_\sigma$ that is the average of $f$ over a Gaussian noise:
% \cnote{change this paragraph}
\begin{definition}[Gaussian smoothing]\label{def:smoothed_f}
		Given $f$ satisfying assumption \ref{assump}, define its Gaussian smoothing as $\tilde{f}_\sigma(\x) = \E_{\z\sim \mathcal{N}(0,\sigma^2 \I)}[f(\x+\z)]$. The parameter $\sigma$ is henceforth called the smoothing radius.
\end{definition}

In general $f$ need not be smooth or even differentiable, but its Gaussian smoothing $\tilde{f}_\sigma$ will be a differentiable function.
Although it is in general difficult to calculate the exact smoothed function $\tilde{f}_\sigma$, it is not hard to give an unbiased estimate of function value and gradient of $\tilde{f}_\sigma$: %The latter is captured by the following claim:
\begin{lemma}\cite[e.g.][]{duchi15optimal}\label{lem:gensgd} Let $\tilde{f}_\sigma$ be the Gaussian smoothing of $f$ (as in Definition~\ref{def:smoothed_f}), the gradient of $\tilde{f}_\sigma$ can be computed as
$
~\grad \tilde{f}_\sigma = \frac{1}{\sigma^2} \E_{\z\sim \mathcal{N}(0,\sigma^2 \I)}[(f(\x+\z) - f(\x))\z].
$
%\rnote{Please check if the sign is correct}
\end{lemma}

Lemma \ref{lem:gensgd} allows us to query the function value of $f$ to get an unbiased estimate of the gradient of $\tilde{f}_\sigma$. This stochastic gradient is used in Algorithm~\ref{algo:PSGD} to find a second-order stationary point of $\tilde{f}_\sigma$.

To make sure the optimizer is effective on $\tilde{f}_\sigma$ and that guarantees on $\tilde{f}_\sigma$ carry over to the target function $F$, we need two sets of properties: the smoothed function $\tilde{f}_\sigma$ should be gradient and Hessian Lipschitz, and at the same time should have gradients and Hessians close to those of the true function $F$. These properties are summarized in the following lemma:

%
% given two intermediate lemmas, \ref{lem:smoothing} and \ref{lem:stocgrad}. A key idea behind ZPSGD is Gaussian smoothing. We may think of ZPSGD as performing perturbed SGD on the smoothed version of function $f$, denoted as $\tilde{f}_\sigma $, obtained by convolving the $f$ with a multivariate Gaussian. As long as the $\tilde{f}_\sigma $ is close to $F$ in gradient and Hessian, we can choose the smoothing parameter $\sigma$ such that any $\epsilon$-SOSP of $\tilde{f}_\sigma $ is also an $O(\epsilon)$-SOSP of $F$.
%
%The first lemma defines the smoothed version of function $f$ and establishes its properties, namely that it is gradient and Hessian Lipschitz, and that its gradients and Hessian matrix are close to those of $F$. 
%
%The second lemma shows that the stochastic gradients $\g$ (as computed in \ref{algo:PSGD}) are subgaussian and their expectation is the gradient of $\tilde{f}_\sigma $
\begin{lemma}[Property of smoothing]
		\label{lem:smoothing}
		Assume that the function pair ($F, f$) satisfies Assumption \ref{assump}, and let $\tilde{f}_\sigma(\x)$ be as given in definition \ref{def:smoothed_f}. Then, the following holds
		\begin{enumerate}
			\item $\tilde{f}_\sigma(\x)$ is $O(\ell + \frac{\nu}{\sigma^2})$-gradient Lipschitz and $O(\rho + \frac{\nu}{\sigma^3})$-Hessian Lipschitz.
			\item %The differences in gradient and Hessian are small
			$\norm{\grad \tilde{f}_\sigma(\x) - \grad F(\x)} \leq  O(\rho d \sigma^2+  \frac{\nu}{\sigma})$ and
			$\norm{\hess \tilde{f}_\sigma(\x) - \hess F(\x)} \leq O(\rho\sqrt{d}\sigma + \frac{\nu}{\sigma^2})$.
		\end{enumerate}
\end{lemma}
%\begin{restatable}{lemma}{lemmasmoothing}
%	\label{lem:smoothing}
%	Assume that the function pair ($F, f$) satisfies Assumption \ref{assump}, and let $\tilde{f}_\sigma(\x)$ be as given in definition \ref{def:smoothed_f}. Then, the following holds
%	\begin{enumerate}
%		\item $\tilde{f}_\sigma(\x)$ is $O(\ell + \frac{\nu}{\sigma^2})$-gradient Lipschitz and $O(\rho + \frac{\nu}{\sigma^3})$-Hessian Lipschitz.
%		\item %The differences in gradient and Hessian are small
%		$\norm{\grad \tilde{f}_\sigma(\x) - \grad F(\x)} \leq  O(\rho d \sigma^2+  \frac{\nu}{\sigma})$ and
%		$\norm{\hess \tilde{f}_\sigma(\x) - \hess F(\x)} \leq O(\rho\sqrt{d}\sigma + \frac{\nu}{\sigma^2})$.
%	\end{enumerate}
%\end{restatable}

The proof is deferred to Appendix \ref{app:upper}.
%We now briefly discuss the intuition behind the propositions in lemma \ref{lem:smoothing}. 
Part (1) of the lemma says that the gradient (and Hessian) Lipschitz constants of $\tilde{f}_\sigma$ are similar to the gradient (and Hessian) Lipschitz constants of $F$ up to a term involving the function difference $\nu$ and the smoothing parameter $\sigma$. This means as $f$ is allowed to deviate further from $F$, we must smooth over a larger radius---choose a larger $\sigma$---to guarantee the same smoothness as before. On the other hand, part (2) implies that choosing a large $\sigma$ increases the upper bound on the gradient and Hessian \emph{difference} between $\tilde{f}_\sigma$ and $F$. Smoothing is a form of local averaging, so choosing a too-large radius will erase information about local geometry. The choice of $\sigma$ must strike the right balance between making $\tilde{f}_\sigma$ smooth (to guarantee ZPSGD finds a $\epsilon$-SOSP of $\tilde{f}_\sigma$ ) and keeping the derivatives of $\tilde{f}_\sigma$ close to those of $F$ (to guarantee any $\epsilon$-SOSP of $\tilde{f}_\sigma$ is also an $O(\epsilon)$-SOSP of $F$). In Appendix \ref{app:proof_upper}, we show that this can be satisfied by choosing $\sigma = \sqrt{\epsilon/(\rho d)}$.

\paragraph{Perturbed stochastic gradient descent}

In ZPSGD, we use the stochastic gradients suggested by Lemma~\ref{lem:gensgd}. % rather than compute the exact gradients of  $\tilde{f}_\sigma$. 
Perturbed Gradient Descent (PGD) \citep{jin17escape} was shown to converge to a second-order stationary point. Here we use a simple modification of PGD that relies on batch stochastic gradient. % Intuitively, when the batch size is large enough, stochastic gradient descent behaves similarly to a gradient descent step. 
In order for PSGD to converge, we require that the stochastic gradients are well-behaved; that is, they are unbiased and have good concentration properties, as asserted in the following lemma. It is straightforward to verify given that we sample $\z$ from a zero-mean Gaussian (proof in Appendix~\ref{app:proof_stoc}). 

% We use the following characterization of sub-Gaussian random variables.
% \begin{definition}[Sub-Gaussianity]\label{def:subgaussian}
% 	A random vector $X$ in $\R^d$ is said to be sub-Gaussian with parameter $s$ if for any $\u \in \R^d$, it holds that $\Pr(|\inner{X}{\u}| > t )\leq \Pr(|Z| > t)$ for all $t>0$,  $Z \sim  \mathcal{N}(0,s\norm{\u})$ .
% \end{definition}

\begin{lemma}[Property of stochastic gradient]\label{lem:stocgrad}
		Let $\g(\x;\z) = \z[f(\x+\z)- f(\x)]/\sigma^2$, where $\z \sim  \mathcal{N}(0,\sigma^2 \I)$. Then $\E_\z \g(\x; \z) = \grad \tilde{f}_\sigma(\x)$,
		and $\g(\x; \z)$ is sub-Gaussian with parameter $\frac{B}{\sigma}$.
\end{lemma}

%\begin{restatable}[Property of stochastic gradient]{lemma}{lemmastocgrad} \label{lem:stocgrad}
%	Let $\g(\x;\z) = \z[f(\x+\z)- f(\x)]/\sigma^2$, where $\z \sim  \mathcal{N}(0,\sigma^2 \I)$. Then $\E_\z \g(\x; \z) = \grad \tilde{f}_\sigma(\x)$,
%	and $\g(\x; \z)$ is sub-Gaussian with parameter $\frac{B}{\sigma}$.
%	% $\g(\x;\z) $ satisfies following properties:
%	% \begin{enumerate}
%	% 	\item $\E_\z \g(\x; \z) = \grad \tilde{f}_\sigma(\x)$.
%	% 	\item $\g(\x; \z)$ is sub-Gaussian with parameter $\frac{B}{\sigma}$.
%	% \end{enumerate}	
%\end{restatable}

As it turns out, these assumptions suffice to guarantee that perturbed SGD (PSGD), a simple adaptation of PGD in \citet{jin17escape} with stochastic gradient and large mini-batch size, converges to the second-order stationary point of the objective function.
% we can use similar techniques as in \cite{} to prove convergence guarantees for ZPSGD on $\tilde{f}_\sigma$

% \cnote{shorten the theorem}

\begin{theorem}[PSGD efficiently escapes saddle points \citep{jin2018sgd}, informal]\label{thm:psgd_guar}
Suppose $f(\cdot)$ is $\ell$-gradient Lipschitz and $\rho$-Hessian Lipschitz, and stochastic gradient $\g(\x, \theta)$ with $\E \g(\x; \theta) = \grad f(\x)$ has a sub-Gaussian tail with parameter $\sigma/\sqrt{d}$, then for any $\delta >0$, with proper choice of hyperparameters, PSGD (Algorithm \ref{algo:Mini_PSGD}) will find an $\epsilon$-SOSP of $f$ with probability $1-\delta$, in $\poly(d, B, \ell,\rho, \sigma, 1/\epsilon, \log (1/\delta))$  number of queries.
% then we run Minibatch PSGD (Algorithm \ref{algo:Mini_PSGD}) with parameter chosen as 
% \begin{equation*}
% m = \tilde{\Theta}\left(\frac{\sigma^2 d \cn}{\epsilon^2 \delta^2}\cdot \frac{\Delta^2_f \rho}{\epsilon^3}\right) \quad 
% \eta = \frac{1}{\ell}, \quad
% r = \tilde{\Theta}(\epsilon)
% \end{equation*} 
% where $\cn = \ell/\sqrt{\epsilon\rho}$. Then, PSGD will at least once be $\epsilon-$second order stationary point
% with $1-\delta$ probability in total number of queries:
% \begin{equation*}
% \tilde{O}\left(\frac{\ell \Delta_f }{\epsilon^2} \cdot m\right)
% \end{equation*}
\end{theorem}

For completeness, we include the formal version of the theorem and its proof in Appendix \ref{app:sgd}. Combining this theorem and the second part of Lemma~\ref{lem:smoothing}, we see that by choosing an appropriate smoothing radius $\sigma$, our algorithm ZPSGD finds an $C\epsilon/\sqrt{d}$-SOSP for $\tilde{f}_\sigma$ which is also an $\epsilon$-SOSP for $F$ for some universal constant~$C$.
\vspace{-2mm}

\subsection{Polynomial queries lower bound}
\vspace{-2mm}

The proof of Theorem \ref{thm:lowerbound_informal} depends on the construction of a `hard' function pair. The argument crucially depends on the concentration of measure in high dimensions. We provide a proof sketch in Appendix~\ref{section:lower-bnd-overview} and the full proof in Appendix~\ref{app:lower_bnd}.

\vspace{-2mm}

% !TeX root = main.tex 

\section{Applications}
\vspace{-2mm}

In this section, we present several applications of our algorithm. We first show a simple example of learning one rectified linear unit (ReLU), where the empirical risk is nonconvex and nonsmooth. %as a example of obtaining such guarantees with a nonconvex nonsmooth empirical risk function using our results. Next 
We also briefly survey other potential applications for our model as stated in Problem~\ref{problem}.% to demonstrate the broad usefulness of our framework.

\subsection{Statistical Learning Example: Learning ReLU}

% To illustrate how our result can be applied to learning and generalization with a nonconvex nonsmooth empirical risk function, we compute the sample complexity of learning a ReLU unit using ZPSGD. 

Consider the simple example of learning a ReLU unit. Let $\relu(z) = \max\{z, 0\}$ for $z\in\R$. Let $\w^\star(\norm{\w^\star} = 1)$ be the desired solution. We %assume the existence of true solution $\w^\star$ where $\norm{\w^\star} = 1$ and 
assume data $(\x_i, \y_i)$ is generated as $y_i = \relu(\x_i\trans \w^\star) + \zeta_i$ where noise $\zeta_i \sim \mathcal{N}(0, 1)$.
We further assume the features $\x_i \sim \mathcal{N}(0, \I)$ are also generated from a standard Gaussian distribution. The empirical risk with a squared loss function is:
\begin{equation*}
\emploss(\w) = \frac{1}{n}\sum_{i=1}^n (y_i - \relu(\x_i\trans\w))^2.
\end{equation*}
Its population version is $\poploss(\w) = \E[ \emploss(\w)]$. In this case, the empirical risk is highly nonsmooth---in fact, not differentiable in all subspaces perpendicular to each $\x_i$. The population risk turns out to be smooth in the entire space $\R^{d}$ except at $\bm{0}$. This is illustrated in Figure \ref{fig:relu}, where the empirical risk displays many sharp corners.
% As can be seen in Figure \ref{fig:relu}, the empirical loss function in this model is highly non-smooth, having many sharp corners.

\begin{figure}[t]
	\centering
	\includegraphics[width=.35\textwidth]{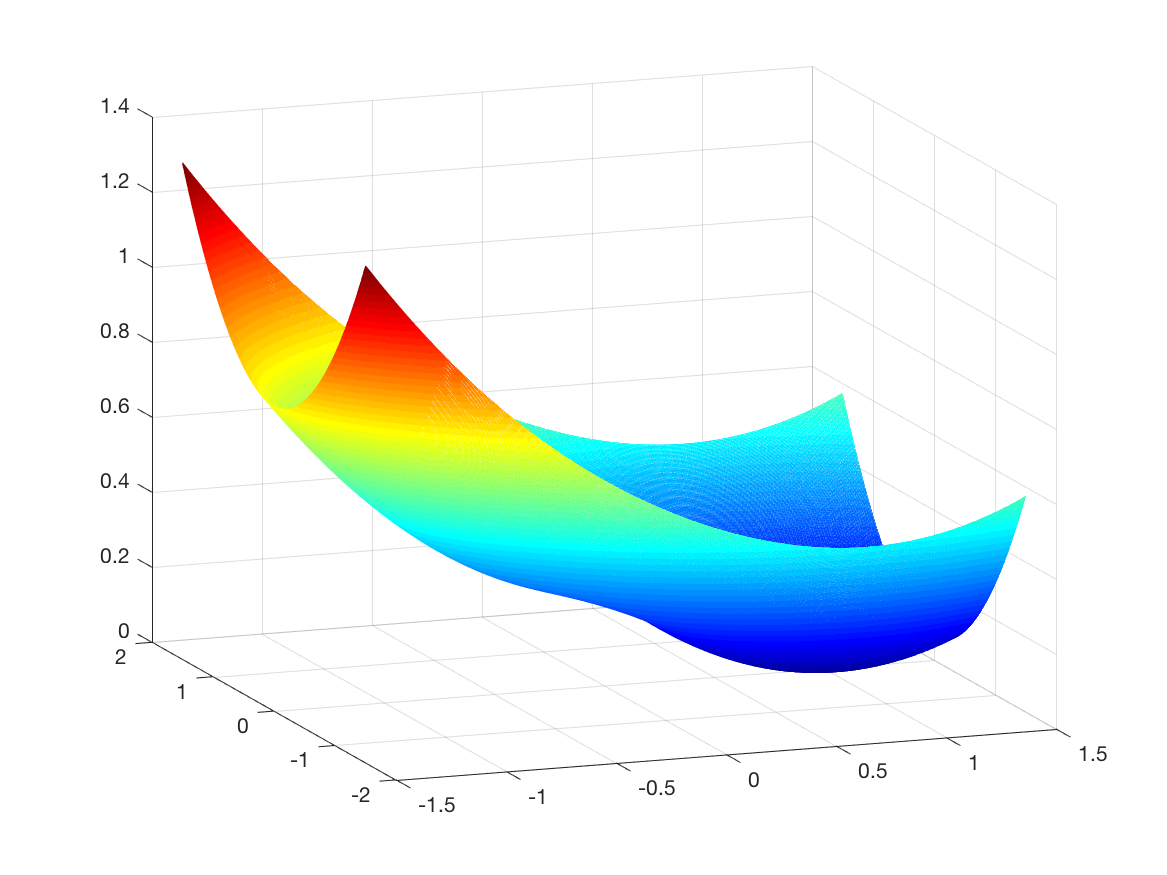}
	\includegraphics[width=.35\textwidth]{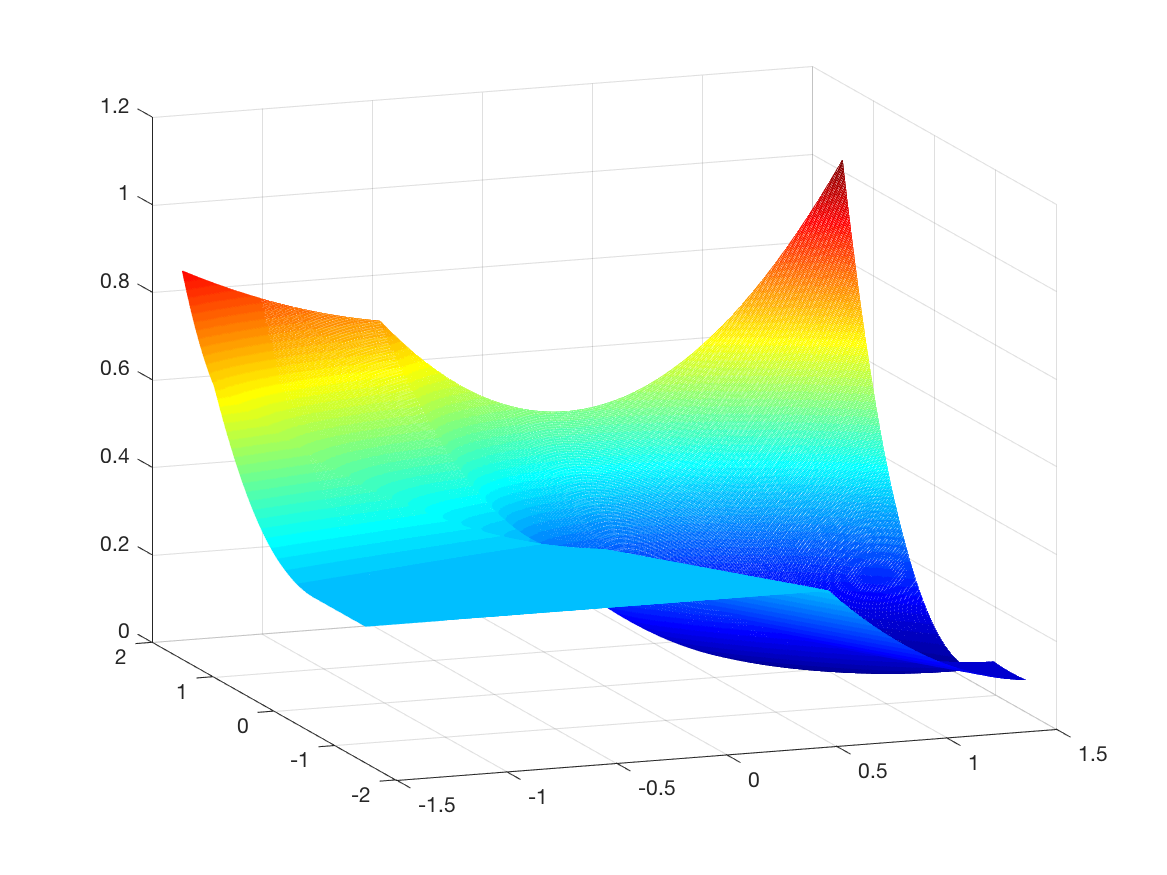}
	\caption{Population (left) and Empirical (right) risk for learning ReLU Unit , $d=2$. %Sample size for the empirical loss function is $2d$. 
    Sharp corners present in the empirical risk are not found in the population version.}\label{fig:relu}
    \vspace{-3ex}
\end{figure}

Due to nonsmoothness at $\bm{0}$ even for population risk, we focus on a compact region $\neibor = \{\w| \w\trans \w^\star \ge \frac{1}{\sqrt{d}}\} \cap  \{\w| \norm{\w} \le 2\}$ which excludes $\bm{0}$. This region is large enough so that a random initialization has at least constant probability of being inside it. We also show the following properties that allow us to apply Algorithm~\ref{algo:PSGD} directly:% of the population risk: %and also satisfies the following four properties:

\begin{lemma} \label{lem:ReLU}
	The population and empirical risk $\poploss, \emploss$ of learning a ReLU unit problem satisfies: \newline
	% let $\neibor = \{\w| \w\trans \w^\star \ge \frac{1}{\sqrt{d}}\} \cap  \{\w| \norm{\w} \le 2\}$.
	1. If $\w_0 \in \neibor$, then runing ZPSGD (Algorithm \ref{algo:PSGD}) gives $\w_t \in \neibor$ for all $t$ with high probability. \newline
	2. Inside $\neibor$, $\poploss$ is $O(1)$-bounded, $O(\sqrt{d})$-gradient Lipschitz, and $O(d)$-Hessian Lipschitz. \newline
	3. $\sup_{\w \in \neibor} | \emploss(\w) -\poploss(\w)| \le \tilde{O}(\sqrt{d/n})$ w.h.p. \newline
	4. Inside $\neibor$, $\poploss$ is nonconvex function, $\w^\star$ is the only SOSP of $\poploss(\w)$.
\end{lemma}

%\begin{restatable}{lemma}{lemrelu} \label{lem:ReLU}
%The population and empirical risk $\poploss, \emploss$ of learning a ReLU unit problem satisfies:
%% let $\neibor = \{\w| \w\trans \w^\star \ge \frac{1}{\sqrt{d}}\} \cap  \{\w| \norm{\w} \le 2\}$.
%\begin{enumerate}
%\item If $\w_0 \in \neibor$, then runing ZPSGD (Algorithm \ref{algo:PSGD}) gives $\w_t \in \neibor$ for all $t$ with high probability.
%\item Inside $\neibor$, $\poploss$ is $O(1)$-bounded, $O(\sqrt{d})$-gradient Lipschitz, and $O(d)$-Hessian Lipschitz.
%\item $\sup_{\w \in \neibor} | \emploss(\w) -\poploss(\w)| \le \tilde{O}(\sqrt{d/n})$ w.h.p.
%\item Inside $\neibor$, $\poploss$ is nonconvex function, $\w^\star$ is the only SOSP of $\poploss(\w)$.
%\end{enumerate}
%\end{restatable}

These properties show that the population loss has a well-behaved landscape, while the empirical risk is pointwise close. This is exactly what we need for Algorithm~\ref{algo:PSGD}. Using Theorem~\ref{thm:upperbound_informal} we immediately get the following sample complexity, which guarantees an approximate population risk minimizer. We defer all proofs to Appendix \ref{app:relu}.
%We note the third claim is due to a standard covering argument over compact space. Equipped with the above lemma, the following theorem applies Theorem \ref{thm:upperbound_informal} to directly give an sample complexity which guarantees to find an approximate population risk minimizer.

\begin{theorem} \label{thm:ReLU}
For learning a ReLU unit problem, suppose the sample size is $n \ge \tilde{O}(d^4/\epsilon^3)$, and the initialization is $\w_0 \sim \mathcal{N}(0, \frac{1}{d} \I)$, then with at least constant probability, Algorithm \ref{algo:PSGD} will output an estimator $\hat{\w}$ so that $\norm{\hat{\w} - \w^\star} \le \epsilon$.
\end{theorem}

% To prove the above result, we use the following lemma which establishes four technical conditions. Firstly, we show that for a suitable initialization, our algorithm never leaves the region $\neibor$. We require this because $R$ is non-smooth at precisely one point, the origin. Secondly, we show that indeed $R$ is smooth and bounded in $\neibor$, so assumption \ref{assump} is satisfied. Next (3), we show that the empirical risk $\hat{R}_n$ is close to $R$ in $\neibor$ by a covering argument, hence bounding $\nu$ with high probability. Finally, we argue that $\R$ is indeed nonconvex and has a unique second-order stationary point inside $\neibor$. 

\vspace{-2mm}
\subsection{Other applications} \label{sec:otherapplications}
\vspace{-2mm}
\paragraph{Private machine learning} Data privacy is a significant concern in machine learning as it creates a trade-off between privacy preservation and successful learning. Previous work on differentially private machine learning \citep[e.g.][]{chaudhuri11differential} have studied \emph{objective perturbation}, that is, adding noise to the original (convex) objective and optimizing this perturbed objective, as a way to simultaneously guarantee differential privacy and learning generalization: $f = F + p(\eps)$.
Our results may be used to extend such guarantees to nonconvex objectives, characterizing when it is possible to optimize $F$ even if the data owner does not want to reveal the true value of $F(\x)$ and instead only reveals $f(\x)$ after adding a perturbation $p(\eps)$, which depends on the privacy guarantee~$\eps$.

\paragraph{Two stage robust optimization}
Motivated by the problem of adversarial examples in machine learning, there has been a lot of recent interest \citep[e.g.][]{steinhardt2017certified,sinha2018certifiable} in a form of robust optimization that involves a minimax problem formulation: $
\min_\x \max_\u G(\x, \u).$ 
The function $F(\x) = \max_\u G(\x, \u)$ tends to be nonconvex in such problems, since $G$ can be very complicated. 
It can be intractable or costly to compute the solution to the inner maximization exactly, but it is often possible to get a good enough approximation $f$, such that $\sup_{\x}|F(\x) - f(\x)| = \nu$. It is then possible to solve $\min_\x f(\x)$ by ZPSGD, with guarantees for the original optimization problem. %Our lower bound suggests that no algorithm can solve every instance of (1) to within $\epsilon$ error if the approximation error $\nu \geq \tilde{\Omega}(\frac{\epsilon^{1.5}}{\sqrt{\rho} d})$.

%\subsubsection{Missing or Corrupted data}
%
%Our approach provides the tightest possible optimization guarantees that take into account all sources of error in the objective function simultaneously, without any assumptions on how they were generated. In particular, such errors may stem from missing or corrupted data, studied in statistical settings by \cite{ loh11nips, loh2013regularized}.
%

%\input{conclusion}
%
\subsubsection*{Acknowledgments}

We thank Aditya Guntuboyina, Yuanzhi Li, Yi-An Ma, Jacob Steinhardt, and Yang Yuan for valuable discussions.

{\small
\bibliographystyle{plainnat}
\bibliography{localmin}

\begin{thebibliography}{33}
\providecommand{\natexlab}[1]{#1}
\providecommand{\url}[1]{\texttt{#1}}
\expandafter\ifx\csname urlstyle\endcsname\relax
  \providecommand{\doi}[1]{doi: #1}\else
  \providecommand{\doi}{doi: \begingroup \urlstyle{rm}\Url}\fi

\bibitem[Agarwal et~al.(2010)Agarwal, Dekel, and Xiao]{agarwal10optimal}
Alekh Agarwal, Ofer Dekel, and Lin Xiao.
\newblock Optimal algorithms for online convex optimization with multi-point
  bandit feedback.
\newblock In \emph{Proceedings of the 23rd Annual Conference on Learning Theory
  (COLT)}, 2010.

\bibitem[Agarwal et~al.(2017)Agarwal, {Allen Zhu}, Bullins, Hazan, and
  Ma]{agarwal17finding}
Naman Agarwal, Zeyuan {Allen Zhu}, Brian Bullins, Elad Hazan, and Tengyu Ma.
\newblock Finding approximate local minima faster than gradient descent.
\newblock In \emph{{Proceedings of the 49th Annual ACM Symposium on Theory of
  Computing}}, pages 1195--1199. {ACM}, 2017.

\bibitem[Anandkumar and Ge(2016)]{AnandkumarG16}
Animashree Anandkumar and Rong Ge.
\newblock Efficient approaches for escaping higher order saddle points in
  non-convex optimization.
\newblock In \emph{{Proceedings of the 29th Annual Conference on Learning
  Theory (COLT)}}, volume~49, pages 81--102, 2016.

\bibitem[Auer et~al.(1996)Auer, Herbster, and Warmuth]{auer96exponentially}
Peter Auer, Mark Herbster, and Manfred~K Warmuth.
\newblock Exponentially many local minima for single neurons.
\newblock In \emph{Advances in Neural Information Processing Systems (NIPS)},
  pages 316--322. 1996.

\bibitem[Bartlett and Mendelson(2003)]{bartlett03rademacher}
Peter~L. Bartlett and Shahar Mendelson.
\newblock Rademacher and {G}aussian complexities: Risk bounds and structural
  results.
\newblock \emph{J. Mach. Learn. Res.}, 3, 2003.

\bibitem[Belloni et~al.(2015)Belloni, Liang, Narayanan, and
  Rakhlin]{belloni2015escaping}
Alexandre Belloni, Tengyuan Liang, Hariharan Narayanan, and Alexander Rakhlin.
\newblock {Escaping the Local Minima via Simulated Annealing: Optimization of
  Approximately Convex Functions}.
\newblock In \emph{Proceedings of the 28th Conference on Learning Theory
  (COLT)}, pages 240--265, 2015.

\bibitem[Boucheron et~al.(2013)Boucheron, Lugosi, and Massart]{boucheron}
St\'ephane Boucheron, G\'abor Lugosi, and Pascal Massart.
\newblock \emph{{Concentration Inequalities: A Nonasymptotic Theory of
  Independence}}.
\newblock Oxford University Press, 2013.

\bibitem[Brutzkus and Globerson(2017)]{brutzkus17globally}
Alon Brutzkus and Amir Globerson.
\newblock Globally optimal gradient descent for a convnet with gaussian inputs.
\newblock In \emph{Proceedings of the International Conference on Machine
  Learning (ICML)}, volume~70, pages 605--614. {PMLR}, 2017.

\bibitem[Carmon et~al.(2016)Carmon, Duchi, Hinder, and
  Sidford]{carmon16accelerated}
Yair Carmon, John~C Duchi, Oliver Hinder, and Aaron Sidford.
\newblock Accelerated methods for non-convex optimization.
\newblock \emph{arXiv preprint arXiv:1611.00756}, 2016.

\bibitem[Chaudhuri et~al.(2011)Chaudhuri, Monteleoni, and
  Sarwate]{chaudhuri11differential}
Kamalika Chaudhuri, Claire Monteleoni, and Anand~D. Sarwate.
\newblock Differentially private empirical risk minimization.
\newblock \emph{J. Mach. Learn. Res.}, 12:\penalty0 1069--1109, July 2011.
\newblock ISSN 1532-4435.

\bibitem[Cho and Saul(2009)]{cho2009kernel}
Youngmin Cho and Lawrence~K Saul.
\newblock Kernel methods for deep learning.
\newblock In \emph{Advances in Neural Information Processing Systems (NIPS)},
  pages 342--350, 2009.

\bibitem[Dinh et~al.(2017)Dinh, Pascanu, Bengio, and Bengio]{dinh2017sharp}
Laurent Dinh, Razvan Pascanu, Samy Bengio, and Yoshua Bengio.
\newblock Sharp minima can generalize for deep nets.
\newblock \emph{arXiv preprint arXiv:1703.04933}, 2017.

\bibitem[Duchi et~al.(2015)Duchi, Jordan, Wainwright, and
  Wibisono]{duchi15optimal}
John~C. Duchi, Michael~I. Jordan, Martin~J. Wainwright, and Andre Wibisono.
\newblock {Optimal rates for zero-order convex optimization: The power of two
  function evaluations}.
\newblock \emph{{IEEE} Trans. Information Theory}, 61\penalty0 (5):\penalty0
  2788--2806, 2015.

\bibitem[Flaxman et~al.(2005)Flaxman, Kalai, and McMahan]{flaxman05online}
Abraham~D. Flaxman, Adam~Tauman Kalai, and H.~Brendan McMahan.
\newblock Online convex optimization in the bandit setting: Gradient descent
  without a gradient.
\newblock In \emph{Proceedings of the Sixteenth Annual ACM-SIAM Symposium on
  Discrete Algorithms (SODA)}, pages 385--394, 2005.

\bibitem[Ge et~al.(2015)Ge, Huang, Jin, and Yuan]{ge2015escaping}
Rong Ge, Furong Huang, Chi Jin, and Yang Yuan.
\newblock Escaping from saddle points---online stochastic gradient for tensor
  decomposition.
\newblock In \emph{Proceedings of the 28th Conference on Learning Theory
  (COLT)}, 2015.

\bibitem[Jin et~al.(2017{\natexlab{a}})Jin, Ge, Netrapalli, Kakade, and
  Jordan]{jin17escape}
Chi Jin, Rong Ge, Praneeth Netrapalli, Sham~M. Kakade, and Michael~I. Jordan.
\newblock How to escape saddle points efficiently.
\newblock In \emph{Proceedings of the International Conference on Machine
  Learning (ICML)}, pages 1724--1732, 2017{\natexlab{a}}.

\bibitem[Jin et~al.(2017{\natexlab{b}})Jin, Netrapalli, and Jordan]{jin17accel}
Chi Jin, Praneeth Netrapalli, and Michael~I. Jordan.
\newblock Accelerated gradient descent escapes saddle points faster than
  gradient descent.
\newblock \emph{CoRR}, abs/1711.10456, 2017{\natexlab{b}}.

\bibitem[Jin et~al.(2018)Jin, Ge, Netrapalli, Kakade, and Jordan]{jin2018sgd}
Chi Jin, Rong Ge, Praneeth Netrapalli, Sham~M. Kakade, and Michael~I. Jordan.
\newblock {SGD} escapes saddle points efficiently.
\newblock \emph{Personal Communication}, 2018.

\bibitem[Keskar et~al.(2016)Keskar, Mudigere, Nocedal, Smelyanskiy, and
  Tang]{keskar2016large}
Nitish~Shirish Keskar, Dheevatsa Mudigere, Jorge Nocedal, Mikhail Smelyanskiy,
  and Ping Tak~Peter Tang.
\newblock On large-batch training for deep learning: Generalization gap and
  sharp minima.
\newblock \emph{arXiv preprint arXiv:1609.04836}, 2016.

\bibitem[Kirkpatrick et~al.(1983)Kirkpatrick, Gelatt, and
  Vecchi]{Kirkpatrick671}
Scott Kirkpatrick, C.~D. Gelatt, and Mario Vecchi.
\newblock Optimization by simulated annealing.
\newblock \emph{Science}, 220\penalty0 (4598):\penalty0 671--680, 1983.

\bibitem[Kleinberg et~al.(2018)Kleinberg, Li, and Yuan]{kleinberg18an}
Robert Kleinberg, Yuanzhi Li, and Yang Yuan.
\newblock An alternative view: When does {SGD} escape local minima?
\newblock \emph{CoRR}, abs/1802.06175, 2018.

\bibitem[Loh and Wainwright(2013)]{loh2013regularized}
Po-Ling Loh and Martin~J Wainwright.
\newblock {Regularized M-estimators with nonconvexity: Statistical and
  algorithmic theory for local optima}.
\newblock In \emph{Advances in Neural Information Processing Systems (NIPS)},
  pages 476--484, 2013.

\bibitem[Mei et~al.(2016)Mei, Bai, and Montanari]{mei2016landscape}
Song Mei, Yu~Bai, and Andrea Montanari.
\newblock The landscape of empirical risk for non-convex losses.
\newblock \emph{arXiv preprint arXiv:1607.06534}, 2016.

\bibitem[Nesterov(1983)]{nesterov83accel}
Yurii Nesterov.
\newblock A method of solving a convex programming problem with convergence
  rate $o(1/k^2)$.
\newblock \emph{{Soviet Mathematics Doklady}}, 27:\penalty0 372–376, 1983.

\bibitem[Nesterov(2004)]{nesterov1998introductory}
Yurii Nesterov.
\newblock \emph{Introductory Lectures on Convex Programming}.
\newblock Springer, 2004.

\bibitem[Rechenberg and Eigen(1973)]{rechenberg73}
Ingo Rechenberg and Manfred Eigen.
\newblock \emph{Evolutionsstrategie: Optimierung Technischer Systeme nach
  Prinzipien der Biologischen Evolution}.
\newblock Frommann-Holzboog, Stuttgart, 1973.

\bibitem[Risteski and Li(2016)]{risteski16alg}
Andrej Risteski and Yuanzhi Li.
\newblock {Algorithms and matching lower bounds for approximately-convex
  optimization}.
\newblock In \emph{Advances in Neural Information Processing Systems (NIPS)},
  pages 4745--4753. 2016.

\bibitem[Shamir(2013)]{shamir13bandit}
Ohad Shamir.
\newblock On the complexity of bandit and derivative-free stochastic convex
  optimization.
\newblock In \emph{Proceedings of the 26th Annual Conference on Learning Theory
  (COLT)}, volume~30, 2013.

\bibitem[Singer and Vondrak(2015)]{singer15information}
Yaron Singer and Jan Vondrak.
\newblock Information-theoretic lower bounds for convex optimization with
  erroneous oracles.
\newblock In \emph{Advances in Neural Information Processing Systems (NIPS)},
  pages 3204--3212. 2015.

\bibitem[Sinha et~al.(2018)Sinha, Namkoong, and Duchi]{sinha2018certifiable}
Aman Sinha, Hongseok Namkoong, and John Duchi.
\newblock Certifiable distributional robustness with principled adversarial
  training.
\newblock \emph{International Conference on Learning Representations}, 2018.

\bibitem[Steinhardt et~al.(2017)Steinhardt, Koh, and
  Liang]{steinhardt2017certified}
Jacob Steinhardt, Pang~W. Koh, and Percy Liang.
\newblock Certified defenses for data poisoning attacks.
\newblock In \emph{Advances in Neural Information Processing Systems (NIPS)},
  2017.

\bibitem[Welling and Teh(2011)]{welling11bayesian}
Max Welling and Yee~Whye Teh.
\newblock {Bayesian Learning via Stochastic Gradient {L}angevin Dynamics}.
\newblock In \emph{Proceedings of the International Conference on Machine
  Learning (ICML)}, pages 681--688, 2011.

\bibitem[Zhang et~al.(2017)Zhang, Liang, and Charikar]{zhang2017hitting}
Yuchen Zhang, Percy Liang, and Moses Charikar.
\newblock A hitting time analysis of stochastic gradient {L}angevin dynamics.
\newblock \emph{Proceedings of the 30th Conference on Learning Theory (COLT)},
  pages 1980--2022, 2017.

\end{thebibliography}
}

\newpage

\appendix

% !TeX root = appendix_main.tex 

\section{Efficient algorithm for optimizing the population risk}\label{app:upper}

As we described in Section~\ref{sec:overview}, in order to find a second-order stationary point of the population loss $F$, we apply perturbed stochastic gradient on a smoothed version of the empirical loss $f$. Recall that the smoothed function is defined as

% In this and subsequent sections, we recall the following notation and assumptions.
 
% Assume $f,F$ satisfies Assumption \ref{assump}. For $\z \sim \mathcal{N}(0, \sigma^2 \I)$, denote 
$$\tilde{f}_\sigma(\x) = \E_\z f(\x+ \z).$$

In this section we will also consider a smoothed version of the population loss $F$, as follows:
$$\tilde{F}_\sigma(\x) = \E_\z F(\x+ \z).$$ 
This function is of course not accessible by the algorithm and we only use it in the proof of convergence rates. 

%We use $x \lesssim y$ to denote $x \le Cy$ for some universal constant $C$.

This section is organized as follows. In section \ref{app:lemsmoothing}, we present and prove the key lemma on the properties of the smoothed function $\tilde{f}_\sigma(\x)$. Next, in section \ref{app:proof_stoc}, we prove the properties of the stochastic gradient $\g$. Combining the lemmas in these two subsections, in section \ref{app:proof_upper} we prove a main theorem about the guarantees of ZPSGD (Theorem \ref{thm:upperbound_informal}). For clarity, we defer all technical lemmas and their proofs to section \ref{app:A_technical_lems}.

\subsection{Properties of the Gaussian smoothing}\label{app:lemsmoothing}

In this section, we show the properties of smoothed function $\tilde{f}_\sigma(\x)$. We first restate Lemma \ref{lem:smoothing}.

\begin{lemma}[Property of smoothing]
	\label{lem:smoothing2}
	Assume that the function pair ($F, f$) satisfies Assumption \ref{assump}, and let $\tilde{f}_\sigma(\x)$ be as given in definition \ref{def:smoothed_f}. Then, the following holds
	\begin{enumerate}
		\item $\tilde{f}_\sigma(\x)$ is $O(\ell + \frac{\nu}{\sigma^2})$-gradient Lipschitz and $O(\rho + \frac{\nu}{\sigma^3})$-Hessian Lipschitz.
		\item %The differences in gradient and Hessian are small
		$\norm{\grad \tilde{f}_\sigma(\x) - \grad F(\x)} \leq  O(\rho d \sigma^2+  \frac{\nu}{\sigma})$ and
		$\norm{\hess \tilde{f}_\sigma(\x) - \hess F(\x)} \leq O(\rho\sqrt{d}\sigma + \frac{\nu}{\sigma^2})$.
	\end{enumerate}
\end{lemma}

Intuitively, the first property states that if the original function $F$ is gradient and Hessian Lipschitz, the smoothed version of the perturbed function $f$ is also gradient and Hessian Lipschitz (note that this is of course not true for the perturbed function $f$); the second property shows that the gradient and Hessian of $\tilde{f}_\sigma$ is point-wise close to the gradient and Hessian of the original function $F$. We will prove the four points (1 and 2, gradient and Hessian) of the lemma one by one, in Sections \ref{subsub:gradlipschitz} to \ref{subsub:hessiandiff}.

In the proof, we frequently require the following lemma (see e.g. \citet{zhang2017hitting}) that gives alternative expressions for the gradient and Hessian of a smoothed function.

\begin{lemma}[Gaussian smoothing identities \citep{zhang2017hitting}]\label{identity_1}
	$\tilde{f}_\sigma$ has gradient and Hessian:
	
	\[\grad \tilde{f}_\sigma(\x) = \E_\z[\frac{\z}{\sigma^2}f(\x+\z)], \quad \hess \tilde{f}_\sigma(\x) = \E_\z[\frac{\z\z\trans - \sigma^2\I}{\sigma^4}f(\x+\z)].\]
	\begin{proof}
	Using the density function of a multivariate Gaussian, we may compute the gradient of the smoothed function as follows:
		\begin{align*}
		\grad \tilde{f}_\sigma(\x) &= \fracpar{}{\x} \frac{1}{(2\pi\sigma^2)^{d/2}} \int f(\x+\z) e^{-\norm{\z}^2/2\sigma^2} d\z 
		=\frac{1}{(2\pi\sigma^2)^{d/2}} \int \fracpar{}{\x}  f(\x+\z) e^{-\norm{\z}^2/2\sigma^2} d\z \\
		&=\frac{1}{(2\pi\sigma^2)^{d/2}} \int \fracpar{}{\x}  f(\z) e^{-\norm{\z - \x}^2/2\sigma^2} d\z 
		=\frac{1}{(2\pi\sigma^2)^{d/2}} \int f(\z') \frac{\x - \z}{\sigma^2}e^{-\norm{\z - \x}^2/2\sigma^2} d\z \\
		&=\frac{1}{(2\pi\sigma^2)^{d/2}} \int f(\z +\x) \frac{\z}{\sigma^2}e^{-\norm{-\z}^2/2\sigma^2} d\z 
		=\E_\z[\frac{\z}{\sigma^2}f(\x+\z)],
		\end{align*}
	and similarly, we may compute the Hessian of the smoothed function:
		\begin{align*}
		\hess \tilde{f}_\sigma(\x) &=  \fracpar{}{\x} \frac{1}{(2\pi\sigma^2)^{d/2}} \int \frac{\z}{\sigma^2} f(\x+\z) e^{-\norm{\z}^2/2\sigma^2} d\z \\
		&=\frac{1}{(2\pi\sigma^2)^{d/2}} \int \fracpar{}{\x}  \frac{\z-\x}{\sigma^2} f(\z) e^{-\norm{\z-\x}^2/2\sigma^2} d\z \\
		&=\frac{1}{(2\pi\sigma^2)^{d/2}} \int f(\z) ( \frac{(\z-\x)(\z-\x)\trans}{\sigma^4}  e^{-\norm{\z-\x}^2/2\sigma^2} -  \frac{ \I}{\sigma^2}  e^{-\norm{\z-\x}^2/2\sigma^2}  ) d\z \\
		&=\frac{1}{(2\pi\sigma^2)^{d/2}} \int f(\z +\x) ( \frac{\z\z\trans - \sigma^2\I }{\sigma^4} ) e^{-\norm{\z}^2/2\sigma^2} d\z 
		=\E_\z[\frac{\z\z\trans - \sigma^2\I}{\sigma^4}f(\x+\z)].
		\end{align*}
	\end{proof}
\end{lemma}

\subsubsection{Gradient Lipschitz}\label{subsub:gradlipschitz}
We bound the gradient Lipschitz constant of $\tilde{f}_\sigma$ in the following lemma.

\begin{lemma}[Gradient Lipschitz of $\tilde{f}_\sigma$]\label{grad_lip}
	% Consider the setting of Lemma \ref{lem:smoothing}, 
	$\norm{\hess \tilde{f}_\sigma(\x) } \leq  O(\ell +\frac{\nu}{\sigma^2})$.

\end{lemma}
	\begin{proof}
	For a twice-differentiable function, its gradient Lipshitz constant is also the upper bound on the spectral norm of its Hessian.
		\begin{align*}
		\norm{\hess \tilde{f}_\sigma(\x) } &=\norm{\hess \tilde{F}_\sigma(\x) + \hess \tilde{f}_\sigma(\x) - \hess \tilde{F}_\sigma(\x)} \\ 
		&\leq \norm{\hess \tilde{F}_\sigma(\x)} +\norm{ \hess \tilde{f}_\sigma(\x) - \hess \tilde{F}_\sigma(\x)} \\
		&= \norm{\hess \E_\z [F(\x+\z)]} + \norm{\E_z[\frac{\z\z\trans-\sigma^2\I}{\sigma^4}(f-F)(\x+\z)]}\\ 
		&\leq \E_\z \norm{\hess  [F(\x+\z)]} + \frac{1}{\sigma^4} \norm{\E_z[\z\z\trans(f-F)(\x+\z)]} + \frac{1}{\sigma^2} \norm{\E_z[(f-F)(\x+\z)\I]} \\ 
		&\leq \ell + \frac{1}{\sigma^4} \norm{\E_z[\z\z\trans|(f-F)(\x+\z)|]} + \frac{1}{\sigma^2} \norm{\E_z[|(f-F)(\x+\z)|\I]} \\ 
		&= \ell + \frac{\nu}{\sigma^4} \norm{\E_z[\z\z\trans} + \frac{\nu}{\sigma^2} = \ell +\frac{2\nu}{\sigma^2}
		\end{align*}
		The last inequality follows from Lemma  \ref{small_lemma:psd}.
	\end{proof}

\subsubsection{Hessian Lipschitz}

We bound the Hessian Lipschitz constant of $\tilde{f}_\sigma$ in the following lemma.
\begin{lemma}[Hessian Lipschitz of $\tilde{f}_\sigma$]\label{hess_lip}
	$\norm{\hess \tilde{f}_\sigma(\x) - \hess \tilde{f}_\sigma(\y) }\leq O(\rho+\frac{\nu }{\sigma^3} ) \norm{\x-\y}$.

\end{lemma}
	\begin{proof}
	By triangle inequality:
		\begin{align*}
		&\norm{\hess \tilde{f}_\sigma(\x) - \hess \tilde{f}_\sigma(\y) } \\
		= &\norm{\hess \tilde{f}_\sigma(\x) - \hess \tilde{F}_\sigma(\x)  - \hess \tilde{f}_\sigma(\y) + \hess \tilde{F}_\sigma(\y) +\hess \tilde{F}_\sigma(\x) - \hess \tilde{F}_\sigma(\y)  } 	\\
		\leq &\norm{\hess \tilde{f}_\sigma(\x) - \hess \tilde{F}_\sigma(\x)  - (\hess \tilde{f}_\sigma(\y) - \hess \tilde{F}_\sigma(\y))} +  \norm{\hess \tilde{F}_\sigma(\x) - \hess \tilde{F}_\sigma(\y)  }  \\
		\leq &O(\frac{\nu }{\sigma^3})\norm{\x-\y} +   O(\rho)\norm{ \x-\y} + O(\norm{\x-\y}^2) \text{ }
		\end{align*}
		The last inequality follows from Lemma \ref{small_lemma:hess_Fs} and \ref{small_lemma:hess_diff_diff}.
	\end{proof}

\subsubsection{Gradient Difference}
We bound the difference between the gradients of smoothed function $\tilde{f}_\sigma(\x) $ and those of the true objective $F$.

\begin{lemma}[Gradient Difference]\label{grad_diff}
	$\norm{\grad \tilde{f}_\sigma(\x) - \grad F(\x)} \leq O(\frac{\nu}{\sigma} + \rho d \sigma^2)$.
	
\end{lemma}
\begin{proof}
	By triangle inequality:
	\begin{equation*}
	\norm{\grad \tilde{f}_\sigma(\x) - \grad F(\x)} \leq \norm{\grad \tilde{f}_\sigma(\x) - \grad \tilde{F}_\sigma(\x)} + \norm{\grad \tilde{F}_\sigma(\x) - \grad F(\x)}.
	\end{equation*}
	Then the result follows from Lemma \ref{small_lemma:grad_diff_fsFs} and \ref{small_lemma:grad_diff_FsF_rho} %\ref{small_lemma:grad_diff_FsF}.
\end{proof}

\subsubsection{Hessian Difference}
\label{subsub:hessiandiff}
We bound the difference between the Hessian of smoothed function $\tilde{f}_\sigma(\x) $ and that of the true objective $F$.

\begin{lemma}[Hessian Difference]\label{hess_diff}
	$\norm{\hess \tilde{f}_\sigma(\x) - \hess F(\x)} \leq O(\rho\sqrt{d}\sigma + \frac{\nu}{\sigma^2})$.

\end{lemma}
	\begin{proof} By triangle inequality:
		\begin{align*}
		\norm{\hess \tilde{f}_\sigma(\x) - \hess F(\x)} &\leq \norm{\hess \tilde{F}_\sigma(\x) - \hess F(\x)}  +  \norm{\hess \tilde{f}_\sigma(\x) - \hess \tilde{F}_\sigma(\x)}  \\
		&\leq \E_\z\norm{\hess F(\x +\z) - \hess F(\x)} + \frac{2\nu}{\sigma^2} \\
		&\leq \E_\z\norm{\rho \z} +\frac{2\nu}{\sigma^2}
		\leq \rho\sqrt{d}\sigma  +\frac{2\nu}{\sigma^2}
		\end{align*}
		
		The first inequality follows exactly from the proof of lemma \ref{grad_lip}. The second equality follows from the definition of Hessian Lipschitz. The third inequality follows from $\E_\z\norm{\rho \z}\leq  \rho \sqrt{\E[\norm{\z}^2]} $.
	\end{proof}

\subsection{Properties of the stochastic gradient} 
\label{app:proof_stoc}

We prove the properties of the stochastic gradient, $\g(\x;\z)$, as stated in Lemma \ref{lem:stocgrad}, restated as follows. Intuitively this lemma shows that the stochastic gradient is well-behaved and can be used in the standard algorithms.

\begin{lemma}[Property of stochastic gradient]\label{lem:stocgrad}
	Let $\g(\x;\z) = \z[f(\x+\z)- f(\x)]/\sigma^2$, where $\z \sim  \mathcal{N}(0,\sigma^2 \I)$. Then $\E_\z \g(\x; \z) = \grad \tilde{f}_\sigma(\x)$,
	and $\g(\x; \z)$ is sub-Gaussian with parameter $\frac{B}{\sigma}$.
\end{lemma}

%\lemmastocgrad*

%
%
%\begin{lemma}[Stochastic gradient $\g(\x;\z)$]
%	Let $\g(\x;\z) = \frac{\z[f(\x+\z)- f(\z)]}{\sigma^2}$, $\z \sim N(0,\sigma \I)$. Assume assumptions 1-5 hold. Then $\E_\z \g(\x; \z) = \grad \tilde{f}_\sigma(\x)$. Also, $\g(\x; \z) $ is sub-Gaussian with parameter $B/\sigma$, and $\g(\x; \z) - \g(\y; \z)$ is sub-Gaussian with parameter $\frac{L}{\sigma}\norm{\x-\y}$.
	\begin{proof}
		The first part follows from Lemma \ref{identity_1}.
		Given any $\u \in \R^d$, by assumption \ref{assump} ($f$ is $B$-bounded),
		$$|\inner{\u}{\g(\x;\z)}| = |(f(\x+\z)- f(\x))||\inner{\u}{\frac{\z}{\sigma^2 }}| \leq \frac{B\norm{u}}{\sigma} |\inner{\frac{\u}{\norm{u}}}{\frac{\z}{\sigma }}|.
		$$
		%$$\inner{\u}{\g(\x;\z)}^2 = (f(\x+\z)- f(\z))^2\inner{\u}{\frac{\z}{\sigma^2 }}^2 \leq \frac{B^2\norm{u}2}{\sigma^2} \inner{\frac{\u}{\norm{u}}}{\frac{\z}{\sigma }}^2 $$
		Note that $\inner{\frac{\u}{\norm{\u}}}{\frac{\z}{\sigma }}\sim  \mathcal{N}(0,1)$. Thus, for $X \sim \mathcal{N}(0, \frac{B\norm{u}}{\sigma})$,  $$\P(|\inner{\u}{\g(\x;\z)}| > s) \leq \P(|X| > s). $$ This shows that $\g$ is sub-Gaussian with parameter  $\frac{B}{\sigma}$.
		
	\end{proof}

\subsection{Proof of Theorem \ref{thm:upperbound_informal}: SOSP of $\tilde{f}_\sigma$ are also SOSP of $F$}\label{app:proof_upper}

%Recall the statement of lemma \ref{lem:smoothing}, which we proved in the previous sections.

%\lemmasmoothing*

Using the properties proven in Lemma~\ref{lem:stocgrad}, we can apply Theorem~\ref{thm:psgd_guar} to find an $\tilde{\epsilon}$-SOSP for $\tilde{f}_\sigma$ for any $\tilde{\epsilon}$. The running time of the algorithm is polynomial as long as $\tilde{\epsilon}$ depends polynomially on the relevant parameters. Now we will show that every $\tilde{\epsilon}$-SOSP of $\tilde{f}_\sigma$ is an $O(\epsilon)$-SOSP of $F$ when $\epsilon'$ is small enough.

%We now use Lemma \ref{lem:smoothing} to prove that any $\epsilon$-SOSP of $\tilde{f}_\sigma(\x)$ is also an $O(\epsilon)$-SOSP of $F$.
More precisely, we use Lemma \ref{lem:smoothing} to show that any $\frac{\epsilon}{\sqrt{d}}$-SOSP of $\tilde{f}_\sigma(\x)$ is also an $O(\epsilon)$-SOSP of $F$.

\begin{lemma}[SOSP of $\tilde{f}_\sigma(\x)$ and SOSP of $F(\x)$]\label{lem:upper_sosp}  Suppose $\x^*$ satisfies $$\norm{\grad \tilde{f}_\sigma(\x^*) } \leq \tilde{\epsilon}\text{ and } \mineval(\hess \tilde{f}_\sigma(\x^*)) \geq - \sqrt{\tilde{\rho}  \tilde{\epsilon}},$$ 
	where $\tilde{\rho} = \rho + \frac{\nu}{\sigma^3}$ and $\tilde{\epsilon} = \epsilon/\sqrt{d}$.
	Then there exists constants $c_1, c_2$ such that $$\sigma \leq c_1\sqrt{\frac{\epsilon}{\rho d} },~ \nu \leq  c_2\sqrt{\frac{\epsilon^{3}}{\rho d^{2}}}.$$
	implies $\x^*$ is an $O(\epsilon)$-SOSP of $F$.
	
\end{lemma}
\begin{proof}
	By applying Lemma \ref{lem:smoothing} and Weyl's inequality, we have that the following inequalities hold up to a constant factor:	
	\begin{align*}
		\norm{\grad F (\x^*)} &\leq \rho d \sigma^2 + \frac{\nu}{\sigma} + \tilde{\epsilon}\\
		\mineval(\hess F(\x^*)) &\geq \mineval(\hess \tilde{f}_\sigma(\x^*)) + \mineval(\hess F(\x^*)-\hess \tilde{f}_\sigma(\x^*)) \\
		&\geq - \sqrt{ (\rho + \frac{\nu}{\sigma^3} )\tilde{\epsilon}} - \norm{\hess \tilde{f}_\sigma(\x) - \hess F(\x)} \\
		&= - \sqrt{\frac{( \rho + \frac{\nu}{\sigma^3})}{\sqrt{d}}\epsilon} -( \rho\sqrt{d}\sigma + \frac{\nu}{\sigma^2} )
	\end{align*}
	Suppose we want any $\tilde{\epsilon}$-SOSP of $\tilde{f}_\sigma(\x)$ to be a $O(\epsilon)$-SOSP of $F$. Then satisfying the following inequalities is sufficient (up to a constant factor):
	\begin{align}
		\rho\sqrt{d}\sigma + \frac{\nu}{\sigma^2} &\leq  \sqrt{\rho \epsilon} \label{ineq1}\\
		\rho d \sigma^2+  \frac{\nu}{\sigma} &\leq \epsilon\label{ineq2} \\
		\rho + \frac{\nu}{\sigma^3} &\leq \rho \sqrt{d} \label{ineq3}
	\end{align}
	We know Eq.(\ref{ineq1}), (\ref{ineq2}) $\implies \sigma \leq \frac{\sqrt{\rho \epsilon}}{\rho \sqrt{d}} = \sqrt{\frac{\epsilon}{\rho d} }$ and $\sigma \leq \sqrt{\frac{\epsilon}{\rho d} } $. 
	
	\noindent Also Eq. (\ref{ineq1}), (\ref{ineq2}) $\implies \nu \leq \sigma \epsilon\leq \sqrt{\frac{\epsilon^3}{\rho d} }$ and $\nu \leq  \sqrt{\rho \epsilon} \sigma^2 \leq \sqrt{\rho \epsilon} \frac{\epsilon}{\rho d} = \sqrt{\frac{\epsilon^{3}}{\rho d^2}}$.
	
	\noindent Finally Eq.(\ref{ineq3}) $\implies \nu \leq \rho \sqrt{d} \sigma^3 \leq \frac{\epsilon^{1.5}}{\rho^{0.5}d}$.

	\noindent Thus the following choice of $\sigma$ and $\nu$ ensures that $\x^*$ is an $O(\epsilon)$-SOSP of $F$: $$\sigma \leq c_1 \sqrt{\frac{\epsilon}{\rho d} },~ \nu \leq c_2 \sqrt{\frac{\epsilon^{3}}{\rho d^{2}}},$$
	where $c_1, c_2$ are universal constants.
\end{proof}

\begin{proof}[Proof of Theorem \ref{thm:upperbound_informal}]
Applying  Theorem~\ref{thm:psgd_guar} on $\tilde{f}_\sigma(\x)$ guarantees finding an $c\frac{\epsilon}{\sqrt{d}}$-SOSP of $\tilde{f}_\sigma(\x)$ in number of queries polynomial in all the problem parameters. By Lemma \ref{lem:upper_sosp}, for some universal constant $c$, this is also an $\epsilon$-SOSP of $F$. This proves Theorem~\ref{thm:upperbound_informal}.
\end{proof}

\subsection{Technical Lemmas} \label{app:A_technical_lems}

In this section, we collect and prove the technical lemmas used in the proofs of the above.

\begin{lemma}\label{small_lemma:psd}
	Let $\lambda$ be a real-valued random variable and $A$ be a random PSD matrix that can depend on $\lambda$. Denote the matrix spectral norm as $\norm{\cdot}$. Then $\norm{\E [A\lambda]} \leq \norm{\E[A|\lambda|]}$.
	
\end{lemma}

\begin{proof}
	For any $\x \in \R^d$, 
	\begin{align*}
	\x\trans\E[A\lambda]\x &=\x\trans\E[A\lambda | \lambda \geq 0] \x\cdot \P(\lambda \geq 0) + \x\trans\E[A\lambda | \lambda < 0] \x\cdot \P(\lambda < 0) \\
	&\leq \x\trans\E[A\lambda | \lambda \geq 0] \x\cdot \P(\lambda \geq 0) - \x\trans\E[A\lambda | \lambda < 0] \x\cdot \P(\lambda < 0) \\
	&=\x\trans\E[A|\lambda|]\x 
	\end{align*}
\end{proof}

The following two technical lemmas bound the Hessian Lipschitz constants of $\tilde{F}_\sigma$ and $(\tilde{f}_\sigma-\tilde{F}_\sigma)$ respectively.

\begin{lemma}\label{small_lemma:hess_Fs}
	$\norm{\hess \tilde{F}_\sigma(\x) - \hess \tilde{F}_\sigma(\y)  } \leq \rho\norm{\x-\y}$.
	
\end{lemma}
\begin{proof}
	By the Hessian-Lipschitz property of $F$:
	\begin{align*}
	\norm{\hess \tilde{F}_\sigma(\x) - \hess \tilde{F}_\sigma(\y)  } &= \norm{\E_\z [\hess F(\x+\z)-\hess F(\y+\z)]} \\
	&\leq \E_\z \norm{\hess F(\x+\z)-\hess F(\y+\z)} \\
	&\leq \rho\norm{ \x-\y}
	\end{align*}
\end{proof}

\begin{lemma}\label{small_lemma:hess_diff_diff}
	$\norm{\hess \tilde{f}_\sigma(\x) - \hess \tilde{F}_\sigma(\x)  - (\hess \tilde{f}_\sigma(\y) - \hess \tilde{F}_\sigma(\y))} \leq O(\frac{\nu }{\sigma^3})\norm{\x-\y}+ O(\norm{\x-\y}^2)  $.
	
\end{lemma}
\begin{proof}
	For brevity, denote $h = \frac{1}{(2\pi\sigma^2)^{\frac{d}{2}}} $.
	\begin{align}
	&\hess \tilde{f}_\sigma(\x) - \hess \tilde{F}_\sigma(\x)  - (\hess \tilde{f}_\sigma(\y) - \hess \tilde{F}_\sigma(\y)) \nonumber  \\
	&= \E_\z[\frac{\z\z\trans-\sigma^2\I}{\sigma^4} ((f-F)(\x+\z) - (f-F)(\y+\z) ] \nonumber \\
	&= h \left( \int \frac{\z{\z}\trans-\sigma^2\I}{\sigma^4} (f-F)(\x+\z) e^{-\frac{\norm{\z}^2}{2\sigma^2}}d\z- \int \frac{\z{\z}\trans-\sigma^2\I}{\sigma^4} (f-F)(\y+\z) e^{-\frac{\norm{\z}^2}{2\sigma^2}}d\z\right) \nonumber  \\
	&= h\int (f-F)(\z + \frac{\x+\y}{2}) \left( \omega(\Delta) - \omega(-\Delta) \right) d\z, \label{eqn:cov1}
	\end{align}
	where $\omega(\Delta) := \frac{(\z + \Delta)(\z + \Delta)\trans-\sigma^2\I}{\sigma^4}  e^{-\frac{\norm{\z+ \Delta}^2}{2\sigma^2}}  $ and $\Delta =\frac{\y - \x}{2}$. Equality (\ref{eqn:cov1}) follows from a change of variables. Now, denote $g(\z) := (f-F)(\z+ \frac{\x+\y}{2})$. 
	
	Using $\omega(\Delta) =\frac{(\z + \Delta)(\z + \Delta)\trans-\sigma^2\I}{\sigma^4}  e^{-\frac{\norm{\Delta}^2 + 2\inner{\Delta}{\z}}{2\sigma^2}}e^{-\frac{\norm{\z}^2}{2\sigma^2}} $, we have the following
	\[h\int g(\z) \left( \omega(\Delta) - \omega(-\Delta) \right) d\z 
	=\E_\z\left[g(\z)\left( \omega(\Delta)e^{\frac{\norm{\z}^2}{2\sigma^2}} -  \omega(-\Delta)e^{\frac{\norm{\z}^2}{2\sigma^2}} \right)\right]. \]
	By a Taylor expansion up to only the first order terms in $\Delta$, 
	\[\omega(\Delta)e^{\frac{\norm{\z}^2}{2\sigma^2}} = \frac{\z\z\trans + \Delta\z\trans + \z\Delta\trans - \sigma^2\I}{\sigma^4}(1+\frac{1}{\sigma^2}\inner{\Delta}{\z}).\]
	We then write the Taylor expansion of $\E_\z\left[g(\z)\left( \omega(\Delta)e^{\frac{\norm{\z}^2}{2\sigma^2}} -  \omega(-\Delta)e^{\frac{\norm{\z}^2}{2\sigma^2}} \right)\right]$ as follows.
	\begin{align}
	&\E_\z[g(\z) \cdot \frac{\z\z\trans - \sigma^2\I}{\sigma^4}\cdot \frac{2}{\sigma^2}\inner{\Delta}{\z} + g(\z)\cdot 2\cdot \frac{\Delta\z\trans + \z\Delta\trans }{\sigma^4}]\nonumber\\
	&= \E_\z[g(\z) \cdot \frac{\z\z\trans }{\sigma^4}\cdot \frac{2}{\sigma^2}\inner{\Delta}{\z}] - \E_\z[g(\z) \cdot \frac{ \sigma^2\I}{\sigma^4}\cdot \frac{2}{\sigma^2}\inner{\Delta}{\z}] + \E_\z[g(\z) \cdot2\cdot \frac{\Delta\z\trans + \z\Delta\trans }{\sigma^4}]\nonumber.
	\end{align}
	Therefore,
	\begin{align*}
	&\norm{\hess \tilde{f}_\sigma(\x) - \hess \tilde{F}_\sigma(\x)  - (\hess \tilde{f}_\sigma(\y) - \hess \tilde{F}_\sigma(\y))} \\
	\leq &\norm{\E_\z[g(\z)\cdot \frac{\z\z\trans}{\sigma^4}\cdot \frac{2}{\sigma^2}\inner{\Delta}{\z}] }+ \norm{\E_\z[g(\z)\cdot \frac{\I}{\sigma^4}\cdot 2\inner{\Delta}{\z}] } + \norm{\E_\z[g(\z)\cdot2\cdot \frac{\Delta\z\trans + \z\Delta\trans }{\sigma^4}]} + O(\norm{\Delta}^2)\\
	= &\frac{2}{\sigma^6}\norm{\E_\z[g(\z)\cdot \z\z\trans\inner{\Delta}{\z}] }+ \frac{2}{\sigma^4} \norm{\E_\z[g(\z)\inner{\Delta}{\z}\I] } + \frac{2}{\sigma^4} \norm{\E_\z[g(\z)(\Delta\z\trans + \z\Delta\trans)]} + O(\norm{\Delta}^2)\\
	\leq & O(\frac{\nu }{\sigma^3})\norm{\Delta}  + O(\norm{\Delta}^2).
	\end{align*}
	The last inequality follows from Lemma \ref{small_lemma:norm_exp}.
\end{proof}

\begin{lemma}\label{small_lemma:norm_exp}
	Given $\z \sim N(0,\sigma \I_{d\times d})$, some $\Delta \in \R^d$, $\norm{\Delta} = 1$, and $f:\R^d \to  [-1, 1]$, 
	\begin{align*}
	&1. ~\norm{\E_\z[f(\z)\cdot \z\z\trans\inner{\Delta}{\z}] } = O(\sigma^3); &\qquad 2. ~\norm{\E_\z[f(\z)\inner{\Delta}{\z}\I] } =O(\sigma);\\
	&3. ~\norm{\E_\z[f(\z)(\Delta\z\trans)]} =O(\sigma) ;&\qquad 4. ~\norm{\E_\z[f(\z)(\z\Delta\trans)]} =O(\sigma).
	\end{align*}
	
\end{lemma}
\begin{proof}
	For the first inequality:
	\begin{align*}
	&\norm{\E_\z[f(\z)\cdot \z\z\trans\inner{\Delta}{\z}] } = \sup_{\v\in\R^d, \norm{\v}=1} \E [\v\trans f(\z)\cdot \z\z\trans\inner{\Delta}{\z} \v] \\
	&=  \sup_{\v\in\R^d, \norm{\v}=1} \E [ f(\z)(\v\trans \z)^2(\Delta\trans\z)] 
	\leq \sup_{\v,\Delta\in\R^d, \norm{\v}=\norm{\Delta}=1} \E [ f(\z)(\v\trans \z)^2(\Delta\trans\z)] \\
	\leq& \sup_{\v,\Delta\in\R^d, \norm{\v}=\norm{\Delta}=1}\E [(\v\trans \z)^2|\Delta\trans\z|] 
	\leq\sup_{\v,\Delta\in\R^d, \norm{\v}=\norm{\Delta}=1} \E [|\v\trans \z|^3 +|\Delta\trans\z|^3] \\
	&=  2\E [|{\v^*}\trans \z|^3] = 4\sqrt{\frac{2}{\pi}}\sigma^3.
	\end{align*}%
	For the second inequality:
	\[\norm{\E_\z[f(\z)\inner{\Delta}{\z}\I] } = |\E [X \tilde{f}(X) ]| \leq \E|X| = \sqrt{\frac{2}{\pi}} \sigma, \]
	where $X = \inner{\Delta}{\z} \sim N(0,\sigma^2)$ and $\tilde{f}(a) = \E[f(\z)|X=a] \in[-1,1] $.\newline
	For the third inequality:
	\begin{align*}\norm{\E_\z[f(\z)(\Delta\z\trans)]} &= \sup_{\v\in\R^d, \norm{\v}=1}\E_\z[ f(\z)\v\trans\Delta\z\trans\v)] \\
	&={\v^*}\trans\Delta \E_\z[ f(\z)\z\trans{\v}^*)] 
	\leq \sqrt{\frac{2}{\pi}} \sigma, 
	\end{align*}
	where the last step is correct due to the second inequality we proved.
	The proof of the fourth inequality directly follows from the third inequality.
\end{proof}

\begin{lemma}\label{small_lemma:grad_diff_fsFs}
	$\norm{\grad \tilde{f}_\sigma(\x) - \grad \tilde{F}_\sigma(\x)} \leq \sqrt{\frac{2}{\pi}}\frac{\nu }{\sigma}$.
	
\end{lemma}
\begin{proof}
	By the Gaussian smoothing identity,
	\[
	\norm{\grad \tilde{f}_\sigma(\x) - \grad \tilde{F}_\sigma(\x)} =\norm{\E_\z[\frac{\z}{\sigma^2}(f-F)(\x-\z)]} \leq\sqrt{\frac{2}{\pi}}\frac{\nu}{\sigma}. \]
	The last inequality follows from Lemma \ref{small_lemma:zfz}.
\end{proof}

\begin{lemma}\label{small_lemma:grad_diff_FsF_rho}
	$\norm{\grad \tilde{F}_\sigma(\x) - \grad F(\x)} \leq  \rho d\sigma^2$.
	
\end{lemma}
\begin{proof} By definition of Gaussian smoothing,
	\begin{align}
	&\norm{\grad \tilde{F}_\sigma(\x) - \grad F(\x)} \nonumber \\
	&=  \norm{\grad \E_\z[F(\x-\z)] - \grad F(\x)} 
	\leq \norm{\E_\z[\left(\int_0^1 \nabla^2 f(\x+t\z) dt\right)\z] } \label{eqn:mvt}\\
	&=  \norm{\E_\z[\left(\int_0^1 \nabla^2 f(\x) + \nabla^2 f(\x+t\z) - \nabla^2 f(\x)  dt\right)\z] } \nonumber\\
	\leq  &\norm{\E_\z[\nabla^2 f(\x) \z ] }+ \norm{\E_\z[\left(\int_0^1 \nabla^2 f(\x+t\z) - \nabla^2 f(\x)  dt\right)\z] } \nonumber\\
%	\leq & 0 + \E_\z[\norm{\int_0^1 \nabla^2 f(\x+t\z) - \nabla^2 f(\x)  dt}\norm{\z}] \nonumber\\
	\leq & \E_\z[\left(\int_0^1 \norm{\nabla^2 f(\x+t\z) - \nabla^2 f(\x) } dt\right)\norm{\z}]\nonumber \\
	\leq & \E_\z[\left(\int_0^1 \rho \norm{t\z} dt\right)\norm{\z}] \nonumber
	= \rho \norm{z}^2 \nonumber
	\leq \rho d \sigma^2. \nonumber
	\end{align}
	Inequality (\ref{eqn:mvt}) follows by applying a generalization of mean-value theorem to vector-valued functions.
\end{proof}

\begin{lemma}\label{small_lemma:zfz}
	Given $\z \sim N(0,\sigma \I_{d\times d})$ and $f:\R^d \to  [-1, 1]$, $$\norm{\E \z f(\z)} \leq \sqrt{\frac{2}{\pi}} \sigma.$$
	
\end{lemma}
\begin{proof} By definition of the $2$-norm,
	\begin{align*}
	\norm{\E \z f(\z)} &= \sup_{\v\in\R^d, \norm{\v}=1} \E [\v\trans \z f(\z)] =  \E [{\v^*}\trans \z f(\z)] \\
	&= \E[\E [X f(\z) |X] ] \qquad\text{ where }X = {\v^*}\trans \z \sim N(0,\sigma^2) \\
	&=\E [X \tilde{f}(X) ] \qquad\text{ where } \tilde{f}(a) = \E[f(\z)|X=a] \in[-1,1] \\
	&\leq \E |X| = \sqrt{\frac{2}{\pi}} \sigma.
	\end{align*}
\end{proof}

% !TeX root = appendix_main.tex 

\section{Overview for polynomial queries lower bound} \label{section:lower-bnd-overview}

In this section, we discuss the key ideas for proving Theorem \ref{thm:lowerbound_informal}. 
We illustrate the construction in two steps: (1) construct a hard instance $(F,f)$ contained in a $d$-dimensional ball;
(2) extend this hard instance to $\R^n$. The second step is necessary as Problem~\ref{problem} is an unconstrained problem; in nonconvex optimization the hardness of optimizing unconstrained problems and optimizing constrained problems can be very different.
For simplicity, in this section we assume $\rho,\epsilon$ are both $1$ and focus on the $d$ dependencies, to highlight the difference between polynomial queries and the information-theoretic limit. The general result involving dependency on $\epsilon$ and $\rho$ follows from a simple scaling of the hard functions.% to demonstrate the $1/d$ difference to the information-theoretic limit.

%We introduce the hard instance of problem \ref{problem} and outline the motivations behind our construction. 

%The main idea for the hard function pair is the following. For simplicity, assume $\epsilon = 1, \rho = 1$ in this sketch. First, suppose $F, f$ are only defined in the unit $d$-dimensional ball centered at the origin (refer to the ball in Figure \ref{fig:construction_lower_bnd}). We construct a function $F$ that is bounded, Hessian Lipschitz, and gradient Lipschitz, and whose only second-order stationary points (SOSPs) lie on an axis of the sphere whose direction is a unit vector $\v$ chosen uniformly at random. $f$ coincides with $F$ except on a region $S_\v$ that consists of all points at most distance $~\frac{1}{\sqrt{d}}$ from the origin along the $\v$ direction. Crucially, $f$ on $S_\v$ reveals no information about $\v$. By a standard fact of measure concentration (lemma \ref{lemma_sa_conc_sphere}), we can argue that any random point chosen independently of $\v$ will lie inside $S_\v$ with high probability. Then it follows that any algorithm making a polynomial number of queries on $f$ will fail to output an SOSP of $F$ with high probability. 

\begin{figure}[h]
	\centering
	\begin{minipage}{.50\textwidth}
		\includegraphics[width=0.9\textwidth, trim = 0 0 0 0, clip]{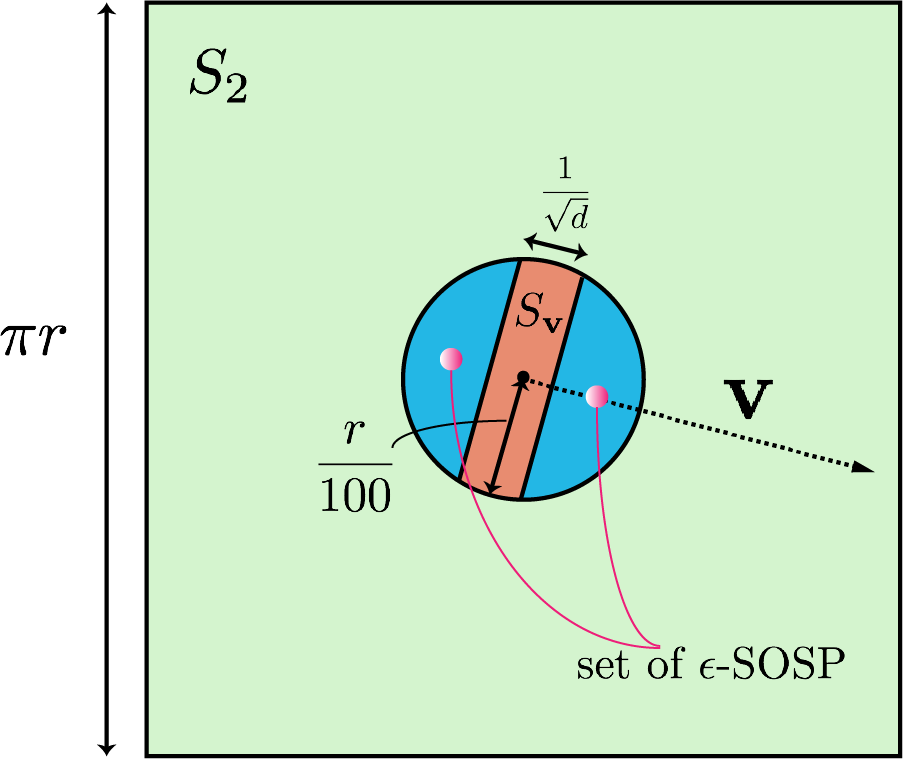}
		\caption{Key regions in lower bound} \label{fig:construction_lower_bnd}
	\end{minipage}%
	\begin{minipage}{.46\textwidth}
		\centering
		\includegraphics[width=0.9\textwidth, trim = 0 0 0 0, clip]{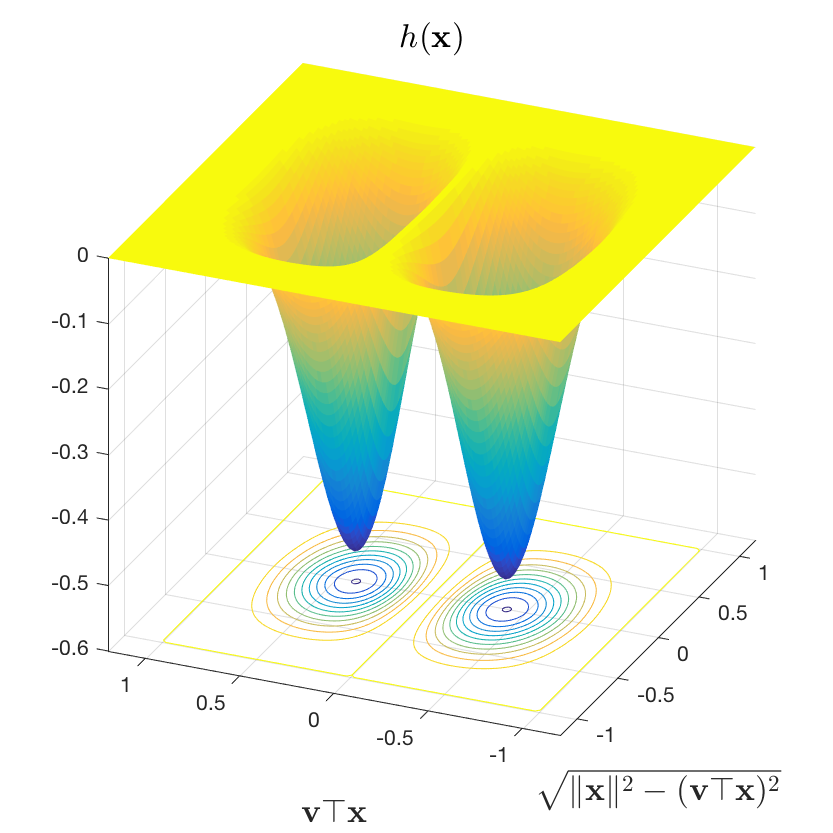}
		\caption{Landscape of $h$}\label{fig:hardfcn}
	\end{minipage}

\end{figure}

\paragraph{Constructing a lower-bound example within a ball} 

The target function $F(\x)$ we construct contains a special direction $\v$ in a $d$-dimensional ball $\mathbb{B}_r$ with radius $r$ centered at the origin. More concretely, let $F(\x) = h(\x) + \norm{\x}^2$, where $h$ (see Figure \ref{fig:hardfcn}) depends on a special direction $\v$, but is spherically symmetric in its orthogonal subspace. 
% Let direction $\v$ (unit vector) uniformly random over sphere, and for each fixed $\v$ we can define a region around its equator of
Let the direction $\v$ be sampled uniformly at random from the $d$-dimensional unit sphere. Define a region around the equator of $\mathbb{B}_r$, denoted $S_\v = \{\x | \x \in \mathbb{B}_r \text{~and~} |\v^\top \x|\le r\log d/\sqrt{d}\}$, as in Figure \ref{fig:construction_lower_bnd}.
% and $g_1$ and $g_2$ has function value as illustrated in \ref{fig:hardfcn}. 
The key ideas of this construction relying on the following three properties:
\begin{enumerate}
\item For any fixed point $\x$ in $\mathbb{B}_r$, we have $\Pr(\x \in S_{\v}) \ge 1-O(1/d^{\log d})$.
% Most measure of $\mathbb{B}_r$ concentrates on $S_\v$: $\text{Vol}(S_\v)/\text{Vol}(\mathbb{B}_r) \ge 1-O(1/d^{\log d})$.
\item The $\epsilon$-SOSP of $F$ is located in a very small set $\mathbb{B}_r - S_\v$.
\item $h(\x)$ has very small function value inside $S_{\v}$, that is, $\sup_{\x \in S_{\v}} |h(\x)| \le \tilde{O}(1/d)$. 
\end{enumerate}
The first property is due to the concentration of measure in high dimensions.
% , but there are some subtleties how we use this intuition in the proof that we defer to section \ref{app:proof:lowerbound} (lemma \ref{lemma:hardfunction}). 
The latter two properties are intuitively shown in Figure~\ref{fig:hardfcn}. %, the first property is due to standard concentration of measure in high dimension (see Lemma \ref{lemma:hardfunction} in appendix for more details).
These properties suggest a natural construction for $f$: 
\begin{equation*}
f(\x) = \begin{cases}
\norm{\x}^2 &\qquad\text{~if~} \x \in S_{\v}\\
F(\x) & \qquad\text{~otherwise~}
\end{cases}.
\end{equation*}
When $\x\in S_{\v}$, by property 3 above we know $|f(\x)-F(\x)|\le \nu = \tilde{O}(1/d)$.
%It is not hard to verify function pair $(f, F)$ satisfy Assumption \ref{assump} with $\nu = \tilde{O}(1/d)$ due to the third fact above and $F(\x) - f(\x) = h(\x)$ in $S_{\v}$ region. 

To see why this construction gives a hard instance of Problem \ref{problem},  
%This basically gives the full setup required for proving of lower bound: 
recall that the direction $\v$ is uniformly random.
%thus unknown to any algorithm beforehead. 
%On one hand, by property 2 above, finding an $\epsilon$-SOSP of $F$ requires approximately identifying the direction of $\v$; on the other hand, 
Since the direction $\v$ is unknown to the algorithm at initialization, the algorithm's first query is independent of $\v$ and thus is likely to be in region $S_{\v}$, due to property~1. The queries inside $S_{\v}$ give no information about $\v$, so any polynomial-time algorithm is likely to continue to make queries in $S_{\v}$ and eventually fail to find $\v$. On the other hand, by property~2 above, finding an $\epsilon$-SOSP of $F$ requires approximately identifying the direction of $\v$, so any polynomial-time algorithm will fail with high probability. %by construction of $f$, is indifferent to direction $\v$ in set $S_{\v}$ where most measure concentrates. It turns out, any adaptive algorithm with polynomial queries will make all queries in $S_{\v}$ with high probability thus reveal no information about $\v$, and fail to find the $\epsilon$-SOSP of $F$.

\paragraph{Extending to the entire space}
To extend this construction to the entire space $\R^d$, we put the ball (the previous construction) inside a hypercube
(see Figure \ref{fig:construction_lower_bnd}) and use the hypercube to tile the entire space $\R^d$. There are two challenges in this approach: (1) The function $F$ must be smooth even at the boundaries between hypercubes; 
(2) The padding region ($S_2$ in Figure \ref{fig:construction_lower_bnd}) between the ball and the hypercube must be carefully constructed to not ruin the properties of the hard functions.

We deal with first problem by constructing a function $\bar{F}(\y)$ on $[-1, 1]^d$, ignoring the boundary condition, and then composing it with a smooth periodic function. For the second problem, we carefully construct a smooth function $h$, as shown in Figure \ref{fig:hardfcn}, to have zero function value, gradient and Hessian at the boundary of the ball and outside the ball, so that no algorithm can make use of the padding region to identify an SOSP of $F$.
Details are deferred to section \ref{app:lower_bnd} in the appendix.

%
%We deal with first problem by first constructing a function $\bar{F}(\y)$ on $[-1, 1]^d$, ignoring the boundary condition, and letting the hard function $F(\x) = \bar{F}(\sin \x)$, where $\sin \x$ denotes $(\sin x_1, \cdots, \sin x_d)$. This immediately gives a periodic and globally smooth function over entire space.
%When the radius of the ball is much smaller than the side length of hypercube, the distortion introduced by the $\sin \x$ function is very small, so the properties of the hard function within the ball are preserved. For the second problem, we carefully construct a smooth function $h$, as shown in Figure \ref{fig:hardfcn}, to have zero function value, gradient and Hessian at the boundary of the ball, that is, its sphere. 
%Since $\bar{F}(\y) = h(\y) + \norm{\y}^2$, the hard function will just be $\norm{\y}^2$ on the sphere. Therefore, the sphere contains no information about $\v$. On the padding region $S_2$, the function $\bar{F}(\y)$ is again equal to $\|\y\|^2$, hence revealing no information about $\v$. Therefore no algorithm can make use of the padding region to identify an SOSP of $F$

% !TeX root = main.tex 
\section{Constructing Hard Functions}\label{app:lower_bnd}

In this section, we prove Theorem \ref{thm:lowerbound_informal}, the lower bound for 
algorithms making a polynomial number of queries. We start by describing the hard function construction that is key to the lower bound.

\subsection{``Scale-free'' hard instance}

We will first present a ``scale-free'' version of the hard function, where we assume $\rho = 1$ and $ \epsilon = 1$. In section \ref{section:hard_instance_scaled}, we will show how to scale this hard function to prove Theorem~\ref{thm:lowerbound_informal}.

%As aforementioned, we now prove the properties of the `scale-free' version of $F, f$, which we denoted as $\sF, \sf$.

	Denote $\sin \x = (\sin (x_1), \cdots, \sin (x_d))$. Let $\Ind\{a\}$ denote the indicator function that takes value $1$ when event $a$ happens and $0$ otherwise. Let $\mu =300 $ . Let the function $\sF:\R^d \to \R$ be defined as follows.
	\begin{equation} 
		\sF(\x) = %\epsilon r \left( h(\sin \frac{1}{r} \x) + D_0 \norm{\sin \frac{1}{r}  \x}^2  \right)
		h(\sin \x) + \norm{\sin  \x}^2, \label{def:hardfunction}
	\end{equation}
	where $h(\y) = h_1( \v\trans \y) \cdot h_2(\sqrt{\norm{\y}^2 - (\v\trans \y)^2})$, and
	\begin{align*}
		h_1(x) &=   g_1(\mu x),\quad  g_1(x)=\left(-16|x|^5+48x^4-48|x|^3+16x^2\right)\Ind\{|x| < 1\}, \\
		h_2(x) &=  g_2(\mu x),\quad g_2(x) =  \left(3x^4-8|x|^3+6x^2-1 \right)\Ind\{|x| < 1\},
	\end{align*}
and the vector $\v$ is uniformly distributed on the $d$-dimensional unit sphere.

	\begin{figure}[h]
		\centering
			\begin{tikzpicture}
			\begin{axis}[
			axis lines = left,
			xlabel = $x$,
			ylabel = {$y$},
			legend pos=south west,
			]
			%Below the red parabola is defined
			\addplot [
			domain=-1:1, 
			samples=100, 
			color=red,
			]
			{-16*abs(x)^5+48*x^4-48*abs(x)^3+16*x^2};
			\addlegendentry{$g_1(x)$}
			%Here the blue parabloa is defined
			\addplot [
			domain=-1:1, 
			samples=100, 
			color=blue,
			]
			{3*x^4-8*abs(x)^3+6*x^2-1};
			\addlegendentry{$g_2(x)$}
			
			\end{axis}
			\end{tikzpicture}
			\caption{Polynomials $g_1, g_2$}
	\end{figure}
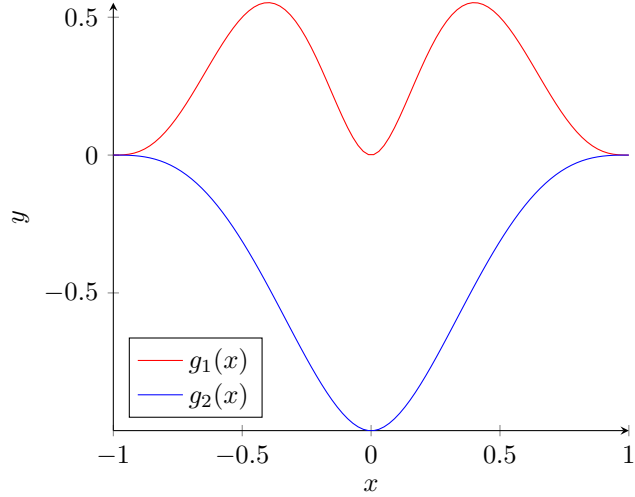

	%To simplify notation, let us write $F = F_{\epsilon, \rho} $ in the ensuing proofs.

	We will state the properties of the hard instance by breaking the space into different regions:
	
	\begin{itemize}
		\item ``ball'' $\sS = \{\x \in \R^d: \norm{\x} \leq 3/\mu \}$ be the $d$-dimensional ball with radius $3/\mu$.
		\item ``hypercube'' $\sH = [-\frac{\pi}{2} ,\frac{\pi}{2} ]^d$ be the $d$-dimensional hypercube with side length $\pi $.
		\item ``band'' $\sS_\v = \{\x\in \sS: \inner{\sin\x}{\v} \leq \frac{\log d}{\sqrt{d}}  \}$
		\item ``padding'' $\sS_2 = \sH - \sS$
		%\item ``non-informative region'' $\tilde{\sS} = \sS_2 \cup \sS_\v$.
	\end{itemize}
	We also call the union of $\sS_2$ and $\sS_\v$ the ``non-informative'' region.
		
		Define the perturbed function, $\sf$: 
		\begin{equation}
		\sf(\x) = \begin{cases}
		\norm{\sin  \x}^2,& \x \in S_\v \\
		\sF(\x),& \x \notin S_\v.
		\end{cases} \label{eqn:perturbedhardfcn}
		\end{equation}

Our construction happens within the ball. However it is hard to fill the space using balls, so we pad the ball into a hypercube. Our construction will guarantee that any queries to the non-informative region do not reveal any information about $\v$. Intuitively the non-informative region is very large so that it is hard for the algorithm to find any point outside of the non-informative region (and learn any information about $\v$).

\begin{lemma}[Properties of scale-free hard function pair $\sF, \sf$] \label{lemma:scale_free_hard_function}
	Let $\sF, \sf, \v$ be as defined in equations \eqref{def:hardfunction}, \eqref{eqn:perturbedhardfcn}. 
		Then $\sF, \sf$ satisfies: 
			\begin{enumerate}
				\item $\sf$ in the non-informative region $\sS_2\cup \sS_\v$ is independent of $\v$.
				\item $\sup_{x\in \sS_{\v}} |\sf-\sF| \leq \tilde{O}(\frac{1}{d})$.% up to a poly-$\log{d}$ factor 
				\item $\sF$ has no SOSP in the non-informative region $\sS_2 \cup \sS_\v$.
				\item $\sF$ is $O(d)$-bounded, $O(1)$-Hessian Lipschitz, and $O(1)$-gradient Lipschitz.
			\end{enumerate}

\end{lemma}

These properties will be proved based on the properties of $h(\y)$, which we defined in \eqref{def:hardfunction} to be the product of two functions. %We first prove this lemma assuming properties of $h(\y)$. %, and then we will complete the proof by filling in the properties of $h(\y)$.%The properties of $h(\y)$

	\begin{proof}
			\textbf{Property 1. }On $\sS_\v$, $\sf(\x) = \sF(\x) = \norm{\sin   \x}^2 $, which is independent of $\v$. On $\sS_2$, we argue that $h(\sin  \x) = 0$ and therefore $\sf(\x) =  \norm{\sin \x}^2 ~\forall \x \in \sS_2$. Note that on $\sS_2$, $\norm{\x} > 3/\mu$ and $(\sin x)^2 > (\frac{2x}{\pi} )^2 ~\forall |x| <  \frac{\pi}{2} $, so 
			\[\norm{\sin  \x} > \norm{\frac{2\x}{\pi }} > \frac{6}{\pi\mu} \implies \max\left\{ \v\trans \sin   \x, \sqrt{ \norm{\sin  \x}^2 - (\v\trans \sin \x)^2} \right\} > \frac{6}{\sqrt{2} \cdot\pi \mu} > \frac{1}{\mu}.\]
			
			Therefore, $h(\sin \x) = h_1(\v\trans \sin   \x)\cdot h_2(\sqrt{ \norm{\sin \x}^2 - (\v\trans \sin  \x)^2}) = 0$.
			
			\noindent
			\textbf{Property 2. } It suffices to show that for $\x \in \sS_\v$, $ | h(\sin  \x)| =  \tilde{O}( \frac{1}{d})$.
			
			For $\x \in \sS_\v$, we have $\v\trans \sin \x \in [- \frac{\log d}{\sqrt{d}},\frac{\log d}{\sqrt{d}}]$. By symmetry, we may just consider the case where $\v\trans\sin   \x  >0$.
			\begin{align*}
			|h(\sin  \x)| &\leq |h_1(\v\trans \sin  \x)| \\
			&=|\left(-16|x|^5+48x^4-48|x|^3+16x^2\right)| \text{ where $x = \mu \v\trans \sin  \x  $}\\
			&\leq C (\frac{\log d}{\sqrt{d}})^2 %\text{ up to a constant multiplicative factor and for all $d$ large enough}
			\\
			&\leq C\frac{\log^2 d}{d} %\text{ up to polylogarithmic factor in $d$}
			\end{align*}
			Here $C > 0$ is a large enough universal constant.

			\noindent
			\textbf{Property 3. }
			\textbf{(Part I.)}  We show that there are no SOSP in $\sS_2$. For $\x: \norm{\x} > 3/\mu$, we argue that either the gradient is large, due to contribution from $\norm{\sin \x}^2$, or the Hessian has large negative eigenvalue (points close to the boundary of $\sH$).
			Denote $G(\x) = \norm{\sin \x}^2 $. We may compute the gradient of $G$ as follows:
			\begin{align*}
			\frac{\partial}{\partial x_i} G(\x) &= 2\sin(x_i)\cos(x_i) = \sin(2x_i), \\
			\norm{\grad G(\x)} &= \sqrt{\sum_i (\sin(2 x_i))^2} \geq  \sqrt{\sum_i x_i^2} = \norm{\x} \text{ for all $\x \in [-\xi,\xi ]^d$},
			\end{align*}
			where $\xi \approx 0.95$ is the positive root of the equation $\sin 2x=x$. On $S_2$, $\grad \sF(\x) =  \grad G(\x)$, so $\norm{\grad \sF(\x)} > \frac{3}{\mu} =1 \times 10^{-2}$ for $\x \in S_2 \cap [-\xi ,\xi  ]^d$.
			We may also compute the Hessian of $G$:
			\begin{align*}
			\hess G(\x) &= diag(2\cos(2\x)), \\
			\lambda_{\min}(\hess G(\x)) &= \min_i\{2\cos(2x_i)\} \leq 2\cdot (\frac{\pi}{4}-x) <2\cdot (\frac{\pi}{4}-\xi)   ~\forall \x \in S_2 \setminus [-\xi ,\xi ]^d.
			\end{align*}
			
			Since $(\frac{\pi}{4}-\xi )< -0.15$, $\lambda_{\min}(\hess \sF(\x)) <-0.3$.
			
			\textbf{(Part II.)} We argue that  $\sF$ has no SOSP in $S_\v$. For $\y = \sin(\x)$, we consider two cases: (i) $z = \sqrt{\norm{\y}^2 - (\v\trans \y)^2}$ large and (ii) $z$ small.
			
			Write $g(\x) = h(\sin\x)$, and denote $\left.\grad h(\x)\right|_{\sin \x},\left.\hess h(\x)\right|_{\sin \x} $ with $\grad h(\y), \hess h(\y)$. Let $\u \circ \v$ denote the Schur product of $\u$ and $\v$. We may compute the gradient and Hessian of $g$:
			
			\[\grad g( \x) = \grad h(\y) \circ \cos( \x),\]
			\[\hess g(\x) = diag( \cos\x)) \hess h(\y)diag( \cos\x)) - \grad h(\y)  \circ \sin( \x). \]
			
			Now we change the coordinate system such that $\v = (1,0, \cdots, 0)$. $\norm{\grad h(\y)}$ and $\lambda_{\min}(\hess h(\y))$ are invariant to such a transform. 
			Under this coordinate system, $h(\y) = h_1(y_1)\cdot h_2(\sqrt{\norm{\y}^2-(y_1)^2})$

			(i): $z \geq \frac{1}{2\mu}$. We show that $\norm{\grad \sF}$ is large.
			
			Let $\proj_{-1}(\u)$ denote the projection of $\u$ onto the orthogonal component of the first standard basis vector.
			
			%For 3, no SOSP, show in case $z$ is large that $\y\trans  \grad h(\y)$ is at the same direction with $\y\trans \grad \norm{\y}^2$, and the second term is large; and in case $z$ is small $\v\trans \hess h(\y) \v$ is very negative, treat remaining terms as error.
			Since $\forall ~i \neq 1, \frac{\partial}{\partial y_i} h(\y) = h_1(y_1 ) h_2'(z)  \frac{y_i}{z }$, we have
			\begin{align*}
			%	\frac{\partial}{\partial y_1} h(\y) &= h_2(z) h_1'(y_1) 
			\proj_{-1}(\grad h(\y)) &= \frac{h_1(y_1 ) h_2'(z)}{z} 	\proj_{-1}(\y) \text{ where } \frac{h_1(y_1 ) h_2'(z)}{z}  > 0\\
			\proj_{-1}(\grad \sF(\x)) &= \proj_{-1}(\grad g(\x) + \grad G(x)) = \proj_{-1}(\grad h(\y) \circ \cos(\x) + \y \circ \cos\x) \\
			&=\left(\frac{h_1(y_1 ) h_2'(z)}{z}	\proj_{-1}(\y) + \proj_{-1}(\y)\right)\circ \cos\x \\
			\norm{\grad \sF (\x)) }&\geq \norm{\proj_{-1}(\grad \sF(\x)) } \geq \norm{\proj_{-1}(\y) \circ \cos\x} \geq z\cdot \min_i|\cos(x_i)| \\
			& \geq  \frac{1}{2\mu}\cdot 0.999 \geq 1\times 10^{-3}~\text{ since $|x_i| \leq 3/\mu$ } 
			\end{align*}
			
			(ii): $z <  \frac{1}{2\mu} $. We show that $\hess \sF(\x)$ has large negative eigenvalue. First we compute the second derivative of $h$ in the direction of the first coordinate: \[\frac{\partial^2h}{\partial y_1^2}= h_2(z) h_1''(y_1) \leq 32\mu^2g_2(1/2) =-10\mu^2.\] 
			Now we use this to upper bound the smallest eigenvalue of $\hess \sF(\x)$.
			
			\begin{align*}
			\lambda_{\min}(\hess h(\y)) &\leq \min_i \frac{\partial^2}{\partial y_i^2} h(\y) \leq \frac{\partial^2h}{\partial y_1^2}\\
			\lambda_{\min}(\hess g(\x)) &\leq \left( \lambda_{\min}\left(  diag(\cos \x )  \hess h(\y)  diag(\cos \x) \right) + \lambda_{\max}\left(- diag(\grad h(\y) \circ \sin(\x)) \right) \right)\\
			&\leq  -0.999\cdot 10 \mu^2  + 0.01\cdot \max_{\y}(h_2(z)h_1'(y_1)+h_1(y_1)h_2'(z))~ \text{ since $|x_i| \leq 3/\mu$}\\
			&\leq -0.999\cdot 10 \mu^2  + 0.01\cdot3\mu \\
			%	&\leq -\sqrt{86400}  \frac{1}{r^2} \text{ for all $d$ large enough and $C > 7\cdot10^{-5}$}\\
			\end{align*}
			Finally,
			\begin{align*}
			\hess G(\x) &= diag(2\cos(2\x)) \implies \norm{\hess G(\x)} \leq 2, \\
			\lambda_{\min}(\hess F(\x)) &\leq -0.999\cdot 10 \mu^2  + 0.01\cdot3\mu +2 \leq -8\times 10^5.
			\end{align*}
			
			\noindent
			\textbf{Property 4. } $O(1)$-bounded: Lemma \ref{lemma:h_function} shows that $|h(\y)| \leq 1$. $\norm{\sin \x}^2 \leq d$. Therefore $|\sF| \leq 1+d$.

			$O(1)$-gradient Lipschitz: $\norm{\hess \sF(\x)}  \leq  \norm{\hess G(\x)} +   \norm{\hess g(\x)} $. We know $\norm{\hess G(\x)} \leq 2$.		
			\begin{align*}
			\norm{\hess g(\x)} &=\norm{ diag(\cos \x )  \hess h(\y)  diag(\cos \x) - diag(\grad h(\y) \circ \sin(\x)) } \\
			&\leq \norm{diag(\cos \x ) }^2\cdot \norm{\hess h(\y)} + \norm{diag(\grad h(\y))}\cdot \norm{diag(\sin \x)}\\
			&\leq 1\cdot 68\mu^2 +3\mu\cdot 1\leq 7\times 10^6 ~\text{ using lemma \ref{lemma:h_function} } 
			\end{align*}
			
			$O(1)$-Hessian Lipschitz: First bound the Hessian Lipschitz constant of $G(\x)$.
			\[ \norm{\hess G(\x) - \hess G(\z)} \leq 2\norm{\cos(2\x) -\cos(2\z) } \leq 4 \norm{\x-\z}. \]
			
			Now we bound the Hessian Lipschitz constant of $g(\x)$. Denote $\A(\x) = diag(\cos \x ) $ and $\B(\x) = diag(\sin \x ) $.			
			\begin{align*}
			&\norm{ \A(\x_1 )  \hess h(\y_1)  \A(\x_1 )  -\A(\x_2 )  \hess h(\y_2) \A(\x_2 )  } \\
			\leq&\| \A(\x_1 )  \hess h(\y_1)  \A(\x_1 )  -  \A(\x_1 )  \hess h(\y_1)  \A(\x_2 ) \| + \| \A(\x_1 )  \hess h(\y_1)  \A(\x_2 )  -  \A(\x_1 )  \hess h(\y_2)  \A(\x_2 )\| \\
			&+	\A(\x_1 )  \hess h(\y_2)  \A(\x_2 )  - \A(\x_2 )  \hess h(\y_2) \A(\x_2 )  \| \\
			\leq & 68\mu^2 \norm{\x_1-\x_2} +1000\mu^3 \norm{\x_1-\x_2} +68\mu^2 \norm{\x_1-\x_2} \text{ from lemma \ref{lemma:h_function}}.\\
			\\
			&\norm{diag(\grad h(\y_1) )\B(\x_1)- diag(\grad h(\y_2)) \B(\x_2)} \\
			\leq & \norm{diag(\grad h(\y_1)) \B(\x_1)- diag(\grad h(\y_1) )\B(\x_2)}+ \norm{diag(\grad h(\y_1) )\B(\x_2)- diag(\grad h(\y_2)) \B(\x_2) } \\ 
			\leq & (3\mu + 68\mu^2)\norm{\x_1-\x_2} \text{ from lemma \ref{lemma:h_function}}.\\
			\\
			&\norm{\hess g(\x_1) - \hess g(\x_2) } \\
			=& (144\mu^3+ 204\mu^2+3\mu )\norm{\x_1-\x_2} .
			\end{align*}
			
			Therefore $\sF(\x)$ is $(2.8\times 10^{10})$-Hessian Lipschitz.
			
	\end{proof}

Now we need to prove smoothness properties of $h(\y)$ that are used in the previous proof. In the following lemma, we prove that $h(\y)$ as defined in equation \eqref{def:hardfunction} is bounded, Lipschitz, gradient-Lipschitz, and Hessian-Lipschitz.

\begin{lemma}[Properties of $h(\y)$] \label{lemma:h_function}
	$h(\y)$ as given in Definition \ref{def:hardfunction} is O(1)-bounded, O(1)-Lipschitz, O(1)-gradient Lipschitz, and O(1)-Hessian Lipschitz.
\end{lemma}
	\begin{proof}
		WLOG assume $\v = (1,0, \cdots, 0)\trans$.
		Denote $u = y_1, \w = (y_2, \cdots, y_d)\trans$. Let $\otimes$ denote tensor product.
		
		Note that $|h_1'| \leq 3\mu, |h_2'| \leq 2 \mu , |h_1''| \leq 32\mu^2 , |h_2''| \leq 12\mu^2 , |h_1'''|\leq 300\mu^3, |h_2'''|\leq 48\mu^3 $. Assume $\mu > 2$.
		\begin{enumerate}
			\item 	O(1)-bounded: $|h| \leq |h_1|\cdot|h_2| \leq 1$.
			\item  O(1)-Lipschitz: $\norm{\grad h(\y)} = \sqrt{h_2(\norm{\w}) h_1'(u) + h_1(u) h_2'(\norm{\w}) } \leq 3\mu \leq O(1)$.
			\item	O(1)-gradient Lipschitz: 
			\[ \hess h(\y) = h_1(u)\hess h_2(\norm{\w}) + h_2(\norm{\w})\hess h_1(u) + \grad h_1(u) \grad h_2(\norm{\w})\trans +\grad h_2(\norm{\w}) \grad h_1(u)\trans.  \]
			$\norm{\grad h_1(u)} \leq 3\mu$. Notice that the following are also O(1):
			\[\norm{\grad h_2(\norm{\w})} = \norm{h_2'(\norm{\w}) \frac{\w}{\norm{\w}} } \leq 2\mu;\]
			\[\norm{\hess h_2(\w)} \leq \norm{h_2''(\norm{\w}) \frac{\w \w\trans}{\norm{\w}^2} }+ \norm{ h_2'(\norm{\w}) \frac{\norm{\w}^2 \I-\w\w\trans}{\norm{\w}^3}} \leq |h_2''(\norm{\w})| + |h_2'(\norm{\w})/\norm{\w}| \leq 12\mu^2 + 12 \mu.\]
			Therefore, $\norm{\hess h(\y) } \leq 24 \mu^2 + 32\mu^2 + 2\cdot2\mu \cdot  3\mu \leq 68 \mu^2$.
			\item O(1)-Hessian Lipschitz:
			We first argue that $\hess h_2(\w)$ is Lipschitz. For $\norm{\w} \geq  1/\mu$, $\hess h_2(\w) = 0$. So we consider $\norm{\w} < \mu$. We obtain the following by direct computation.
			\begin{align*}
			\hess h_2(\w) &= h_2''(\norm{\w}) \frac{\w \w\trans}{\norm{\w}^2}+h_2'(\norm{\w}) \frac{ \norm{\w}^2\I-\w\w\trans }{\norm{\w}^3} \\
			%&= (36\mu^4\norm{\w}^2 - 48\mu^3 \norm{\w} +12\mu^2) \frac{\w\w\trans}{\norm{\w}^2} \\
			%&\quad + (12\mu^4\norm{\w}^2-24\mu^3\norm{\w} +12\mu^2) \left( \frac{\norm{\w}^2\I - \w\w\trans}{\norm{\w}^2} \right) \\
			&=24 \mu^4 \w\w\trans - 24 \mu^3\frac{\w\w\trans }{\norm{\w}} + (12\mu^4\norm{\w}^2 - 24\mu^3\norm{\w} + 12\mu^2) \I \\
			\nabla^3 h_2(\w) &= 24\mu^4(\w\otimes\I + \I \otimes \w) -24\mu^3\frac{\norm{\w}^2(\w\otimes\I + \I \otimes \w) - \w\otimes\w\otimes\w }{\norm{\w}^3} \\
			&+ 24 \mu^4 \I\otimes \w - 24\mu^3 \frac{\I \otimes \w }{\norm{\w}}  \\
			\norm{\nabla^3 h_2(\w) } &\leq 48\mu^3 + 72\mu^2 + 24\mu^3+24\mu^2 \leq 144\mu^3
			\end{align*}
			
			We may easily check that indeed $\lim_{\norm{\w}\to 1}	\norm{\hess h_2(\w) }= 0$.
			
			Therefore $\hess h_2(\w)$ is $144\mu^3$-Lipschitz.			
			\begin{align*}
			\norm{h_1(u_1)\hess h_2(\norm{\w_1})-h_1(u_2)\hess h_2(\norm{\w_1})  } &\leq (144\mu^3+3\mu\cdot 24\mu^2) \norm{\y_1-\y_2}\\
			\norm{h_2(\norm{\w_1})\hess h_1(u_1)-h_2(\norm{\w_2})\hess h_1(u_2) } &\leq (32\mu^2\cdot 2\mu+ 300\mu^3  )\norm{\y_1-\y_2}\\
			\norm{\grad h_1(u_1) \grad h_2(\norm{\w_1})\trans-\grad h_1(u_2) \grad h_2(\norm{\w_2})\trans } &\leq (3\mu\cdot 24\mu^2 + 2\mu\cdot 32\mu^2)\norm{\y_1-\y_2}
			\end{align*}
			
			By triangle inequality, using the above, we obtain \[			\norm{\hess h(\y_1) - \hess h(\y_2) }\leq 1000 \mu^3 \norm{\y_1-\y_2}.\]
			
			This proves that $h(\y)$ is $1000 \mu^3$-Hessian Lipschitz.
		\end{enumerate}
		
	\end{proof}

\subsection{Scaling the Hard Instance}\label{section:hard_instance_scaled}

Now we show how to scale the function we described in order to achieve the final lower bound with correct dependencies on $\epsilon$ and $\rho$.

	Given any $\epsilon, \rho > 0$, define
	\begin{equation}
\tF(\x) = \epsilon r\sF(\frac{1}{r}\x), \tf(\x) = \epsilon r\sf(\frac{1}{r}\x), \label{def:hardfunction:scaled}
	\end{equation} 
	where $r = \sqrt{\epsilon/\rho}$ and $\sF, \sf$ are defined as in Equation~\ref{def:hardfunction}. Define the `scaled' regions:
	
	\begin{itemize}
		\item $\tS = \{\x \in \R^d: \norm{\x} \leq 3r/\mu \}$ be the $d$-dimensional ball with radius $3r/\mu$.
		\item $\tH = [-\frac{\pi}{2}r ,\frac{\pi}{2}r ]^d$ be the $d$-dimensional hypercube with side length $\pi r$.
		\item $\tS_\v = \{\x\in \tS: \inner{\sin \frac{1}{r}\x}{\v} \leq \frac{\log d}{\sqrt{d}}  \}$.
		\item $\tS_2 = \tH - \tS$.
	\end{itemize}

Defined as above, $(\tF, \tf)$ satisfies the properties stated in lemma \ref{lemma:hardfunction}, which makes it hard for any algorithm to optimize $\tF$ given only access to $\tf$. 

%To prove lemma \ref{lemma:hardfunction}, we use and prove a `scale-free' version in lemma \ref{lemma:scale_free_hard_function} that does not involve problem dependent constants, $\epsilon, \rho$. Then in lemma \ref{lemma:hardfunction} we show that $(F,f)$ satisfies the required properties with explicit dependence on $\epsilon, \rho$

\begin{lemma}\label{lemma:hardfunction}
	Let $\tF, \tf, \v, \tS_2, \tS_\v$ be as defined in \ref{def:hardfunction:scaled}.
	Then for any $\epsilon, \rho > 0$, $F, f$ satisfies:
	\begin{enumerate}
		\item $\tf$ in the non-informative region $\tS_2 \cup \tS_\v$ is independent of $\v$.
		\item $\sup_{x\in \tS_{\v}} |\tf-\tF| \leq \frac{\epsilon^{1.5}}{\sqrt{\rho}d}$ up to poly-$\log{d}$ and constant factors.
		\item $\tF$ has no $O(\epsilon)$-SOSP in the non-informative region $\tS_2 \cup \tS_\v$.
		\item $\tF$ is $B$-bounded, $O(\rho)$-Hessian Lipschitz, and $O(\ell)$-gradient Lipschitz.
	\end{enumerate}
\end{lemma}

	\begin{proof}
		This is implied by Lemma \ref{lemma:scale_free_hard_function}. To see this, notice
		
		\begin{enumerate}
			\item We have simply scaled each coordinate axis by $r$.
			\item $|\tF - \tf| = \epsilon r |\sF - \sf| = \frac{\epsilon^{1.5}}{\sqrt{\rho}} |\sF - \sf|$.
			\item $\norm{\grad \tF} = \epsilon \norm{\grad \sF} $ and $\norm{\hess \tF} = \sqrt{\rho \epsilon } \norm{\hess \sF} $. Since $\sF$ has no $1\times 10^{-12}$-SOSP in $S_2 \cup S_\v$. Taking into account the Hessian Lipschitz constant of $\sF$, $\tF$ has no $\frac{\epsilon}{10^{12}}$-SOSP in $\tS_2 \cup \tS_\v$.
			\item We must have $B > d+\frac{\epsilon^{1.5}}{\sqrt{\rho}}$. Then, $\tF$ is $B$-bounded, $(7\times 10^6)\sqrt{\rho \epsilon}$- gradient Lipschitz, and $(2.8\times 10^{10})\rho$-Hessian Lipschitz.
			
			%\lnote{Do we want to change $\mu$ to make these constants smaller (by about 1-3 orders of magnitude)?}
		\end{enumerate}
	\end{proof}

\subsection{Proof of the Theorem}\label{app:proof:lowerbound}

We are now ready to state the two main lemmas used to prove Theorem \ref{thm:lowerbound_informal}.

The following lemma uses the concentration of measure in higher dimensions to argue that the probability that any fixed point lies in the informative region $\tS_\v$ is very small.

\begin{lemma}[Probability of landing in informative region]\label{lemma:fixed_x}
	For any arbitrarily fixed point $\x \in \tS$, $	\Pr(\x \notin \tS_\v) \leq 2e^{-(\log d)^2/2} $.
\end{lemma}

\begin{proof} Recall the definition of $\tS_\v$: $\tS_\v = \{\x\in \tS: \inner{\sin \frac{1}{r}\x}{\v} \leq \frac{\log d}{\sqrt{d}}  \}$.
Since $\x \in \tS$, we have $\norm{\x} \le 3r/\mu \le r$ (as $\mu \ge 3$). Therefore, by inequality $|\sin \theta| \le |\theta|$, we have:
\begin{equation*}
\norm{\sin \frac{\x}{r}}^2 = \sum_{i=1}^d |\sin \frac{x^{(i)}}{r}|^2\le \sum_{i=1}^d |\frac{x^{(i)}}{r}|^2 \le 1.
\end{equation*}
Denote unit vector $\hat{\y} = \sin \frac{\x}{r} / \norm{\sin \frac{\x}{r}}$. This gives:
	\begin{align*}
		\Pr(\x \notin \tS_\v) &= \Pr(|\inner{\sin \frac{\x}{r}}{\v}| \geq\frac{\log d }{\sqrt{d}}) 
		= \Pr(|\inner{\hat{\y}}{\v}| \geq\frac{\log d }{\sqrt{d} \norm{\sin \frac{\x}{r}}} )\\
		&\le \Pr(|\inner{\hat{\y}}{\v}| \geq\frac{\log d }{\sqrt{d}} )   \text{ ~~since~} \norm{\sin \frac{\x}{r}} \le 1\\
		&=\frac{\text{Area}(\{\u: \norm{\u} =1, |\inner{\hat{\y}}{\u}| > \frac{\log d }{\sqrt{d}} \})}{\text{Area}(\{\u:  \norm{\u} =1\})} \\
		&\leq 2e^{-(\log d)^2/2} \text{ by lemma }\ref{lemma_sa_conc_sphere} 
	\end{align*}
	This finishes the proof.
	\end{proof}

Thus we know that for a single fixed point, the probability of landing in $\tS_\v$ is less than $2(1/d)^{\log d/2}$. We note that this is smaller than $1/\poly(d)$.
The following lemma argues that even for a possibly adaptive sequence of points (of polynomial size), the probability that any of them lands in $\tS_\v$ remains small, as long as the query at each point does not reveal information about~$\tS_\v$.

% \cnote{Change this theorem}
\begin{lemma}[Probability of adaptive sequences landing in the informative region]\label{lemma:high_prob}
Consider a sequence of points and corresponding queries with size $T$: $\{(\x_i, q(\x_i))\}_{i=1}^T$, where the sequence can be adaptive, i.e. $\x_t$ can depend on all previous history $\{(\x_i, q(\x_i))\}_{i=1}^{t-1}$. Then as long as $q(\x_i) \perp \v ~|~ \x_i \in \tS_\v$, we have $\Pr(\exists t \le T: \x_{t} \not\in \tS_\v) \le 2Te^{-(\log d)^2/2}$.
\end{lemma}

\begin{proof}
Clearly $\Pr(\exists t \le T: \x_{t} \not\in \tS_\v) = 1- \Pr(\forall t \le T: \x_{t} \in \tS_\v)$. By product rule, we have:
\begin{equation*}
\Pr(\forall t \le T: \x_{t} \in \tS_\v) = \prod_{t=1}^T \Pr(\x_{t} \in \tS_\v | \forall \tau < t: \x_{\tau} \in \tS_\v).
\end{equation*}

Denote $D_i = \{\v \in \mathbb{S}^{d-1}| \inner{\sin \frac{1}{r}\x_i}{\v} > \frac{\log d}{\sqrt{d}}\}$, where $\mathbb{S}^{d-1}$ denotes the unit sphere in $\R^d$ centered at the origin.
Clearly, $\v \not\in D_i$ is equivalent to $\x_i \in \tS_\v$.
Consider term $\Pr(\x_{t} \in \tS_\v | \forall \tau < t: \x_{\tau} \in \tS_\v)$. Conditioned on the event that $E = \{\forall \tau < t: \x_{\tau} \in \tS_\v\}$, we know $\v \in \mathbb{S}^{d-1} - \cup_{i=1}^{t-1} D_i$. On the other hand, since $q(\x_\tau) \perp \v |E$ for all $\tau < t$, therefore, conditioned on event $E$, $\v$ is uniformly distributed  over $\mathbb{S}^{d-1} - \cup_{i=1}^{t-1} D_i$, and:
\begin{equation*}
\Pr(\x_{t} \in \tS_\v | \forall \tau < t: \x_{\tau} \in \tS_\v)
=\frac{\text{Area}(\mathbb{S}^{d-1} - \cup_{i=1}^{t} D_i)}{\text{Area}(\mathbb{S}^{d-1} - \cup_{i=1}^{t-1} D_{i})}.
\end{equation*}
Thus by telescoping:
\begin{equation*}
\Pr(\forall t \le T: \x_{t} \in \tS_\v) = \prod_{t=1}^T\frac{\text{Area}(\mathbb{S}^{d-1} - \cup_{i=1}^{t} D_i)}{\text{Area}(\mathbb{S}^{d-1} - \cup_{i=1}^{t-1} D_{i})} = \frac{\text{Area}(\mathbb{S}^{d-1} - \cup_{i=1}^{T} D_i)}{\text{Area}(\mathbb{S}^{d-1})}.
\end{equation*}
This gives:
\begin{align*}
\Pr(\exists t \le T: \x_{t} \not\in \tS_\v) =& 1- \Pr(\forall t \le T: \x_{t} \in \tS_\v)= \frac{\text{Area}(\cup_{i=1}^{T} D_i)}{\text{Area}(\mathbb{S}^{d-1})}\\
\le & \sum_{i=1}^T\frac{\text{Area}(D_i)}{\text{Area}(\mathbb{S}^{d-1})} \le T \max_i\frac{\text{Area}(D_i)}{\text{Area}(\mathbb{S}^{d-1})}
= T \max_\x \Pr(\x\not\in \tS_\v) \le 2Te^{-(\log d)^2/2}.
\end{align*}
In last inequality, we used Lemma \ref{lemma:fixed_x}, which finishes the proof.
\end{proof}

% \begin{restatable}[Bound on probability for a stochastic process]{lemma}{lemmahighprob}\label{lemma:high_prob}
% 	Consider a stochastic process $X_t$, $t \in \N$, taking values in $\B_0(1)$. Let $S$ be a random subset of $\B_0(1)$ and let $X_1$ be distributed independently of $S$. Suppose for all $X$ independent of $S$, we have $\Pr(X \notin S) \leq \delta$  for some $\delta > 0$, and $X_\tau$ is independent of $S$ if and only if $X_t \notin S \forall t < \tau$. Then for any $T \in \N$,  \[ \Pr(X_t \in S_\v \forall t \leq T) \geq 1-T\delta \]
% \end{restatable}

% \begin{proof}
% 	We proceed by induction. Since $X_1$ is independent of $S$, we have $\Pr(X_1 \in S) \leq \delta$.
% 	Assume $\Pr(X_\tau \in S \forall \tau \in \{0, \cdots, t-1\}  ) \leq (t-1)\delta$ for all $\tau \in \{0, \cdots, t-1\} $. 
	
% 	\begin{align*}
% 		\Pr(X_\tau \in S_\v \forall \tau \leq t)  &= \Pr(X_t \in S | X_\tau \in S \forall \tau \leq t-1  ) \cdot \Pr(X_\tau \in S_\v \forall \tau \leq t-1 ) \\
% 		&\geq (1-\delta)(1-(t-1)\delta) \\
% 		&\geq 1- t\delta
% 	\end{align*}
% \end{proof}

Now we have all the ingredients to prove Theorem \ref{thm:lowerbound_informal}, restated below more formally.

\begin{theorem}[Lower bound]\label{theorem:lowerbound}
	For any $B >0, \ell >0, \rho > 0$, there exists $\epsilon_0 = \Theta(\min\{\ell^2/\rho, (B^2 \rho)^{1/3}\})$ so that for any $\epsilon \in (0, \epsilon_0]$, there exists a function pair ($F, f$) satisfying Assumption \ref{assump} with $\nu = \tilde{\Theta}(\sqrt{\epsilon^3/\rho} \cdot (1/d))$, so that any algorithm will fail, with high probability, to find SOSP of $F$ given only $o(d^{\sqrt{\log d}})$ of zero-th order queries of $f$.
%	If $F, f$ are as defined in Lemma \ref{lemma:hardfunction}, then any algorithm $\alg$ making $o(d^{\sqrt{\log d}})$ function-value queries only queries points in $S_2\cup S_\v$, w.h.p., and fails to output $\epsilon$-SOSP of $F$.
	
\end{theorem}
\begin{proof}
	Take $(\tF, \tf)$ to be as defined in Definition \ref{def:hardfunction:scaled}.
	The proof proceeds by first showing that no SOSP can be found in a constrained set $\tS$, and then using a reduction argument. The key step of the proof involves the following two claims:
	\begin{enumerate}
		\item First we claim that any algorithm $\alg$ making $o(d^{\sqrt{\log d}})$ function-value queries of $f$ to find $\epsilon$-SOSP of $F$ in $\tS$, only queries points in $\tS_\v$, w.h.p., and fails to output $\epsilon$-SOSP of $F$.
		\item Next suppose if there exists $\alg$ making $o(d^{\sqrt{\log d}})$ function-value queries of $f$ that finds $\epsilon$-SOSP of $F$ in $\R^d$ w.h.p. Then this algorithm also finds $\epsilon$-SOSP of $F$ on $\tS$ w.h.p., which is a contradiction.
	\end{enumerate}
	
	Proof of claim 1:
	
	Note that because $\norm{\x_t} \leq r/100$, $\norm{\sin \frac{1}{r}\x_t} \leq \norm{\x_t}/r \leq 1/100 $.
	
	Let $\v$ be an arbitrary unit vector. Suppose a possibly randomized algorithm $\alg$ queries points in $\tS$, $\{X_t\}_{t=1}^T$. Let $\F_t$ denote $\sigma(f(X_1), \cdots, f(X_t))$. Let $X_t \sim \alg(t | \F_{t-1})$.
	
	For any $i$, on the event that $X_i \in \tS_\v$, we have that $f(X_i) = 0$, as established in Lemma \ref{lemma:hardfunction}. Therefore it is trivially true that $f(X_i)$ is independent of $\v$ conditioned on $\{X_i \in \tS_\v \}$.

%	Let $\{\x_t\}$ be an arbitrary sequence of points queried by the algorithm. For each $t$, we have by lemma \ref{lemma:fixed_x}:
%	\[	\Pr(\x_t \notin S_\v) \leq 2e^{-(\log d)^2/2} \]
%	
%	If $X_t$ is independent of $\v$, $$\Pr(X_t \notin S_\v) =\E_{X_t}[\Pr(X_t \notin S_\v | X_t=\x_t)] = \E_{X_t}[\Pr(\x_t \notin S_\v)]\leq 2e^{-(\log d)^2/2} $$
%	
%	Since $\F_1$ is independent of $\v$, $X_1$ must be independent of $\v$. Moreover, as long as $X_t \in S_\v$ for all $t < \tau$, $f(X_t)$ is independent of $\v$ for all $t < \tau$ and so $X_\tau$ is independent of $\v$.
%	
	By Lemma \ref{lemma:high_prob},
	\begin{align*}
	\Pr(X_t \in \tS_\v \forall  t \leq T) &\geq 1 - 2T e^{-(\log d)^2/2} \\
	&\geq 1 - e^{-(\log d )^2/4}\text{ for all $d$ large enough since } T = o(d^{\sqrt{\log d }}). \\
	\end{align*}
	
	Proof of claim 2:
	
	Since $f, F$ are periodic over $d$-dimensional hypercubes of side length $\pi r$, finding $\epsilon$-SOSP of $F$ on $\R$ implies finding $\epsilon$-SOSP of $F$ in $S$. Given claim 1, any algorithm making only $o(d^{\sqrt{\log d}})$ queries will fail to find $\epsilon$-SOSP of $F$ in $\R^d$ w.h.p.
\end{proof}

For completeness, we now state the classical result showing that most of the surface area of a sphere lies close to the equator; it was used in the proof of Lemma \ref{lemma:fixed_x}.

\begin{lemma}[Surface area concentration for sphere]\label{lemma_sa_conc_sphere}
	Let $S^{d-1} = \{x\in\R^d: \norm{x}_2=1\}$ denote the Euclidean sphere in $\R^d$. For $\eps > 0$, let $C(\eps)$ denote the spherical cap of height $\eps$ above the origin. Then $$\frac{\text{Area}(C(\eps))}{\text{Area}(S^{d-1})} \leq e^{-d\eps^2/2}.$$

\end{lemma}
	\begin{proof}
		Let $D$ be the spherical cone subtended at one end by $C(\eps)$ and let $B^d$ denote the unit Euclidean ball in $\R^d$. By Pythagoras' Theorem, we can enclose $D$ in a sphere of radius $\sqrt{1-\eps^2}$. By elementary calculus, \[\frac{\text{Area}(C(\eps))}{\text{Area}(S^{d-1})} =\frac{\text{Volume}(D)}{\text{Volume}(B^d)} \leq \frac{\text{Volume}(\sqrt{1-\eps^2}B^d)}{\text{Volume}(B^d)} \leq (1-\eps^2)^{d/2} \leq e^{-d\eps^2/2}.\]
	\end{proof}

% !TeX root = appendix_main.tex 

\section{Information-theoretic Limits}\label{app:exp}

%\lnote{under construction}

In this section, we prove upper and lower bounds for algorithms that may run in exponential time. This establishes the information-theoretic limit for problem \ref{problem}. Compared to the previous (polynomial time) setting, now the dependency on dimension $d$ is removed.

\subsection{Exponential Time Algorithm to Remove Dimension Dependency}

\begin{algorithm}[t]
	\caption{Exponential Time Algorithm}\label{algo:Exponential}
	\begin{algorithmic}
		\renewcommand{\algorithmicrequire}{\textbf{Input: }}
		\renewcommand{\algorithmicensure}{\textbf{Output: }}
		\REQUIRE function value oracle for $f$, hyperparameter $\epsilon'$
		\STATE Construct (1) $\{\x_t\}_{t=1}^N$, an $O(\epsilon/\ell)$-cover in the euclidean metric of the ball of radius $O(B/\epsilon)$ in $\R^d$ centered at the origin (lemma~\ref{lem:eps_cover}); (2) $\{\v_i\}_{i=1}^V$, an $O(\epsilon/\ell)$-cover in the euclidean metric of the ball of radius $O(\epsilon)$ in $\R^d$ centered at the origin (lemma~\ref{lem:eps_cover}); (3) $\{H_j\}_{j=1}^P$, an $O(\epsilon/\ell)$-cover in the $L_\infty$ metric of the ball of radius $O(\ell)$ in $\R^{d\times d}$ centered at the origin (lemma~\ref{lem:mat_eps_cover}); (4) $\mathcal{Z}$, an $O(\epsilon')$ cover in the euclidean metric of the unit sphere $\mathbb{S}^{d-1}$ in $\R^d$ (lemma~\ref{lem:eps_cover})
		%\STATE Construct a $O(\epsilon/\ell)$-cover of the ball of radius $O(B/\epsilon)$ centered at the origin, as instructed in lemma~\ref{lem:eps_cover}. Let the cover be $\{\x_1, \cdots, \x_N\}$.
		\FOR{$t = 0, 1, \ldots, N$}
		\FOR{$i = 0, 1, \ldots, V$ }
		\FOR{$j = 0, 1, \ldots, P$}
		%\STATE Estimate the gradient and Hessian at the center $\x_t$: $\grad F$, $\hess F$
		\IF{$|f(\x) + \v_i \trans (\y - \x) + \frac{1}{2} (\y-\x)\trans H_j (\y-\x) - f(\y) | \le O(\rho r^3 + \nu)\nn \quad \forall \y = \x + r \z, \quad \z \in \mathcal{Z} \nn$}
		\IF{$\norm{\v_i} \leq O(\epsilon) \text{ and } \lambda_{\min}(H_j) \geq -O(\sqrt{\rho \epsilon} )$ } 
			\STATE \textbf{return} $\x_t$
		\ENDIF
		\ENDIF
		\ENDFOR
		\ENDFOR
		\ENDFOR
		
	\end{algorithmic}
\end{algorithm}

We first restate our upper bound, first stated in Theorem \ref{thm:exp_upperbound}, below.

\begin{theorem}\label{thm:exp_upperbound2}
	There exists an algorithm so that if the function pair ($F, f$) satisfies Assumption \ref{assump} with $\nu \le O(\sqrt{\epsilon^3/\rho})$ and $\ell > \sqrt{\rho \epsilon}$, then the algorithm will find an $\epsilon$-second-order stationary point of $F$ with an exponential number of queries.
\end{theorem}

%\expupperbound*

The algorithm is based on a procedure to estimate the gradient and Hessian at point $x$. This procedure will be applied to a exponential-sized covering of a compact space to find an SOSP. 

Let $\mathcal{Z}$ be a $\epsilon'$ covering for unit sphere $\mathbb{S}^{d-1}$, where $\mathcal{Z}$ is symmetric (i.e. if $\z \in \mathcal{Z}$ then $-\z \in \mathcal{Z}$). It is easy to verify that such covering can be efficiently constructed with $|\mathcal{Z}| \le O((1/\epsilon')^d)$ (Lemma \ref{lem:eps_cover}). Then, for each point in the cover, we solve following feasibility problem:

\begin{align}
\text{find} &\quad  \g, \H \label{prob:feasible}\\
\text{s.t.} &\quad |f(\x) + \g \trans (\y - \x) + \frac{1}{2} (\y-\x)\trans \H (\y-\x) - f(\y) | \le O(\rho r^3 + \nu)\nn \\
& \quad \forall \y = \x + r \z, \quad \z \in \mathcal{Z} \nn,
\end{align}
where $r$ is scalar in the order of $O(\sqrt{\epsilon/\rho})$.

We will first show that any solution of this problem will give good estimates of the gradient and Hessian of $F$.

\begin{lemma}\label{lem:feas_sol_is_close}
Any solution $(\g, \H)$to the above feasibility problem, Eq.\eqref{prob:feasible}, gives
$$\norm{\g - \grad F(\x)} \le O(\epsilon) \text{~~and~~} \norm{\H -\hess F(\x)} \le O(\sqrt{\rho\epsilon})$$
\end{lemma}

\begin{proof}
When we have $\norm{f-F}_{\infty} \le \nu$, above feasibility problem is equivalent to solve following:
\begin{align*}
\text{find} &\quad  \g, \H\\
\text{s.t.} &\quad |F(\x) + \g \trans (\y - \x) + \frac{1}{2} (\y-\x)\trans \H (\y-\x) - F(\y) | \le O(\rho r^3 + \nu)\\
& \quad \forall \y = \x + r \z, \quad \z \in \mathcal{Z}.
\end{align*}
Due to the Hessian-Lipschitz property, we have $|F(\y) - F(\x) - \grad F(\x)\trans (\y-\x) - (\y-\x)\trans \hess F(\x) (\y - \x)| \le \frac{1}{6} \rho r^3$,
this means above feasibility problem is also equivalent to:
\begin{align*}
\text{find} &\quad  \g, \H\\
\text{s.t.} &\quad |(\g - \grad F(\x) \trans (\y - \x) + \frac{1}{2} (\y-\x)\trans (\H-\hess F(\x)) (\y-\x) | \le O(\rho r^3 + \nu)\\
& \quad \forall \y = \x + r \z, \quad \z \in \mathcal{Z}.
\end{align*}
Picking $\y - \x = \pm r\z$, by triangular inequality and the fact that $\mathcal{Z}$ is an $\epsilon'$-covering of $\mathbb{S}^{d-1}$, it is not hard to verify:
\begin{align*}
&\norm{\g - \grad F(\x)} \le O\left(\frac{1}{1-\epsilon'}(\rho r^2 + \frac{\nu}{r})\right) \\
&\norm{\H - \hess F(\x)} \le O\left(\frac{1}{1-2\epsilon'}(\rho r + \frac{\nu}{r^2})\right).
\end{align*}
Given $\nu \le \frac{1}{c}\sqrt{\frac{\epsilon^3}{\rho}}$ for large enough constant $c$, and picking $r = c' \sqrt{\frac{\epsilon}{\rho}}$ with proper constant $c'$,
we prove the lemma.
\end{proof}

We then argue that \eqref{prob:feasible} always has a solution.

\begin{lemma}\label{lem:feas_sol}
	Consider the metric $\norm{\cdot}: \R^{d} \times \R^{d\times d} \to \R$, where $\norm{(\g, H)} = \sqrt{\norm{\g}^2 + \norm{H}^2}$.
Then $(\grad F(\x), \hess F(\x))$ and a $O(\epsilon/\ell)$-neighborhood around it with respect to the $\norm{\cdot}$ metric are the solutions to above feasibility problem.
\end{lemma}
\begin{proof}
$(\grad F(\x), \hess F(\x))$ is clearly one solution to the feasibility problem Eq.\eqref{prob:feasible}. Then, this lemma is true due to Hessian Lipschitz and gradient Lipschitz properties of $F$. %\lnote{why? more details. }
\end{proof}

Now, since the algorithm can do an exhaustive search over a compact space, we just need to prove that there is an $\epsilon$-SOSP within a bounded distance.
%Therefore, the feasibility problem can be solved in exponential time by exhausitive search.

\begin{lemma}\label{lem:sosp_exists_in_ball}
Suppose function $f$ is $B$-bounded, then inside any ball of radius $B/\epsilon$, there must exist a $O(\epsilon/\ell)$-ball full of $2\epsilon$-SOSP.
\end{lemma}

\begin{proof}
We can define a search path $\{\x_t\}$ to find a $\epsilon$-SOSP. Starting from an arbitrary point $\x_0$. (1) If the current point $\x_t$ satisfies
$\norm{\g} \ge \epsilon$, then following gradient direction with step-size $\epsilon/\ell$ decreases the function value by at least $\Omega(\norm{\g}\epsilon/\ell)$;
(2) If the current point $\x_t$ has negative curvature $\gamma \le -\sqrt{\rho\epsilon}$, moving along direction of negative curvature with step-size $\sqrt{\epsilon/\rho}$ decreases the function value by at least $\Omega(\gamma \epsilon/\rho)$.

In both cases, we decrease the function value on average by $\Omega(\epsilon)$ per step. That is in a ball of radius $B/\epsilon$ around $\x_0$, there must be a $\epsilon$-SOSP. and in a $O(\epsilon/\ell)$-ball around this $\epsilon$-SOSP are all $2\epsilon$-SOSP due to the gradient and Hessian Lipschitz properties of $F$.
\end{proof}

Combining all these lemmas we are now ready to prove the main theorem of this section:
%Therefore, $2\epsilon$-SOSP can be identified in exponential time by grid search in a compact space and solving above feasibility problem for each grid point.

\begin{proof}[Proof of Theorem \ref{thm:exp_upperbound}]
%The algorithm is as follows. Pick the ball of radius $O(B/\epsilon)$ in $\R^d$ centered at the origin. Do grid search over the ball with grid size $(\epsilon/\ell)$, and solve the feasibility problem \ref{prob:feasible} for each point on the grid. Stop when a grid point $\x$ is found where its $(\grad F(\x), \hess F(\x))$ satisfies the conditions of $\epsilon$-SOSP. Using the above lemmas, it is easy to see that the algorithm is guaranteed to succeed with number of queries exponential in problem parameters.
We show that Algorithm \ref{algo:Exponential} is guaranteed to succeed within a number of function value queries of $f$ that is exponential in all problem parameters.
First, by Lemma \ref{lem:sosp_exists_in_ball}, we know that at least one of $\{\x_t\}_{t=1}^N$ must be an $O(\epsilon)$-SOSP of $F$. It suffices to show that for any $\x$ that is an $O(\epsilon)$-SOSP, Algorithm \ref{algo:Exponential}'s subroutine will successfully return $\x$, that is, it must find a solution $\g, \H$, to the feasibility problem \ref{prob:feasible} that satisfies $\norm{\g} \leq O(\epsilon) \text{ and } \lambda_{\min}(\H) \geq -O(\sqrt{\rho \epsilon} )$. 

If $\x$ satisfies $\norm{\grad F(\x)} \leq O(\epsilon)$, then by lemma \ref{lem:feas_sol_is_close}, all solutions to the feasibility problem \ref{prob:feasible} at $\x$ must satisfy $\norm{\g} \leq O(\epsilon)$ and we must have $\norm{\H}_\infty \leq \ell$ (implied by $\ell$-gradient Lipschitz). Therefore, by Lemma \ref{lem:feas_sol}, we can guarantee that at least one of $\{\v_i,H_j\}_{i=1,j=1}^{i=V, j=P}$ will be in a solution to the feasibility problem.

Next, notice that because all the covers in Algorithm \ref{algo:Exponential} have size at most $O((d/\epsilon)^{d^2})$ must terminate in $O(e^{d^2\log \frac{d}{\epsilon}})$ steps.
\end{proof}

In the following two lemmas, we provide simple methods for constructing an $\epsilon$-cover for a ball (as well as a sphere), and for matrices with bounded spectral norm.

\begin{lemma}[Construction of $\epsilon$-cover for ball and sphere]\label{lem:eps_cover}
	For a ball in $\R^d$ of radius $R$ centered at the origin, the set of points $C = \{\x \in \R^d: \forall i, \x_i = j\cdot\frac{\epsilon}{\sqrt{d}}, j \in \Z, -\frac{R\sqrt{d}}{\epsilon}-1 \leq j \leq \frac{R\sqrt{d}}{\epsilon} +1 \}$ is an $\epsilon$-cover of the ball, of size $O((R\sqrt{d}/\epsilon)^d)$. Consequently, it is also an $\epsilon$-cover for the sphere of radius $R$ centered at the origin.
\end{lemma}
\begin{proof}
	For any point $\y$ in the ball, we can find $\x \in C$ such that $|\y_i - \x_i|\leq \epsilon/\sqrt{d}$ for each $i\in [d]$. By the Pythagorean theorem, this implies $\norm{\y-\x} \leq \epsilon$.
\end{proof}

\begin{lemma}[Construction of $\epsilon$-cover for matrices with $\ell$-bounded spectral norm]\label{lem:mat_eps_cover}
	Let $\mathcal{M} = \{A\in \R^{d\times d}:\norm{A} \leq \ell \}$ denote the set of $d$ by $d$ matrices with $\ell$-bounded spectral norm. Then the set of points $C = \{M \in \R^{d\times d}: \forall i,k, M_{i,k} = j\cdot\frac{\epsilon}{d}, j \in \Z, -\frac{\ell d }{\epsilon}-1 \leq j \leq \frac{\ell d}{\epsilon} +1 \}$ is an $\epsilon$-cover for $\mathcal{M} $, of size $O((\ell d/\epsilon)^{d^2})$
\end{lemma}
\begin{proof}
	For any matrix $M$ in $\mathcal{M} $, we can find $N \in C$ such that $|N_{i,k} - M_{i,k}| \leq \epsilon/d$ for each $i, k \in [d]$. Since the Frobenius norm dominates the spectral norm, we have $\norm{N-M} \leq \norm{N-M}_F \leq \epsilon$.
\end{proof}

\subsection{Information-theoretic Lower bound}

To prove the lower bound for an arbitrary number of queries, we base our hard function pair on our construction in definition \ref{def:hardfunction}, except $\tf$ now coincides with $\tF$ only outside the sphere $S$. With this construction, no algorithm can do better than random guessing within $S$, since $f$ is completely independent of $\v$.

\begin{theorem}[Information-theoretic lower bound]
	For $\tf,\tF$ defined as follows:
	\[\tF(\x) = \epsilon r\sF(\frac{1}{r}\x), \tf(\x) = \begin{cases}
	\epsilon r\norm{\sin \frac{1}{r} \x}^2  &, \x \in S \\
	\tF(\x) &, \x \notin S \end{cases}\]
	where $\sF$ is as defined in definition \ref{def:hardfunction}.
	Then we have $\sup_\x|\tF(\x) - \tf(\x)| \leq O(\frac{\epsilon^{1.5}}{\sqrt{\rho}}) $ and no algorithm can output SOSP of $F$ with probability more than a constant.
	
\end{theorem}
\begin{proof}
	$\sup_\x|\tF(\x) - \tf(\x)| = \sup_{\x \in S}|\tF(\x) - \tf(\x)| \leq \epsilon r$.
	Any solution output by any algorithm must be independent of $\v$ with probability $1$, since $h = 0$ outside of $S$. Suppose the algorithm $\alg$ outputs $\x$. Then $\Pr(\x \text{ is $\epsilon$-SOSP of $\tF$}) \leq \Pr(\x \notin S_\v) \leq 2 e^{-(\log d)^2/2}$. The upper bound on probability of success does not depend on the number of iterations. Therefore, no algorithm can output SOSP of $F$ with probability more than a constant.
\end{proof}

% !TeX root = appendix_main.tex 

\section{Extension: Gradients pointwise close}\label{app:extension_grad}

In this section, we present an extension of our results to the problem of optimizing an unknown smooth function $F$ (population risk) when given only a gradient vector field $\g: \R^d \to \R^d$ that is pointwise close to the gradient $\grad F$. In other words, we now consider the analogous problem but for a first-order oracle. Indeed, in some applications including the optimization of deep neural networks, it might be possible to have a good estimate of the gradient of the population risk. A natural question is, what is the error in the gradient oracle that we can tolerate to obtain optimization guarantees for the true function $F$? More precisely, we work with the following assumption.

\begin{assumption} \label{assump2} Assume that the function pair ($F:\R^d\to\R, f:\R^d\to\R$) satisfies the following properties:
	\begin{enumerate}
		\item $F$ is $\ell$-gradient Lipschitz and $\rho$-Hessian Lipschitz. 
		\item $f$ is $L$-Lipschitz and differentiable, and $\grad f, \grad F$ are $\gnu$-pointwise close; i.e., $\|\grad f - \grad F \|_\infty \le \gnu$.
	\end{enumerate}
\end{assumption}

We henceforth refer to $\gnu$ as the \emph{gradient error}. As we explained in Section~\ref{sec:prelim}, our goal is to find second-order stationary points of $F$ given only function value access to $\g$. More precisely:

\begin{prob}\label{problem2}
	Given function pair ($F, f$) that satisfies Assumption \ref{assump}, find an $\epsilon$-second-order stationary point of $F$ with only access to function values of $\g = \grad f$.
\end{prob}
%
% \textbf{Assumptions:}
% \begin{enumerate}
% 	\item $F$ is $\ell$-gradient Lipschitz, i.e. $\norm{\grad F(\x) - \grad F(\y)} \le \ell \norm{\x - \y}$. 
% 	\item $F$ is $\rho$-Hessian Lipschitz, i.e. $\norm{\hess F(\x) - \hess F(\y)} \le \rho \norm{\x - \y}$. 
% 	\item $\grad f$ exists, and $F$ and $f$'s gradients are $\tilde{\nu}$-pointwise close: $\sup_{x\in \R^d} \norm{\grad f(x) - \grad F(x)} \leq \tilde{\nu}$
% 	\item $f$ is $L$-Lipschitz
% 	\item (optional) $f$ if $\tilde{\ell}$-gradient Lipschitz
% 	\end{enumerate}
%	
 		%
 		%\begin{remark}
 		%	Removing Assumption 5, we can still get above polynomial queries guarantee, but incur a lot poly factors of $d$.
 		%\end{remark}

We provide an algorithm, Algorithm \ref{algo:FPSGD}, that solves Problem \ref{problem2} for gradient error $\gnu \leq O(\epsilon/\sqrt{d})$. Like Algorithm \ref{algo:PSGD}, Algorithm \ref{algo:FPSGD} is also a variant of SGD whose stochastic gradient oracle, $\g(\x+ \z)$ where $\z \sim \mathcal{N}(0,\sigma^2 \I)$, is derived from Gaussian smoothing.

\begin{algorithm}[t]
	\caption{First order Perturbed Stochastic Gradient Descent (FPSGD)}\label{algo:FPSGD}
	\begin{algorithmic}
		\renewcommand{\algorithmicrequire}{\textbf{Input: }}
		\renewcommand{\algorithmicensure}{\textbf{Output: }}
		\REQUIRE $\x_0$, learning rate $\eta$, noise radius $r$, mini-batch size $m$.
		\FOR{$t = 0, 1, \ldots, $}
		\STATE sample $(\z^{(1)}_t, \cdots, \z^{(m)}_t) \sim \mathcal{N}(0,\sigma^2 \I)$
		\STATE $\g_t(\x_t)  \leftarrow  \sum_{i=1}^m \g(\x_t+\z^{(i)}_t)$
		\STATE $\x_{t+1} \leftarrow \x_t - \eta (\g_t(\x_t) + \xi_t), \qquad \xi_t \text{~uniformly~} \sim \mathbb{B}_0(r)$
		\ENDFOR
		\STATE \textbf{return} $\x_T$
	\end{algorithmic}
\end{algorithm}

 	\begin{theorem}[Rates for Algorithm \ref{algo:FPSGD}]\label{thm:gradients_pw_close}
 			Given that the function pair ($F, f$) satisfies Assumption \ref{assump2} with $\gnu \le O(\epsilon/\sqrt{d})$, then for any $\delta >0$, with smoothing parameter $\sigma=\Theta(\sqrt{\epsilon/(\rho d)})$, learning rate $\eta=1/\ell$, perturbation $r = \tilde{\Theta}(\epsilon)$ and large mini-batch size $m=\poly(d, B, \ell,\rho,1/\epsilon, \log(1/\delta))$, FPSGD will find an $\epsilon$-second-order stationary point of $F$ with probability $1-\delta$, in $\poly(d, B, \ell,\rho,1/\epsilon, \log (1/\delta))$  number of queries.
 		\end{theorem}

Note that Algorithm \ref{algo:FPSGD} doesn't require oracle access to $f$, only to $\g$. We also observe the tolerance on $\tilde{\nu}$ is much better compared to Theorem \ref{thm:upperbound_informal}, as noisy gradient information is available here while only noisy function value is avaliable in Theorem \ref{thm:upperbound_informal}. The proof of this theorem can be found in Appendix~\ref{app:gradients_pw_close}.

% !TeX root = appendix_main.tex 
\section{Proof of Extension: Gradients pointwise close}\label{app:gradients_pw_close}

This section proceeds similarly as in section \ref{app:upper} with the exception that all the results are now in terms of the gradient error, $\gnu$. First, we present the gradient and Hessian smoothing identities (\ref{eqn:smoothed_g} and \ref{eqn:smoothed_g_hess}) that we use extensively in the proofs. In section \ref{app:grad_lemsmoothing}, we present and prove the key lemma on the properties of the smoothed function $\tilde{f}_\sigma(\x)$. Next, in section \ref{app:g_stoc_grad}, we prove the properties of the stochastic gradient $\g(\x+\z)$. Then, using these lemmas, in section \ref{app:proof_grad_pw_close} we prove a main theorem about the guarantees of FPSGD (Theorem~\ref{thm:gradients_pw_close}). For clarity, we defer all technical lemmas and their proofs to section \ref{app:g_tech_lems}.

Recall the definition of the gradient smoothing of a function given in Definition \ref{def:smoothed_f}. In this section we will consider a smoothed version of the (possibly erroneous) gradient oracle, defined as follows.
\begin{equation}\label{eqn:smoothed_g}
\grad \tilde{f}_\sigma(\x) = \E_\z \grad f(\x+ \z).
\end{equation}

Note that indeed $\grad \tilde{f}_\sigma(\x) = \grad \E_\z f(\x+\z)$. We can also write down following identity for the Hessian of the smoothed function.
\begin{equation}\label{eqn:smoothed_g_hess}
\hess \tilde{f}_\sigma(\x) = \E_\z [ \frac{\z}{\sigma^2} \grad f (\x+\z)\trans]
\end{equation}

The proof is a simple calculation.

\begin{proof}[Proof of Equation \ref{eqn:smoothed_g_hess}]
	We proceed by exchanging the order of differentiation. The last equality follows from applying lemma \ref{identity_1} to the function $\fracpar{}{x_j}  f (\x+\z) $
		\[ \fracpar{}{x_i \partial x_j} \tilde{f}_\sigma(\x) = \fracpar{}{x_i} \fracpar{}{x_j} \E_\z [f (\x+\z)] 
		=\fracpar{}{x_i} \E_\z [ \fracpar{}{x_j}  f (\x+\z)] 
		=\E_\z [ \frac{z_i}{\sigma^2} \fracpar{}{x_j}  f (\x+\z)] \]
\end{proof}

\subsection{Properties of the Gaussian smoothing}\label{app:grad_lemsmoothing}

In this section, we show the properties of smoothed function $\grad \tilde{f}_\sigma(\x) $.
\begin{lemma}[Property of smoothing]\label{lem:grad_smoothing}
	Assume function pair ($F, f$) satisfies Assumption \ref{assump2}, and let $\grad \tilde{f}_\sigma(\x)$ be as given in equation \ref{eqn:smoothed_g}. Then, the following holds
	\begin{enumerate}
		\item $\tilde{f}_\sigma(\x)$ is $O(\ell +\frac{\tilde{\nu}}{\sigma})$-gradient Lipschitz and $O(\rho+\frac{\tilde{\nu}}{\sigma^2})$-Hessian Lipschitz.
		\item 
		$\norm{\grad \tilde{f}_\sigma(\x) - \grad F(\x)} \leq  O(\rho d \sigma^2+ \tilde{\nu})$ and
		$\norm{\hess \tilde{f}_\sigma(\x) - \hess F(\x)} \leq O(\rho\sqrt{d}\sigma + \frac{\tilde{\nu}}{\sigma})$.
	\end{enumerate}
\end{lemma}

We will prove the 4 claims of the lemma one by one, in the following 4 sub-subsections.

\subsubsection{Gradient Lipschitz}

We bound the gradient Lipschitz constant of $\tilde{f}_\sigma$ in the following lemma.
\begin{lemma}[Gradient Lipschitz of $\tilde{f}_\sigma$ under gradient closeness]\label{grad_lip_1st_order}
	$\norm{\hess \tilde{f}_\sigma(\x) } \leq O(\ell +\frac{\tilde{\nu}}{\sigma})$.
	\begin{proof}
		By triangle inequality, 
		\begin{align*}
		\norm{\hess \tilde{f}_\sigma(\x) } &= \norm{\hess \tilde{F}_\sigma(\x) + \hess \tilde{f}_\sigma(\x) - \hess \tilde{F}_\sigma(\x)} \\ 
		&\leq \norm{\hess \tilde{F}_\sigma(\x)} +\norm{ \hess \tilde{f}_\sigma(\x) - \hess \tilde{F}_\sigma(\x)} \\
		&\leq \ell + \norm{\E_\z [ \frac{\z}{\sigma^2} (\grad f - \grad F ) (\x+\z)\trans]} \\
		&\leq O (\ell + \frac{\tilde{\nu}}{\sigma} )
		\end{align*}
		The last inequality follows from Lemma  \ref{small_lemma:1st_order_smoothing}.
	\end{proof}
\end{lemma}

\subsubsection{Hessian Lipschitz}

We bound the Hessian Lipschitz constant of $\tilde{f}_\sigma$ in the following lemma.
\begin{lemma}[Hessian Lipschitz of $\tilde{f}_\sigma$ under gradient closeness]\label{hess_lip_1st_order}
	\[\norm{\hess \tilde{f}_\sigma(\x) - \hess \tilde{f}_\sigma(\y) }\leq O(\rho+\frac{\tilde{\nu}}{\sigma^2} ) \norm{\x-\y}.\]
	\end{lemma}
	\begin{proof}
		By triangle inequality:
		\begin{align*}
		&\norm{\hess \tilde{f}_\sigma(\x) - \hess \tilde{f}_\sigma(\y) } \\
		= &\norm{\hess \tilde{f}_\sigma(\x) - \hess \tilde{F}_\sigma(\x)  - \hess \tilde{f}_\sigma(\y) + \hess \tilde{F}_\sigma(\y) +\hess \tilde{F}_\sigma(\x) - \hess \tilde{F}_\sigma(\y)  } 	\\
		\leq &\norm{\hess \tilde{f}_\sigma(\x) - \hess \tilde{F}_\sigma(\x)  - (\hess \tilde{f}_\sigma(\y) - \hess \tilde{F}_\sigma(\y))} +  \norm{\hess \tilde{F}_\sigma(\x) - \hess \tilde{F}_\sigma(\y)  }  \\
		= &O(\frac{\tilde{\nu}}{\sigma^2})\norm{\x-\y} +   O(\rho)\norm{ \x-\y} + O(\norm{\x-\y}^2) 
		\end{align*}
		The last inequality follows from Lemmas \ref{small_lemma:hess_Fs} and \ref{lem:tech_hess}.
		\end{proof}

\subsubsection{Gradient Difference}
We bound the difference between the gradients of smoothed function $\tilde{f}_\sigma(\x) $ and those of the true objective $F$.

\begin{lemma}[Gradient Difference under gradient closeness]\label{grad_diff_1st_order}
	$\norm{\grad \tilde{f}_\sigma(\x) - \grad F(\x)} \leq O(\rho d \sigma^2+ \tilde{\nu})$.
	\begin{proof}
		By triangle inequality:
		\begin{align}
		\norm{\grad \tilde{f}_\sigma(\x) - \grad F(\x)} &\leq \norm{\grad \tilde{f}_\sigma(\x) - \grad \tilde{F}_\sigma(\x)} + \norm{\grad \tilde{F}_\sigma(\x) - \grad F(\x)} \nonumber \\
		&\leq \norm{\E_\z[(\grad f - \grad F)(\x+\z)]} + O(\rho d \sigma^2) \label{eqn:grad_diff_1st_order} \\
		&\leq O(\tilde{\nu} +  \rho d \sigma^2 )\nonumber
		\end{align}
	 The inequality at \eqref{eqn:grad_diff_1st_order}	follows from Lemma \ref{small_lemma:grad_diff_FsF_rho}.
	\end{proof}
\end{lemma}

\subsubsection{Hessian Difference}

We bound the difference between the Hessian of smoothed function $\tilde{f}_\sigma(\x) $ and that of the true objective $F$.

\begin{lemma}[Hessian Difference under gradient closeness]\label{hess_diff_1st_order}
	$\norm{\hess \tilde{f}_\sigma(\x) - \hess F(\x)} \leq O(\rho\sqrt{d}\sigma + \frac{\tilde{\nu}}{\sigma})$
	\begin{proof}
		By triangle inequality:
		\begin{align*}
		\norm{\hess \tilde{f}_\sigma(\x) - \hess F(\x)} &\leq \norm{\hess \tilde{F}_\sigma(\x) - \hess F(\x)}  +  \norm{\hess \tilde{f}_\sigma(\x) - \hess \tilde{F}_\sigma(\x)}  \\
		&\leq O(\rho\sqrt{d}\sigma + \frac{\tilde{\nu}}{\sigma} )
		\end{align*}
		
		The last inequality follows from Lemma \ref{hess_diff} and  \ref{grad_lip_1st_order}.
	\end{proof}
\end{lemma}

\subsection{Properties of the stochastic gradient}\label{app:g_stoc_grad}

\begin{lemma}[Stochastic gradient $\g(\x;\z)$]\label{lem:grad_stoc_grad}
	Let $\g(\x;\z) = \grad f(\x+\z)$, $\z \sim N(0,\sigma \I)$. Then  $\E_\z \g(\x; \z) = \grad \tilde{f}_\sigma(\x)$ and  $\g(\x; \z) $ is sub-Gaussian with parameter $L$.
	\begin{proof}
		For the first claim we simply compute:
		$$\E_\z \g(\x; \z) = \E_\z \grad f(\x+\z) = \grad \E_\z[f(\x+\z)] =\grad \tilde{f}_\sigma(\x).$$
		For the second claim, since function $f$ is L-Lipschitz, we know $\norm{\g(\x, \z)} = \norm{\grad f(\x + \z)} \le L$. 
		This implies that $\g(\x;\z) $ is sub-Gaussian with parameter $L$.		
	\end{proof}
\end{lemma}

\subsection{Proof of Theorem  \ref{thm:gradients_pw_close}}\label{app:proof_grad_pw_close}

Using the properties proved in Lemma~\ref{lem:grad_stoc_grad}, we can apply Theorem~\ref{thm:psgd_guar} to find an $\epsilon$-SOSP for $\tilde{f}_\sigma$. 

We now use lemma  \ref{lem:grad_smoothing} to prove that any $\frac{\epsilon}{\sqrt{d}}$-SOSP of $\tilde{f}_\sigma(\x)$ is also an $O(\epsilon)$-SOSP of $F$.

\begin{lemma}[SOSP of $\tilde{f}_\sigma(\x)$ and SOSP of $F(\x)$]\label{lem:grad_SOSP_close}
	Suppose $\x^*$ satisfies $$\norm{\grad \tilde{f}_\sigma(\x^*) } \leq \tilde{\epsilon}\text{ and } \mineval(\hess \tilde{f}_\sigma(\x^*)) \geq - \sqrt{\tilde{\rho} \tilde{\epsilon}},$$ 
	where $\tilde{\rho} = \rho + \frac{\gnu}{\sigma^2}$ and $\tilde{\epsilon} = \epsilon/\sqrt{d}$.
	Then there exists constants $c_1, c_2$ such that $$\sigma \leq c_1\sqrt{\frac{\epsilon}{\rho d} },~ \gnu \leq c_2 \frac{\epsilon}{\sqrt{d}}$$
	implies $\x^*$ is an $O(\epsilon)$-SOSP of $F$.
\end{lemma}

\begin{proof}
	By Lemma \ref{lem:grad_smoothing} and Weyl's inequality, we have that the following inequalities hold up to a constant factor:
	\begin{align*}
	\norm{\grad F (\x^*)} &\leq \rho d \sigma^2 + \gnu +\tilde{\epsilon}\\
	\mineval(\hess F(\x^*)) &\geq \mineval(\hess \tilde{f}_\sigma(\x^*)) + \mineval(\hess F(\x^*)-\hess \tilde{f}_\sigma(\x^*))  \text{ (Weyl's theorem)} \\
	&\geq - \sqrt{ (\rho + \frac{\gnu}{\sigma^2} )\tilde{\epsilon}} - \norm{\hess \tilde{f}_\sigma(\x) - \hess F(\x)} \\
	&= - \sqrt{\frac{( \rho + \frac{\gnu}{\sigma^2})}{\sqrt{d}} \epsilon} -( \rho\sqrt{d}\sigma + \frac{\gnu}{\sigma} )
	\end{align*}
	Suppose we want any $\tilde{\epsilon}$-SOSP of $\tilde{f}_\sigma(\x)$ to be a $O(\epsilon)$-SOSP of $F$. Then the following is sufficient (up to a constant factor):
	\begin{align}
	\rho\sqrt{d}\sigma + \frac{\gnu}{\sigma} &\leq  \sqrt{\rho \epsilon} \label{g_ineq1}\\
	\rho d \sigma^2+  \gnu&\leq \epsilon\label{g_ineq2} \\
	\rho + \frac{\gnu}{\sigma^2} &\leq \rho \sqrt{d} \label{g_ineq3}
	\end{align}

	We know Eq.(\ref{g_ineq1}), (\ref{g_ineq2}) $\implies \sigma \leq \frac{\sqrt{\rho \epsilon}}{\rho \sqrt{d}} = \sqrt{\frac{\epsilon}{\rho d} }$ and $\sigma \leq \sqrt{\frac{\epsilon}{\rho d} } $. 
	
	\noindent Also Eq. (\ref{g_ineq1}), (\ref{g_ineq2}) $\implies \gnu \leq  \epsilon$ and $\gnu \leq  \sqrt{\rho \epsilon} \sigma \leq \sqrt{\rho \epsilon} \sqrt{\frac{\epsilon}{\rho d} } = \frac{\epsilon}{\sqrt{d}}$.
	
	\noindent Finally Eq.(\ref{g_ineq3}) $\implies \gnu \leq \rho \sqrt{d}  \sigma^2 \leq \frac{\epsilon}{\sqrt{d} }$.

	\noindent Thus the following choices ensures $\x^*$ is an $O(\epsilon)$-SOSP of $F$: $$\sigma \leq \sqrt{\frac{\epsilon}{\rho d} },~ \gnu \leq \frac{\epsilon}{\sqrt{d} }.$$
\end{proof}

\begin{proof}[Proof of Theorem \ref{thm:gradients_pw_close}]
	%Combining Lemma \ref{lem:grad_SOSP_close} with Theorem~\ref{thm:psgd_guar} immediately gives Theorem \ref{thm:gradients_pw_close}.
	Applying  Theorem~\ref{thm:psgd_guar} on $\tilde{f}_\sigma(\x)$ guarantees finding an $c\frac{\epsilon}{\sqrt{d}}$-SOSP of $\tilde{f}_\sigma(\x)$ in number of queries polynomial in all the problem parameters. By Lemma \ref{lem:grad_SOSP_close}, for some universal constant $c$, this is also an $\epsilon$-SOSP of $F$. This proves Theorem~\ref{thm:gradients_pw_close}.
\end{proof}

\subsection{Technical lemmas}\label{app:g_tech_lems}

In this section, we collect and prove the technical lemmas used in section \ref{app:gradients_pw_close}.

\begin{lemma}\label{small_lemma:1st_order_smoothing}
	Let $\z \sim N(0, \sigma^2\I)$, $g: \R^d \to \R^d$, and $\exists a \in \R^+$ s.t. $\norm{g(\x)} \leq a \forall \x \in \R^d$. Let $\Delta \in \R^d$ be fixed. Then,
	\begin{align}
	\norm{\E_\z[\z g(\z)\trans]} \leq \sigma a; \label{eqn1}\\
	\norm{\E_\z[\z\inner{\z}{\Delta} g(\z)\trans]} \leq  \frac{a}{\sigma^2}. \label{eqn2}
	\end{align}
	\begin{proof}
		\begin{align*}
		(\ref{eqn1}): \norm{\E_z[\z g(\z)\trans]} &= \sup_{\v \in \R^d, \norm{\v} = 1} \v\trans (\E_z[\z g(\z)\trans]) \v \\
		&= \E_\z[{\v^*}\trans\z g(\z)\trans \v] \\
		&\leq \sqrt{\E_\z[({\v^*}\trans\z)^2]\E[(g(\z)\trans \v^*)^2]} \\
		&\leq \sqrt{\sigma^2 a^4 } \text{ since } {\v^*}\trans\z \sim N(0,1) \\
		(\ref{eqn2}): \norm{\E_\z[\z\inner{\z}{\Delta} g(\z)\trans]}  &= \sup_{\v \in \R^d, \norm{\v} = 1}  \v\trans\E_\z[\z\inner{\z}{\Delta} g(\z)\trans]\v \\
		&= \E_\z[\inner{\v^*}{\z}\inner{\z}{\Delta} \inner{g(\z)}{\v^*}]\\
		&\leq a \E_\z[|\inner{\v^*}{\z}\inner{\z}{\Delta}|]\\
		&\leq a \sqrt{\E_\z[\inner{\v^*}{\z}^2]\E_\z[\inner{\z}{\Delta}^2] }\\
		&\leq a\norm{\Delta} \sigma^2.
		\end{align*}
	\end{proof}
\end{lemma}

\begin{lemma}\label{lem:tech_hess}
	$\norm{\hess \tilde{f}_\sigma(\x) - \hess \tilde{F}_\sigma(\x)  - (\hess \tilde{f}_\sigma(\y) - \hess \tilde{F}_\sigma(\y)) } \leq O(\frac{\tilde{\nu} }{\sigma^2} ) \norm{\x-\y} + O(\norm{\x-\y}^2) $
\end{lemma}
\begin{proof}
	For brevity, denote $h = \frac{1}{(2\pi\sigma^2)^{\frac{d}{2}}} $.  We have:
		\begin{align}
		&\hess \tilde{f}_\sigma(\x) - \hess \tilde{F}_\sigma(\x)  - (\hess \tilde{f}_\sigma(\y) - \hess \tilde{F}_\sigma(\y)) \nonumber\\
		&= \E_\z[\frac{\z}{\sigma^2} ((\grad f - \grad F)(\x+\z) - (\grad f - \grad F)(\y+\z))\trans] \nonumber\\
		&= h \left( \int \frac{\z}{\sigma^2} (\grad f - \grad F)(\x+\z)\trans e^{-\frac{\norm{\z}^2}{2\sigma^2}}d\z - \int \frac{\z}{\sigma^2} (\grad f - \grad F)(\y+\z)\trans e^{-\frac{\norm{\z}^2}{2\sigma^2}}d\z \right) \nonumber \\
		&= h \left( \int \left((\z+\Delta) e^{-\frac{\norm{\z+\Delta}^2}{2\sigma^2}}-(\z-\Delta)  e^{-\frac{\norm{\z-\Delta}^2}{2\sigma^2}} \right) (\grad f - \grad F)(\z + \frac{\x+\y}{2}) \trans d\z \right), \label{eqn:lemma_tech_hess}
		\end{align}
		where $\Delta = \frac{\y-\x}{2}$. The last equality follows from a change of variables. Now denote $\g(\z):=(\grad f - \grad F)(\z + \frac{\x+\y}{2}) $. By a Taylor expansion up to only the first order terms in $\Delta$, we have
		\begin{align*}
	(\ref{eqn:lemma_tech_hess}) - O(\norm{\Delta})^2&= h ( \int  ((\z+\Delta) (1-\frac{\inner{\z}{\Delta} }{\sigma^2})-(\z-\Delta)  (1+\frac{\inner{\z}{\Delta} }{\sigma^2})) g(\z)\trans e^{-\frac{\norm{\z}^2}{2\sigma^2}}d\z \\
		&= 2h( \int  (\Delta -\z\frac{\inner{\z}{\Delta} }{\sigma^2} ) g(\z)\trans e^{-\frac{\norm{\z}^2}{2\sigma^2}}d\z \\
		&= 2\E_\z [ (\Delta -\z\frac{\inner{\z}{\Delta} }{\sigma^2} ) g(\z)\trans ]. \\
		\end{align*}
		Therefore,
		\begin{align*}
		&\norm{\hess \tilde{f}_\sigma(\x) - \hess \tilde{F}_\sigma(\x)  - (\hess \tilde{f}_\sigma(\y) - \hess \tilde{F}_\sigma(\y)) }\\
		\leq &\frac{2}{\sigma^2} \norm{\E_\z [ (\Delta -\frac{\inner{\z}{\Delta} }{\sigma^2} ) g(\z)\trans }+O(\norm{\Delta}^2) \\
		\leq &\frac{2}{\sigma^2} \norm{\E_\z [ \Delta  g(\z)\trans ]}+\frac{2}{\sigma^4}\norm{\E_\z [ \inner{\z}{\Delta}  g(\z)\trans] }+ O(\norm{\Delta}^2) \\
		\leq& \frac{2}{\sigma^2} \tilde{\nu} \norm{\Delta} +  \frac{2}{\sigma^2}\tilde{\nu} \norm{\Delta}+ O(\norm{\Delta}^2).
		\end{align*}
		The last inequality follows from Lemma \ref{small_lemma:1st_order_smoothing}.

\end{proof}

% !TeX root = appendix_main.tex 
\section{Proof of Learning ReLU Unit}\label{app:relu}
In this section we analyze the population loss of the simple example of a single ReLU unit.

Recall our assumption that $\norm{\w^\star} = 1$ and that the data distribution is $\x \sim \mathcal{N} (0, \I)$; thus,  
\begin{equation*}
y_i = \relu(\x_i\trans \w^\star) + \zeta_i, \qquad \zeta_i \sim \mathcal{N}(0, 1).
\end{equation*}
We use the squared loss as the loss function, hence writing the empirical loss as:
\begin{equation*}
\emploss(\w) = \frac{1}{2n}\sum_{i=1}^n (y_i - \relu(\x_i\trans\w))^2.
\end{equation*}

The main tool we use is a closed-form formula for the kernel function defined by ReLU gates.

\begin{lemma}\citep{cho2009kernel} For fixed $\u, \v$, if $\x \sim \mathcal{N} (0, \I)$, then
\begin{equation*}
\E ~\relu(\x\trans \u)\cdot \relu(\x\trans \v) = \frac{1}{2\pi} \norm{\u}\norm{\v}[\sin\theta + (\pi - \theta) \cos \theta],
\end{equation*}
where $\theta$ is the angel between $\u$ and $\v$ satisfying $\cos\theta = \u\trans \v / (\norm{\u}\norm{\v})$.
\end{lemma}

Then, the population loss has the following analytical form:
\begin{equation*}
\poploss(\w) = \frac{1}{4} \norm{\w}^2  + \frac{5}{4} -  \frac{1}{2\pi}\norm{\w}[\sin \theta + (\pi - \theta) \cos \theta],
\end{equation*}
and so does the gradient ($\hat{\w}$ is the unit vector along $\w$ direction):
\begin{equation*}
\grad \poploss(\w) = \frac{1}{2} (\w - \w^\star) + \frac{1}{2\pi}(\theta\w^\star - \hat{\w} \sin\theta).
\end{equation*}

\subsection{Properties of Population Loss}

We first prove the properties of the population loss, which were stated in Lemma~\ref{lem:ReLU} and we also restate the lemma below.
Let $\neibor = \{\w| \w\trans \w^\star \ge \frac{1}{\sqrt{d}}\} \cap  \{\w| \norm{\w} \le 2\}$.

\begin{lemma} \label{lem:ReLU2}
	The population and empirical risk $\poploss, \emploss$ of learning a ReLU unit problem satisfies:
	% let $\neibor = \{\w| \w\trans \w^\star \ge \frac{1}{\sqrt{d}}\} \cap  \{\w| \norm{\w} \le 2\}$.
	\begin{enumerate}
		\item If $\w_0 \in \neibor$, then runing ZPSGD (Algorithm \ref{algo:PSGD}) gives $\w_t \in \neibor$ for all $t$ with high probability.
		\item Inside $\neibor$, $\poploss$ is $O(1)$-bounded, $O(\sqrt{d})$-gradient Lipschitz, and $O(d)$-Hessian Lipschitz.
		\item $\sup_{\w \in \neibor} | \emploss(\w) -\poploss(\w)| \le \tilde{O}(\sqrt{d/n})$ w.h.p.
		\item Inside $\neibor$, $\poploss$ is nonconvex function, $\w^\star$ is the only SOSP of $\poploss(\w)$.
	\end{enumerate}
\end{lemma}
%\lemrelu*

To prove these four claims, we require following lemmas.

The first important property we use is that the gradient of population loss $l$ has the one-point convex property inside $\neibor$, stated as follows:

\begin{lemma} \label{lem:onepointconvex_RELU}
Inside $\neibor$, we have:
\begin{align*}
\langle - \grad \poploss(\w), \w^\star - \w \rangle \ge \frac{1}{10}\norm{\w - \w^\star}^2.
\end{align*}
\end{lemma}

\begin{proof}
Note that inside $\neibor$, we have the angle $\theta \in [0, \pi/2)$.
Also, let $\mathfrak{W}_\theta = \{\w| \angle(\w, \w^\star) = \theta\}$, then for $\theta \in [0, \pi/2)$:
\begin{align*}
\min_{\w \in \mathfrak{W}_\theta} \norm{\w - \w^\star} = \sin \theta.
\end{align*}
On the other hand, note that $\theta \le 2\sin\theta$ holds true for $\theta \in [0, \pi/2)$; thus we have:
\begin{align*}
\langle - \grad \poploss(\w), \w - \w^\star \rangle
= & \langle \frac{1}{2} (\w - \w^\star) + \frac{1}{2\pi}(\theta\w^\star - \hat{\w} \sin\theta), \w - \w^\star \rangle \\
= & \frac{1}{2}\norm{\w - \w^\star}^2 + \frac{1}{2\pi}\langle[\w^\star(\theta - \sin\theta) + (\w^\star - \hat{\w})\sin\theta], \w - \w^\star\rangle\\
\ge & \frac{1}{2}\norm{\w - \w^\star}^2
- \frac{1}{2\pi} (\sin\theta + \sqrt{2}\sin\theta) \norm{\w - \w^\star}\\
\ge & (\frac{1}{2} - \frac{1 + \sqrt{2}}{2\pi}) \norm{\w - \w^\star}^2 \ge \frac{1}{10}\norm{\w - \w^\star}^2,
\end{align*}
where the second last inequality used the fact that $\sin \theta \le \norm{\w - \w^\star}$ for all $\w \in \neibor$.
\end{proof}

One-point convexity guarantees that ZPSGD stays in the region $\neibor$ with high probability.
\begin{lemma}\label{lem:stay_RELU}
ZPSGD (Algorithm \ref{algo:PSGD}) with proper hyperparameters will stay in $\neibor$ with high probability.
\end{lemma}
\begin{proof}
We prove this by two steps:
\begin{enumerate}
\item The algorithm always moves towards $\x^\star$ in the region $\neibor - \{\norm{\w - \w^\star} \le 1/10\}$.
\item The algorithm will not jump from $\{\norm{\w} \le 1/10\}$ to $\neibor^c$ in one step.
\end{enumerate}
The second step is rather straightforward since the function $\ell(\w)$ is Lipschitz, and the learning rate is small.
The first step is due to the large minibatch size and the concentration properties of sub-Gaussian random variables:
\begin{align*}
&\norm{\w_{t+1} -\w^\star}^2 = \norm{\w_t - \eta (\g_t(\x_t) + \xi_t)  -\w^\star}^2\\
\le &\norm{\w_t - \w^\star}^2 - \eta \langle \grad f_\sigma(\x_t), \w_t -\w^\star \rangle + \eta \norm{\zeta_t}\norm{\w_t - \w^\star}
+ \eta^2 \E \norm{\g_t(\x_t) + \xi_t}^2\\
\le &\norm{\w_t - \w^\star}^2  - \frac{\eta}{10}\norm{\w_t - \w^\star}^2 + \eta\epsilon\norm{\w_t - \w^\star} + \eta^2 \E \norm{\g_t(\x_t) + \xi_t}^2\\
\le &\norm{\w_t - \w^\star}^2  - (\frac{\eta}{100} - \eta\epsilon - O(\eta^2)) \norm{\w_t - \w^\star}
\le 0
\end{align*}
The last step is true when we pick a learning rate that is small enough (although we pick $\eta = 1/\ell$, this is still fine because a $\ell$-gradient Lipschitz function is clearly also a $10\ell$-gradient Lipschitz function) and $\epsilon$ is small.
\end{proof}

\begin{lemma}\label{lem:nonconvex_RELU}
Let $\w(t) = \frac{1}{5} (\w^\star + t \e)$ where $\e$ is any direction so that $\e\trans \w^\star = 0$
\begin{equation*}
\poploss(\w(t)) =\frac{t^2}{100} - \frac{t}{10\pi} + \frac{1}{10\pi} \tan^{-1}(t) + const,
\end{equation*}
which is nonconvex in domain $t \in [0, 1]$. Therefore $f(\w(t))$ is nonconvex along this line segment inside $\neibor$.
\end{lemma}

\begin{proof}
Note that in above setup, $\tan \theta = t$, so the population loss can be calculated as:
\begin{align*}
\poploss(\w(t)) =& \frac{1}{100}\norm{\w}^2 - \frac{1}{2\pi}[\frac{t}{5}  - \tan^{-1}(t)\cdot \frac{1}{5}] + const \\
=& \frac{t^2}{100} - \frac{t}{10\pi} + \frac{1}{10\pi} \tan^{-1}(t) + const
\end{align*}
It's easy to show $\w(t)\in \neibor$ for all $t\in [0, 1]$ and if $g(t) = \poploss(\w(t))$, then $g''(0.6) <0$ and thus the function is nonconvex.
\end{proof}

Next, we show that the empirical risk and the population risk are close by a covering argument.
\begin{lemma} \label{lem:functiondif_RELU}
For sample size $n \ge d$, with high probability, we have:
\begin{equation*}
\sup_{\w \in \neibor} | \emploss(\w) -\poploss(\w)| \le \tilde{O}\left(\sqrt{\frac{d}{n}}\right).
\end{equation*}
\end{lemma}
\begin{proof}
Let $\{\w^{j}\}_{j=1}^J$ be a $\epsilon$-covering of $\neibor$. By triangular inequality:
\begin{align*}
\sup_{\w \in \neibor} | \emploss(\w) -\poploss(\w)|
\le \underbrace{\sup_{\w \in \neibor} | \emploss(\w) -\emploss(\w^j)|}_{T_1}
+ \underbrace{\sup_{j \in J} | \emploss(\w^j) -\poploss(\w^j)|}_{T_2}
+ \underbrace{\sup_{\w \in \neibor} | \poploss(\w^j) -\poploss(\w)|}_{T_3},
\end{align*}
where $\w^j$ is the closest point in the cover to $\w$. Clearly, the $\epsilon$-net of $\neibor$ requires fewer points than the $\epsilon$-net of $\{ \w| \norm{\w} \le 2\}$. By the standard covering number argument, we have $\log N_\epsilon = O(d \log \frac{1}{\epsilon})$. We proceed to bound each term individually.

\textbf{Term $T_2$:} For a fixed $j$, we know $\emploss(\w^j) = \frac{1}{n}\sum_{i=1}^n (y_i - \relu(\x_i\trans\w^j))^2$, 
where $y_i - \relu(\x_i\trans\w^j)$ is sub-Gaussian with parameter $O(1)$, thus $(y_i - \relu(\x_i\trans\w^j))^2$ is sub-Exponential with parameter $O(1)$. We have the concentration inequality:
\begin{equation*}
\P(| \emploss(\w^j) -\poploss(\w^j)| \ge t) \le e^{O(\frac{nt^2}{1 + t})}.
\end{equation*}
By union bound, we have:
\begin{equation*}
\P(\sup_{j\in J} | \emploss(\w^j) -\poploss(\w^j)| \ge t) \le N_{\epsilon} e^{O(\frac{nt^2}{1 + t})}.
\end{equation*}
That is, with $n \ge d$, and probability $1-\delta$, we have:
\begin{equation*}
\sup_{j\in J} | \emploss(\w^j) -\poploss(\w^j)|
\le \sqrt{\frac{1}{n} (\log \frac{1}{\delta} + d\log\frac{1}{\epsilon} )}.
\end{equation*}

\textbf{Term $T_3$:} Since the population loss is $O(1)$-Lipschitz in $\neibor$, we have:
\begin{equation*}
\sup_{\w \in \neibor} | \poploss(\w^j) -\poploss(\w)|
 \le  L |\w^j - \w| \le O(\epsilon).
\end{equation*}

\textbf{Term $T_1$:} Note that for a fixed pair $(\x_i, \y_i)$, the function $g_i(\w) = (y_i - \relu(\x_i\trans\w))^2$ is $O(\norm{\zeta_i}\norm{\x_i} + \norm{\x_i}^2)$-Lipschitz. Therefore,
\begin{align*}
\sup_{\w \in \neibor} | \emploss(\w^j) -\emploss(\w)|
 \le & O(1) \cdot \frac{1}{n}\sum_i\left[\norm{\zeta_i}\norm{\x_i} + \norm{\x_i}^2\right] |\w^j - \w| \\
 \le &O(\epsilon) \cdot \frac{1}{n}\sum_i\left[\norm{\zeta_i}\norm{\x_i} + \norm{\x_i}^2\right].
\end{align*}
With high probability, $\frac{1}{n}\sum_i\left[\norm{\zeta_i}\norm{\x_i} + \norm{\x_i}^2\right]$ concentrates around its mean, $O(d)$.

In summary, we have:
\begin{equation*}
\sup_{\w \in \neibor} | \emploss(\w) -\poploss(\w)| \le \sqrt{\frac{1}{n} (\log \frac{1}{\delta} + d\log\frac{1}{\epsilon} )}
+ O(\epsilon) + O(\epsilon d).
\end{equation*}
By picking $\epsilon$ (for the $\epsilon$-covering) small enough, we finish the proof.
\end{proof}

~

Finally we prove the smoothness of population risk in $\neibor$, we have $1/\sqrt{d} \le \norm{\w}\le 2$.
\begin{lemma}
For population loss $\poploss(\w) = \frac{1}{4} \norm{\w}^2  + \frac{5}{4} -  \frac{1}{2\pi}\norm{\w}[\sin \theta + (\pi - \theta) \cos \theta]$, its gradient and Hessian are equal to:
\begin{align*}
\grad \poploss(\w) =& \frac{1}{2} (\w - \w^\star) + \frac{1}{2\pi}(\theta\w^\star - \hat{\w} \sin\theta), \\
\hess \poploss(\w) =& 
\begin{cases}
\frac{1}{2} \I & \mbox{~if~} \theta = 0 \\
\frac{1}{2} \I -\frac{\sin\theta}{2\pi\norm{\w}}(\I + \hat{\u}\hat{\u}\trans - \hat{\w}\hat{\w}\trans)
& \mbox{~otherwise}
\end{cases},
\end{align*}
where $\hat{\w}$ is the unit vector along the $\w$ direction, and $\hat{\u}$ is the unit vector along the $\w^\star - \hat{\w}\cos \theta$ direction.
\end{lemma}

\begin{proof} Note $\norm{\w^\star} = 1$.
Let $z(\w, \w^\star) = \frac{\w\trans \w^\star}{\norm{\w}}$, we have:
\begin{equation*}
\grad_\w z(\w, \w^\star) = \frac{\w^\star \norm{\w} - (\w\trans \w^\star) \hat{\w} }{\norm{\w}^2}
= \frac{(\w^\star - \hat{\w} \cos \theta)}{\norm{\w}}.
\end{equation*}
Since $\cos \theta = z(\w, \w^\star)$, we obtain:
\begin{equation*}
-\sin \theta \cdot \grad \theta = \frac{(\w^\star - \hat{\w} \cos \theta)}{\norm{\w}}.
\end{equation*}
This gives:
\begin{align*}
\grad \poploss(\w) =& \frac{1}{2} \w
- \frac{1}{2\pi} \hat{\w}[\sin\theta + (\pi -\theta)\cos\theta]
- \frac{1}{2\pi} \norm{\w}[\cos \theta -\cos \theta- (\pi - \theta) \sin \theta] \grad \theta \\
= &\frac{1}{2} \w - \frac{1}{2\pi} \hat{\w}[\sin\theta + (\pi -\theta)\cos\theta]
+ \frac{1}{2\pi} \norm{\w} (\pi - \theta) \sin \theta \cdot \grad \theta \\
= &\frac{1}{2} \w - \frac{1}{2\pi} \hat{\w}[\sin\theta + (\pi -\theta)\cos\theta]
- \frac{1}{2\pi}  (\pi - \theta) (\w^\star - \hat{\w} \cos \theta) \\
= & \frac{1}{2}(\w - \w^\star) + \frac{1}{2\pi}(\theta \w^\star - \hat{\w} \sin\theta)
\end{align*}

Therefore, the Hessian (when $\theta \neq 0$):
\begin{align*}
\hess \poploss(\w) =& \grad [\frac{1}{2}(\w - \w^\star) + \frac{1}{2\pi}(\theta \w^\star - \hat{\w} \sin\theta)] \\
=& \frac{1}{2} \I  + \frac{1}{2\pi}[\grad \theta \cdot (\w^\star - \hat{\w}\cos \theta) \trans]
-\frac{\sin\theta}{2\pi\norm{\w}}(\I - \hat{\w}\hat{\w}\trans)\\
=& \frac{1}{2} \I -\frac{\sin\theta}{2\pi\norm{\w}}(\I + \hat{\u}\hat{\u}\trans - \hat{\w}\hat{\w}\trans),
\end{align*}
where $\hat{\u}$ is the unit vector along $\w^\star - \hat{\w}\cos \theta$ direction. 

And for $\theta = 0$, Hessian $\hess \poploss(\w) = \frac{1}{2}\I$. We prove this by taking the limit.
For $\hat{\v} = \w^\star$
\begin{align*}
\hess  \poploss(\w) \cdot \hat{\v}
= & \lim_{\epsilon \rightarrow 0} \frac{\grad \poploss(\w + \epsilon\hat{\v})
- \grad \poploss(\w)}{\epsilon}
= \frac{1}{2}\hat{\v}.
\end{align*}
For any $\hat{\v} \perp \w^\star$, the angle $\theta$ between $\w + \epsilon \hat{v}$ and $\w^\star$ is $\Theta(\frac{\epsilon}{\norm{\w}})$ up to first order in $\epsilon$, we have:
\begin{align*}
\hess  \poploss(\w) \cdot \hat{\v}
= & \lim_{\epsilon \rightarrow 0} \frac{\grad \poploss(\w + \epsilon\hat{\v})
- \grad \poploss(\w)}{\epsilon} \\
= &\frac{1}{2}\hat{\v} + \frac{1}{2\pi} 
 \lim_{\epsilon \rightarrow 0} \frac{\epsilon \w^\star - (\w^\star + \Theta(\frac{\epsilon}{\norm{\w}})\hat{\v})\cdot \Theta(\frac{\epsilon}{\norm{\w}}) + o(\epsilon)}{\epsilon}
 = \frac{1}{2} \hat{\v}.
\end{align*}
This finishes the proof.
\end{proof}

\begin{lemma} \label{lem:smoothness_RELU}
The population loss function $\poploss$ is $O(1)$-bounded, $O(1)$-Lipschitz, $O(\sqrt{d})$-gradient Lipschitz, and $O(d)$-Hessian Lipschitz. 
\end{lemma}
\begin{proof}
The bounded, Lipschitz, and gradient Lipschitz are all very straightforward given the formula of gradient and Hessian.
We will focus on proving Hessian Lipschitz. Equivalently, we show upper bounds on following quantity:

\begin{align*}
\lim_{\epsilon \rightarrow 0} \frac{\norm{\hess\poploss(\w + \epsilon\hat{\v})
- \hess \poploss(\w)}}{\epsilon}.
\end{align*}

Note that the change in $\theta$ is at most $O(\frac{\epsilon}{\norm{\w}})$, we have:
\begin{align*}
\norm{\hess\poploss(\w + \epsilon\hat{\v}) - \hess \poploss(\w)}
\le 
O(\frac{\epsilon}{\norm{\w}^2})
+ o(\epsilon).
\end{align*}
This gives:
\begin{align*}
\lim_{\epsilon \rightarrow 0} \frac{\norm{\hess\poploss(\w + \epsilon\hat{\v})
- \hess \poploss(\w)}}{\epsilon}
\le O(\frac{1}{\norm{\w}^2}) \le O(d),
\end{align*}
which finishes the proof.

\end{proof}

\begin{proof}[Proof of Lemma \ref{lem:ReLU}] For four claims in Lemma \ref{lem:ReLU},
claim 1 follows from Lemma \ref{lem:stay_RELU};
claim 2 follows from Lemma \ref{lem:smoothness_RELU};
claim 3 follows from Lemma \ref{lem:functiondif_RELU};
claim 4 follows from Lemma~\ref{lem:nonconvex_RELU} and Lemma~\ref{lem:onepointconvex_RELU}.
\end{proof}

\subsection{Proof of Theorem \ref{thm:ReLU}}
\begin{proof}
	The sample complexity $ \tilde{O}(d^4/\epsilon^3)$ can be directly computed from Lemma~\ref{lem:ReLU} and Theorem~\ref{thm:upperbound_informal}. 
    % By simply checking the conditions, the result follows from application of lemma~\ref{lem:ReLU}.
\end{proof}

% !TeX root = main.tex 

\newcommand{\ugrad}{\mathscr{G}}
\newcommand{\ufun}{\mathscr{F}}
\newcommand{\uspace}{\mathscr{S}}
\newcommand{\utime}{\mathscr{T}}
\newcommand{\cXs}{\mathcal{X}^{\xi}_{\text{stuck}}}
\newcommand{\la}{\langle}
\newcommand{\ra}{\rangle}
\newcommand{\ball}{\mathbb{B}}
\newcommand{\modify}[1]{#1 '}
\newcommand{\h}{\bm{h}}

\section{Proof of Stochastic gradient descent}\label{app:sgd}

Here for completeness we give the result for perturbed stochastic gradient descent, which is a adaptation of results in \cite{jin17escape} and will be formally presented in \citet{jin2018sgd}. 
% The results in section are a part of an unpublished manuscript. 

% \cnote{Go over proof once more}

Given stochastic gradient oracle $\g$, where $\E \g(\x; \theta) = \grad f(\x)$, and 

\begin{assumption} \label{assumption_SGD} function $f$ satisfies following property:
\begin{itemize}
\item $f(\cdot)$ is $\ell$-gradient Lipschitz and $\rho$-Hessian Lipschitz.
\item For any $\x \in \R^d$, $\g(\x;\theta)$ has sub-Gaussian tail with parameter $\sigma/\sqrt{d}$.
\end{itemize}
\end{assumption}

\begin{algorithm}[t]
\caption{Perturbed Stochastic Gradient Descent with Minibatch}\label{algo:Mini_PSGD}
\begin{algorithmic}
\renewcommand{\algorithmicrequire}{\textbf{Input: }}
\renewcommand{\algorithmicensure}{\textbf{Output: }}
\REQUIRE $\x_0$, learning rate $\eta$, noise radius $r$.
\FOR{$t = 0, 1, \ldots, $}
\STATE sample $\{\theta^{(1)}_t, \cdots \theta^{(m)}_t\} \sim \mathcal{D}$
\STATE $\g_t(\x_t) \leftarrow \sum_{i=1}^m \g (\x_t;\theta^{(i)}_t) /m $
\STATE $\x_{t+1} \leftarrow \x_t - \eta (\g_t(\x_t)+ \xi_t), \qquad \xi_t \text{~uniformly~} \sim B_0(r)$
\ENDFOR
\STATE \textbf{return} $\x_T$
\end{algorithmic}
\end{algorithm}

\begin{theorem}\label{thm:main_PSGD}
If function $f(\cdot)$ satisfies Assumption \ref{assumption_SGD}, then for any $\delta > 0$, with learning rate $\eta=1/\ell$, perturbation $r = \tilde{\Theta}(\epsilon)$ and large mini-batch size $m=\poly(d, B, \ell,\rho, \sigma, 1/\epsilon, \log(1/\delta))$, PSGD (Algorithm \ref{algo:Mini_PSGD}) will find $\epsilon$-second-order stationary point of $F$ with $1-\delta$ probability in following number of stochastic gradient queries:
\begin{equation*}
\tilde{O}\left(\frac{\ell \Delta_f }{\epsilon^2} \cdot m\right).
\end{equation*}
\end{theorem}

% \begin{theorem}\label{thm:main_PSGD}
% If function $f(\cdot)$ satisfies Assumption 2, then for any $\delta>0$ and we run PSGD (Algorithm \ref{algo:Mini_PSGD}) with parameter chosen as 
% \begin{equation*}
% m = \tilde{\Theta}\left(\frac{\sigma^2 d \cn}{\epsilon^2 \delta^2}\cdot \frac{\Delta^2_f \rho}{\epsilon^3}\right) \quad 
% \eta = \frac{1}{\ell}, \quad
% r = \tilde{\Theta}(\epsilon)
% \end{equation*} 
% where $\cn = \ell/\sqrt{\epsilon\rho}$. Then, PSGD will at least once be $\epsilon-$second order stationary point
% with $1-\delta$ probability in total number of queries:
% \begin{equation*}
% \tilde{O}\left(\frac{\ell \Delta_f }{\epsilon^2} \cdot m\right)
% \end{equation*}
% \end{theorem}

In order to prove this theorem, let
\begin{equation} \label{eq:para_choice}
\eta = \frac{1}{\ell}, \qquad
\utime = \frac{\chi c}{\eta \sqrt{\rho\epsilon}},
\quad
\ufun = \sqrt{\frac{\epsilon^3}{\rho}}\chi^{-3}c^{-5},
\quad
r = \epsilon \chi^{-3}c^{-6},
\quad m=\poly(d, B, \ell,\rho, \sigma, 1/\epsilon, \log(1/\delta)),
\end{equation}
where $c$ is some large constant and $\chi = \max\{1, \log \frac{d \ell \Delta_f}{\rho\epsilon\delta}\}$

\begin{lemma}\label{lem:concentration} for any $\lambda > 0, \delta>0$, if minibatch size $m \ge \frac{2\lambda^2\sigma^2}{\epsilon^2} \log \frac{d}{\delta}$, then for a fixed $\x$, with probability $1-\delta$, we have:
\begin{equation*}
\norm{\grad f(\x) - \frac{1}{m}\sum_{i=1}^m \g (\x;\theta^{(i)}) } \le \frac{\epsilon}{\lambda}.
\end{equation*}
\end{lemma}

This lemma means, when mini-batch size is large enough, we can make noise in the stochastic gradient descent polynomially small.

\begin{lemma}\label{lem:large_grad_GD}
Consider the setting of Theorem \ref{thm:main_PSGD}, if $\norm{\grad f(\x_t)} \ge \epsilon$, 
then by running Algorithm \ref{algo:PSGD}, with probability $1- \delta$, we have $f(\x_{t+1}) - f(\x_t) \le -\eta \epsilon^2/4$.
\end{lemma}
\begin{proof}
By gradient Lipschitz, and the fact $\norm{\xi_t} \le \epsilon/20$ and with minibatch size $m$ large enough, with high probability we have 
$\norm{\grad f(\x_t) - \g_t} \le \epsilon/20$. Let $\zeta_t = \g_t - \grad f(\x_t) + \xi_t$, by triangle inequality, 
we have $\norm{\zeta_t} \le \epsilon/10$ and update equation $\x_{t+1} = \x_t - \eta(\grad f(\x_t) + \zeta_t)$:
\begin{align*}
f(\x_{t+1}) \le& f(\x_t) + \la \grad f(\x_t), \x_{t+1} - \x_t \ra + \frac{\ell}{2}\norm{\x_{t+1} - \x_t}^2 \\
\le& f(\x_t) - \eta\norm{\grad f(\x_t)}^2 + \eta \norm{\grad f(\x_t)}\norm{\zeta_t} + \frac{\eta^2 \ell}{2}\left[\norm{\grad f(\x_t)}^2 + 
2\norm{\grad f(\x_t)}\norm{\zeta_t} + \norm{\zeta_t}^2\right]\\
\le& f(\x_t) - \eta \norm{\grad f(\x_t)}\left[\frac{1}{2}\norm{\grad f(\x_t)} - 2\norm{\zeta_t}\right] + \frac{\eta}{2}\norm{\zeta_t}^2
\le f(\x_t) - \eta \epsilon^2/4
\end{align*}
\end{proof}

\begin{lemma}\label{lem:neg_curve_GD}
Consider the setting of Theorem \ref{thm:main_PSGD}, if $\norm{\grad f(\x_t)} \le \epsilon$
and $\lambda_{\min}(\hess f(\x_t)) \le -\sqrt{\rho\epsilon}$, 
then by running Algorithm \ref{algo:PSGD}, with probability $1-\delta$, we have $f(\x_{t+\utime}) - f(\x_t) \le -\ufun$.
\end{lemma}
\begin{proof}
See next section.
\end{proof}

\begin{proof}[Proof of Theorem \ref{thm:main_PSGD}]
Combining lemma \ref{lem:large_grad_GD} and \ref{lem:neg_curve_GD}, we know
with probability $1-\frac{\Delta_f}{\ufun} \delta$, algorithm will find $\epsilon$-second order stationary point in following iterations:
\begin{align*}
\frac{\Delta_f}{\eta \epsilon^2} + \frac{\Delta_f \utime}{\ufun}
\le O(\frac{2\Delta_f }{\eta \epsilon^2}\chi^4)
\end{align*}
Let $\delta' = \frac{\Delta_f}{\ufun} \delta$ and substitute $\delta$ in $\chi$ with $\delta'$, since $\chi = \max\{1, \log \frac{d \ell \Delta_f}{\rho\epsilon\delta}\}$, this substitution only affects constants. Finally note in each iteration, we use $m$ queries, which finishes the proof.
\end{proof}

\subsection{Proof of Lemma \ref{lem:neg_curve_GD}}

\begin{lemma}\label{lem:potential_decrease}
Let $\eta \le \frac{1}{\ell}$, then we have SGD satisfies:
\begin{equation*}
f(\x_{t+1}) - f(\x_{t}) \le -\frac{\eta}{4}\norm{\grad f(\x_t)}^2 + 5\eta \norm{\zeta_t}^2,
\end{equation*}
where $\zeta_t = \g_t - \grad f(\x_t) + \xi_t$.
\end{lemma}
\begin{proof}
By assumption, function $f$ is $\ell$-gradient Lipschitz, we have:
\begin{align*}
f(\x_{t+1}) \le & f(\x_t) + \la\grad f(\x_t), \x_{t+1} - \x_t \ra + \frac{\ell}{2} \norm{\x_{t+1} - \x_t}^2 \\
 \le & f(\x_t)  - \eta\la\grad f(\x_t), \grad f(\x_t) + \zeta_t\ra + \frac{\eta^2\ell}{2} (\norm{\grad f(\x_t)}^2 + 2\norm{\grad f(\x_t)}\norm{\zeta_t} + \norm{\zeta_t}^2)\\
 \le & f(\x_t)  - \frac{\eta}{2}\norm{\grad f(\x_t)}^2 + 2\eta \norm{\grad f(\x_t)}\norm{\zeta_t} + \frac{\eta}{2}  \norm{\zeta_t}^2 \\
 \le& f(\x_t) - \frac{\eta}{4}\norm{\grad f(\x_t)}^2  + \frac{9\eta}{2} \norm{\zeta_t}^2
\end{align*}
which finishes the proof.
\end{proof}

\begin{lemma}(Improve or Localize) \label{lem:locality}
Suppose $\{\x_t\}_{t=0}^T$ is a SGD sequence, then
for all $t \le T$:
$$\norm{\x_t - \x_0}^2 \le 8\eta T(f(\x_0) - f(\x_T)) + 50\eta^2 T \sum_{t=0}^{T-1}\norm{\zeta_t}^2,$$
where $\zeta_t = \g_t - \grad f(\x_t) + \xi_t$
\end{lemma}
\begin{proof}
For any $t \le T-1$, by Lemma \ref{lem:potential_decrease}, we have:
\begin{align*}
\norm{\x_{t+1} - \x_{t}}^2 \le& \eta^2\norm{\grad f(\x_t) + \zeta_t}^2
\le 2\eta^2 \norm{\grad f(\x_t)}^2 + 2\eta^2 \norm{\zeta_t}^2 \\
\le& 8\eta (f(\x_{t+1} - \x_t)) + 50 \eta^2 \norm{\zeta_t}^2
\end{align*}
By Telescoping argument, we have:
\begin{align*}
\sum_{t=0}^{T-1}\norm{\x_{t+1} - \x_{t}}^2
\le 8\eta (f(\x_{T}) - f(\x_0)) + 50 \eta^2 \sum_{t=0}^{T-1}\norm{\zeta_t}^2
\end{align*}
Finally, by Cauchy-Schwarz, we have for all $t \le T$:
\begin{align*}
\norm{\x_t - \x_0}^2 \le &(\sum_{\tau = 1}^t \norm{\x_\tau - \x_{\tau - 1}})^2
\le t \sum_{\tau = 0}^{t-1} \norm{\x_{\tau+1} - \x_{\tau}}^2
\le T \sum_{\tau = 0}^{T-1} \norm{\x_{\tau+1} - \x_{\tau}}^2
\end{align*}
which finishes the proof.
\end{proof}

To study escaping saddle points, we need a notion of coupling. Recall the PSGD update has two source of randomness:
$\g_t - \grad f(\x_t)$ which is the stochasticity inside the gradient oracle and $\xi_t$ which is the perturbation we deliberately added into the algorithm to help escape saddle points.
% Since two are random, the distribution of $\xi_t$ also does not depends on the location of iterates. We can have an equivalent view of PSGD, which instead of sample perturbation $\xi_t$ at each iteration, we sample the entire sequence $\{\xi_1, \xi_2, \cdots\}$ even before running the algorithm. The second view will give the same dynamics as the first view.
Let $\text{SGD}_{\xi}^{(t)}(\cdot)$ denote the update via SGD $t$ times with perturbation $\xi = \{\xi_2, \cdots\}$ fixed. Define Stuck region:
\begin{equation}
\cXs(\tilde{\x}) = \{\x | \x \in \ball_{\tilde{\x}}(\eta r), \text{~and~} \Pr(f(\text{SGD}_{\xi}^{(\utime)}(\x)) - f(\tilde{\x}) \ge -\ufun) \ge \sqrt{\delta}\}
\end{equation}
Intuitively, the later perturbations of coupling sequence are the same, while the very first perturbation is used to escape saddle points. 
% \cnote{Define Stuck region for one specific perturbation $\{\xi_t\}$, then two sequence is independent, we proved $\Pr(\x_1 \in \cXs(\tilde{\x}) \text{~and~}
% \x_2 \in \cXs(\tilde{\x})) \le \delta$, therefore $\min\{\Pr(\x_1 \in \cXs(\tilde{\x}), \Pr(\x_2 \in \cXs(\tilde{\x})\} \le \sqrt{\delta}$ }

\begin{lemma}\label{lem:width}
There exists large enough constant $c$, so that if $\norm{\grad f(\tilde{\x})} \le \epsilon$
and $\lambda_{\min}(\hess f(\tilde{\x})) \le -\sqrt{\rho\epsilon}$, then the width of $\cXs(\tilde{\x})$ along 
the minimum eigenvector direction of $\tilde{\x}$ is at most $\delta \eta r \sqrt{2 \pi/d}$.
\end{lemma}

\begin{proof}
To prove this, let $\e_{\min}$ be the minimum eigenvector direction of $\hess f(\tilde{\x})$, it suffices to show for any $\x_0, \modify{\x}_0\in \ball_{\tilde{\x}}(\eta r)$ so that $\x_0 - \modify{\x}_0 = \lambda \e_{\min}$ where $|\lambda| \ge \delta \eta r \sqrt{2\pi/d}$, then either $\x_0 \not \in \cXs(\tilde{\x})$ or $\modify{\x}_0 \not \in\cXs(\tilde{\x})$.
Let $\x_\utime = \text{SGD}^{(\utime)}(\x_0)$ and $\modify{\x}_\utime = \text{SGD}^{(\utime)}(\modify{\x}_0)$ where two sequence are independent. 
To show $\x_0 \not \in \cXs(\tilde{\x})$ or $\modify{\x}_0 \not \in\cXs(\tilde{\x})$. We first argue showing following with probability $1-\delta$ suffices:
\begin{equation} \label{eq:SGD_width_obj}
\min\{f(\x_\utime) - f(\tilde{\x}), f(\modify{\x}_\utime) - f(\tilde{\x})\} \le - \ufun
\end{equation}
Since $\x_\utime$ and $\modify{\x}_0$ are independent, we have 
$$\Pr(\x_1 \in \cXs(\tilde{\x}) \cdot \Pr(\x_2 \in \cXs(\tilde{\x}) = \Pr(\x_1 \in \cXs(\tilde{\x}) \text{~and~} \x_2 \in \cXs(\tilde{\x})) \le \delta$$
This gives $\min\{\Pr(\x_1 \in \cXs(\tilde{\x}), \Pr(\x_2 \in \cXs(\tilde{\x})\} \le \sqrt{\delta}$ i.e. $\x_0 \not \in \cXs(\tilde{\x})$ or $\modify{\x}_0 \not \in\cXs(\tilde{\x})$ by definition.

In the remaining proof, we will proceed proving Eq.\eqref{eq:SGD_width_obj} by showing two steps:
\begin{enumerate}
\item $\max\{f(\x_0) - f(\tilde{\x}), f(\modify{\x}_0) - f(\tilde{\x})\} \le  \ufun$
\item $\min\{f(\x_\utime) - f(\x_0), f(\modify{\x}_\utime) - f(\modify{\x}_0)\} \le - 2\ufun$ with probability $1-\delta$
\end{enumerate}
The final result immediately follow from triangle inequality.

\textbf{Part 1.} Since $\x_0 \in \ball_{\tilde{\x}}(\eta r)$ and $\norm{\grad f(\x)} \le \epsilon$, by smoothness, we have:
\begin{equation*}
f(\x_0) - f(\tilde{\x}) \le \epsilon \eta r + \frac{\ell}{2}(\eta r)^2
\le O(\frac{\epsilon^2}{\ell} \chi^{-3} c^{-6}) \le \ufun
\end{equation*}
The last inequality is due to $\ell/\sqrt{\rho\epsilon} \ge 1$, and constant $c$ large enough. By symmetry, we can also prove same upper bound for 
$f(\modify{\x}_0) - f(\tilde{\x})$.

~

\textbf{Part 2.} Assume the contradiction $\min\{f(\x_\utime) - f(\x_0), f(\modify{\x}_\utime) - f(\modify{\x}_0)\} \ge - 2\ufun$, by Lemma \ref{lem:locality}
(note $\norm{\zeta_t} \le \norm{\g_t - \grad f(\x_t)} + \norm{\xi_t} \le 2r$ with high probability when $m$ is large enough), with $1-\delta/2$ probability, this implies localization:
\begin{align*}
\forall t\le \utime, \quad &\max\{\norm{\x_t - \tilde{\x}}, \norm{\modify{\x}_t - \tilde{\x}}\} \\
\le& \max\{\norm{\x_t - \x_0} + \norm{\x_0 - \tilde{\x}}, \norm{\modify{\x}_t - \modify{\x}_0} + \norm{\modify{\x}_0 - \tilde{\x}}\} \\
\le& \sqrt{8\eta \utime\ufun + 50\eta^2 \utime \epsilon^2 \chi^{-4}c^{-4}} + \eta r \defeq \uspace = O(\sqrt{\frac{\epsilon}{\rho}} \chi^{-1} c^{-2})
\end{align*}
That is, both SGD sequence $\{\x_t\}_{t=0}^{\utime}$ and $\{\modify{\x}_t\}_{t=0}^{\utime}$ will not leave a local ball with radius $\uspace$ around $\tilde{\x}$. Denote $\H = \hess f(\tilde{\x})$. By stochastic gradient update $\x_{t+1} = \x_t - \eta (\g_t(\x_t)+ \xi_t)$, we can track the difference sequence $\w_t \defeq \x_t- \modify{\x}_t$ as:
\begin{align*}
\w_{t+1} =& \w_t -\eta [\grad f(\x_t) - \grad f(\modify{\x}_t)] - \eta \h_t
= (\I - \eta\H)\w_t - \eta (\Delta_t \w_t + \h_t) \\
=& (\I - \eta\H)^{t+1}\w_0 - \eta\sum_{\tau = 0}^t(\I - \eta\H)^{t-\tau} (\Delta_\tau \w_\tau  + \h_\tau),
\end{align*}
where $\H = \hess f(\tilde{\x})$ and $\Delta_t = \int_{0}^1 [\hess f(\modify{\x}_t + \theta (\x_t -\modify{\x}_t) - \H]\mathrm{d} \theta $
and $\h_t = \g_t(\x_t) - \g_t(\modify{\x}_t) - [\grad f(\x_t) - \grad f(\modify{\x}_t)]$. By Hessian Lipschitz, we have
$\norm{\Delta_t} \le \rho \max\{\norm{\x_t - \tilde{\x}}, \norm{\modify{\x}_t - \tilde{\x}} \} \le \rho\uspace$. We use induction to prove following:
\begin{align*}
\norm{\eta\sum_{\tau = 0}^{t-1}(\I - \eta\H)^{t-1-\tau} (\Delta_\tau \w_\tau  + \h_\tau) \w_\tau }
\le \frac{1}{2}\norm{(\I - \eta\H)^{t}\w_0}
\end{align*}
That is, the first term is always the dominating term.
It is easy to check for base case $t=0$; we have $0 \le \norm{\w_0}/2$. Suppose for all $\modify{t} \le t$ the induction holds, this gives:
\begin{align*}
\norm{\w_{\modify{t}}} \le
 \norm{(\I - \eta\H)^{\modify{t}}\w_0} + \norm{\eta\sum_{\tau = 0}^{\modify{t} - 1}(\I - \eta\H)^{\modify{t} - 1-\tau} (\Delta_\tau \w_\tau  + \h_\tau)} 
 \le 2\norm{(\I - \eta\H)^{\modify{t}}\w_0}
\end{align*}
Denote $\gamma = \lambda_{\min}(\hess f(\tilde{\x}))$, for case $t+1 \le \utime$, we have:
\begin{align*}
\norm{\eta\sum_{\tau = 0}^{t}(\I - \eta\H)^{t-\tau} \Delta_\tau \w_\tau }
\le& \eta\rho\uspace \sum_{\tau = 0}^{t} \norm{(\I - \eta\H)^{t-\tau}}\norm{\w_\tau} 
\le \eta\rho\uspace \sum_{\tau = 0}^{t} (1+\eta\gamma)^t \norm{\w_0}\\
\le &\eta\rho\uspace (t+1) \norm{(\I - \eta\H)^{t+1}\w_0}
\le \eta\rho\uspace\utime \norm{(\I - \eta\H)^{t+1}\w_0}\\
\le &\frac{1}{4}\norm{(\I - \eta\H)^{t+1}\w_0},
\end{align*}
where the third last inequality use the fact $\w_0$ is along minimum eigenvector direction of $\H$,
the last inequality uses the fact $\eta\rho\uspace T =c^{-1} \le 1/4$ for $c$ large enough. 

On the other hand, with $1-\delta/2$ probability, we also have:
\begin{align*}
\norm{\eta\sum_{\tau = 0}^{t}(\I - \eta\H)^{t-\tau} \h_{\tau} }
\le \eta \sum_{\tau = 0}^t (1+\eta\gamma)^{t - \tau} \norm{\h_{\tau}}
\le (1+\eta\gamma)^{t+1}\frac{\max_{\tau} \norm{\h_{\tau}}}{\gamma}
\le \frac{1}{4}\norm{(\I - \eta\H)^{t+1}\w_0},
\end{align*}
where the last inequality requires $\max_{\tau} \norm{\h_{\tau}} \le \gamma \norm{\w_0}$ which can be achieved by making minibatch size $m$ large enough. Now, by triangular inequality, we finishes the induction. 

Finally, we have:
\begin{align*}
\norm{\w_\utime} \ge& \norm{(\I - \eta\H)^{\utime}\w_0} - \norm{\eta\sum_{\tau = 0}^{\utime-1}(\I - \eta\H)^{\utime-1-\tau} (\Delta_\tau \w_\tau + \h_{\tau}) }\\
\ge& \frac{1}{2}\norm{(\I - \eta\H)^{\utime}\w_0}
\ge \frac{(1+\eta\sqrt{\rho\epsilon})^\utime \norm{\w_0}}{2}\\
=& 2^{\chi c} \cdot \frac{\delta \epsilon\chi^{-3}c^{-6}}{2\ell} \sqrt{\frac{2\pi}{d}} 
\ge 8\sqrt{\frac{\epsilon}{\rho}} \chi^{-1}c^{-2} = 2\uspace,
\end{align*}
where the last inequality requires 
$$2^{\chi c} \ge \frac{16}{\sqrt{2\pi}} \cdot \frac{\ell \sqrt{d}}{\delta\sqrt{\rho \epsilon}} \chi^2 c^4$$
Since $\chi = \max\{1, \log \frac{d \ell \Delta_f}{\rho\epsilon\delta}\}$, it is easy to verify when $c$ large enough, above inequality holds. 
% \cnote{More careful about the choice of $\chi$.}
% To make $\norm{\w_\utime} \ge 2\uspace$, 
% The choice of $\chi$ in Lemma \ref{lem:two_seq} gives $T  = \frac{\chi}{\eta\sqrt{\rho\epsilon}} \ge \frac{4}{\eta\sqrt{\rho\epsilon}} \log \frac{2}{\eta\rho r_0}$. 
% Since $2\eta\rho\uspace T \le 1$, we know
% $\frac{2}{\eta\rho r_0} \ge \frac{4\uspace T}{r_0}\ge\frac{4\uspace}{r_0}$. This implies
%  $T \ge \frac{4}{\eta\sqrt{\rho\epsilon}} \log \frac{4\uspace}{r_0}$, combined with $\eta\sqrt{\rho\epsilon} \le 1$, 
This gives $\norm{\w_\utime} \ge 2\uspace$
, which contradicts with the localization fact $\max\{\norm{\x_\utime - \tilde{\x}}, \norm{\modify{\x}_\utime - \tilde{\x}}\} \le \uspace$.

\end{proof}

\begin{proof}[Proof of Lemma \ref{lem:neg_curve_GD}]
Let $r_0 = \delta r \sqrt{\frac{2\pi}{d}}$ and applying Lemma \ref{lem:width}, we know $\cXs(\x_t)$ has at most width $\eta r_0$ in 
the minimum eigenvector direction of $\hess f(\x_t)$ and thus,
\begin{align*}
\text{Vol}(\cXs) \le \text{Vol}(\ball_0^{(d-1)}(\eta r)) \cdot \eta r_0
\end{align*}
which gives:
\begin{align*}
\frac{\text{Vol}(\cXs)}{\text{Vol}(\ball^{(d)}_{\x_t}(\eta r))}
\le \frac{\eta r_0 \times \text{Vol}(\ball^{(d-1)}_0(\eta r))}{\text{Vo{}l} (\ball^{(d)}_0( \eta  r))}
= \frac{r_0}{r\sqrt{\pi}}\frac{\Gamma(\frac{d}{2}+1)}{\Gamma(\frac{d}{2}+\frac{1}{2})}
\le \frac{r_0}{r\sqrt{\pi}} \cdot \sqrt{\frac{d}{2}+\frac{1}{2}} \le \delta
\end{align*}
Therefore with $1-\delta$ probability, the perturbation lands in $\ball^{(d)}_{\x_t}(\eta r) - \cXs$, where by definition we have with probability at least $1-\sqrt{\delta}$
\begin{equation*}
f(\text{SGD}_{\xi}^{(\utime)}(\x)) - f(\tilde{\x}) \le -\ufun
\end{equation*}
Therefore the probabilty of escaping saddle point is $(1-\delta)(1-\sqrt{\delta}) \ge  1-2\sqrt{\delta}$. Reparametrizing $\delta' = 2\sqrt{\delta}$ only affects constant factors in $\chi$, hence we finish the proof.
\end{proof}

\end{document}